\documentclass[a4paper,utf8]{scrartcl}
\usepackage[margin=1.5cm,foot=0.7cm
]{geometry}
\usepackage{multirow}%
\usepackage{datetime2}
\usepackage{graphicx}
\usepackage[font=footnotesize,labelfont=bf,format=plain]{caption}
\usepackage{subcaption}
\usepackage{hyperref}
\usepackage{amssymb}
\usepackage{amsmath}
\usepackage{outlines}
\usepackage{float}
\usepackage{bm}
\usepackage{amstext} 
\usepackage{array}   
\usepackage[citestyle=nature,backend=biber,bibencoding=utf8,
natbib=true]{biblatex} 
\usepackage[markup=nocolor]{changes}
\usepackage{booktabs}
\usepackage{graphicx}
\usepackage{adjustbox}
\usepackage{ifthen}
\usepackage{pdftexcmds}
\usepackage{csvsimple}
\usepackage{datetime}
\usepackage{makecell}
\usepackage{siunitx}
\usepackage[graphicx]{realboxes} 
\usepackage{authblk}
\usepackage{amsthm}

\usepackage{pifont}
\newcommand{\cmark}{\ding{51}}%
\newcommand{\xmark}{\ding{55}}%

\newcommand{\meth}{"Methods"}
\newcommand{\eulerA}{\ensuremath{\Bar{A}_\mathrm{EF}}}
\newcommand{\eulerB}{\ensuremath{\Bar{B}_\mathrm{EF}}}
\newcommand{\seA}{\ensuremath{\Bar{A}_\mathrm{SE}}}
\newcommand{\seB}{\ensuremath{\Bar{B}_\mathrm{SE}}}
\newcommand{\revision}[1]{#1}
\newtheorem{theorem}{Theorem}[section]
\newtheorem{corollary}{Corollary}[theorem]
\newtheorem{lemma}[theorem]{Lemma}
\newtheorem{prop}{Proposition}
\newcommand{\enameref}[1]{"\nameref{#1}"}

\bibliography{sn-bibliography}

\usepackage{fancyhdr}
\begin{document}

\title{Advancing Spatio-Temporal Processing in Spiking Neural Networks through Adaptation}
\author[1,2,*]{Maximilian Baronig}
\author[1,2,*]{Romain Ferrand}
\author[3]{Silvester Sabathiel}
\author[1]{Robert Legenstein}
\affil[1]{Institute of Machine Learning and Neural Computation, Graz University of Technology, Graz,
Austria}
\affil[2]{TU Graz - SAL Dependable Embedded Systems Lab, Silicon Austria Labs, Graz, Austria}
\affil[3]{Silicon Austria Labs GmbH, Graz, Austria}
\affil[*]{Equal contribution}
\date{\vspace{-5ex}}

\maketitle
\thispagestyle{fancy}

\begin{abstract}
    \revision{Implementations of spiking neural networks on neuromorphic hardware promise orders of magnitude less power consumption than their non-spiking counterparts.
    The standard neuron model for spike-based computation on such systems has long been the leaky integrate-and-fire (LIF) neuron.
    A computationally light augmentation of the LIF neuron model with an adaptation mechanism has recently been shown to exhibit superior performance on spatio-temporal processing tasks. The root of the superiority of these so-called adaptive LIF neurons however is not well understood.     
    In this article, we thoroughly analyze the dynamical, computational, and learning properties of adaptive LIF neurons and networks thereof. Our investigation reveals significant challenges related to stability and parameterization when employing the conventional Euler-Forward discretization for this class of models. We report a rigorous theoretical and empirical demonstration that these challenges can be effectively addressed by adopting an alternative discretization approach -- the Symplectic Euler method, allowing to improve over state-of-the-art performances on common event-based benchmark datasets.
    Our further analysis of the computational properties of networks of adaptive LIF neurons shows that they are particularly well suited to exploit the spatio-temporal structure of input sequences without any normalization techniques.} 
\end{abstract}

\section{Introduction}\label{sec:intro}

Spiking neural networks (SNNs) have emerged as a viable biologically inspired alternative to artificial neural network (ANN) models \cite{maass1997networks}. In contrast to ANNs, where neurons communicate analog numbers, neurons in SNNs communicate via digital pulses, so-called spikes. This event-based communication resembles the communication of neurons in the brain and enables highly energy efficient implementations in neuromorphic hardware \cite{schemmel2010wafer,furber2014spinnaker, merolla2014million,davies2018loihi,davies2021advancing}. Recent advances in SNN research have shown that SNNs can be trained in a similar manner as ANNs using backpropagation through time (BPTT), leading to highly accurate models \cite{leeTrainingDeepSpiking2016,slayer}.

The dominant spiking neuron model used in SNNs is the leaky integrate and fire (LIF) neuron \cite{gerstner2014neuronal}. The LIF neuron has a single state variable $u(t)$ that represents the membrane potential of a biological neuron. Incoming synaptic currents are integrated over time in a leaky manner on the time scale of tens of milliseconds. Once the membrane potential reaches a threshold $\vartheta$, the membrane potential is reset and the neuron emits a spike (i.e., its output is set to $1$). 
The leaky integration property of the LIF neuron model reproduces the sub-threshold behavior of so-called excitability class 1 neurons in the brain (Fig.~\ref{fig:fig1}a, left) \cite{hodgkinLocalElectricChanges1948,izhikevichNeuralExcitabilitySpiking2000a}.
\begin{figure}[!t]
    \centering
    \includegraphics[width=\textwidth]{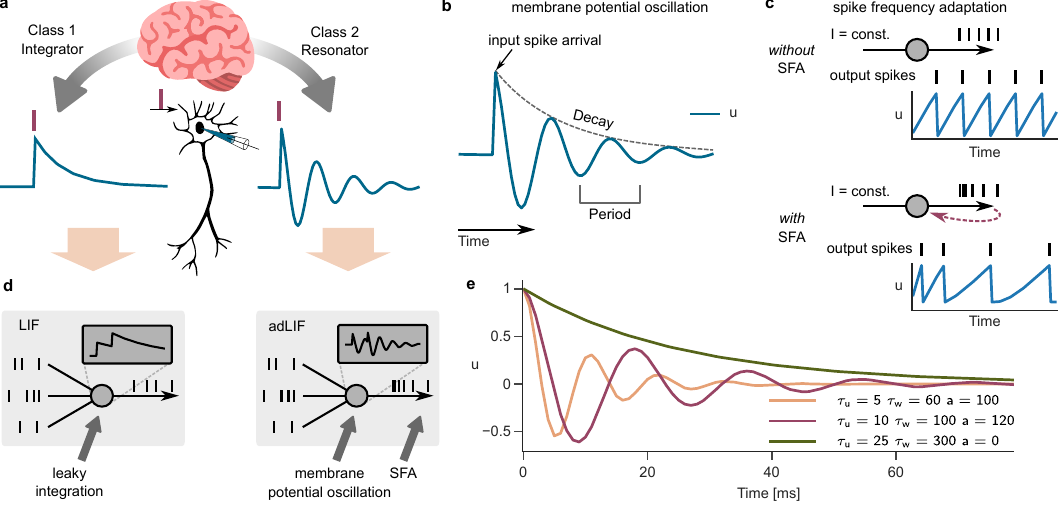}
    \caption{\textbf{The adLIF neuron with membrane potential oscillation and spike-frequency adaptation.} \textbf{a)} Neurons in the brain have been classified into two excitability classes, integrators (class 1) and resonators (class 2). While resonators show membrane potential oscillations, integrators do not. \textbf{b)} Membrane potential oscillation in response to an input pulse. The period $P$ and decay $r$ are sufficient to fully characterize the oscillating spike response. \textbf{c)} Example of a neuron without (top) and with (bottom) spike frequency adaptation (SFA). Dotted arrow illustrates the feed-back of output spikes to the neuron membrane state in the adLIF model. \textbf{d)} Adaptive LIF (adLIF) neurons (see Eq.~\eqref{eq:continuous1}) differ in 2 features from vanilla LIF neurons: membrane potential oscillations and SFA. \textbf{e)} Impulse response functions for different parameterizations of adLIF neurons. For $a=0$ (green), no oscillations occur and the spike response reduces to leaky integration.}
    \label{fig:fig1}
\end{figure}
A second class of neurons called excitability class 2 neurons (Fig.~\ref{fig:fig1}a, right) exhibit more complex dynamics with sub-threshold membrane potential oscillations (Fig.~\ref{fig:fig1}b) and spike frequency adaptation (SFA, Fig.~\ref{fig:fig1}c). 
Such complex dynamics cannot be modelled with the single state variable $u(t)$ of the LIF neuron. Pioneering work has shown that these behaviors can be reproduced by a simple extension of the LIF neuron model that adds a second state variable to the neuron dynamics which interacts with the membrane potential in a --- typically negative --- feedback loop \cite{izhikevich2001resonate}. Neuron models of this type are called adaptive LIF neurons.

With the growing interest in SNNs for neuromorphic systems, researchers have started to train recurrent SNNs (RSNNs) consisting of adaptive LIF neurons with BPTT on spatio-temporal processing tasks. First results were based on neuron models that implement a threshold adaptation mechanism, where the second state variable is a dynamic threshold $\vartheta(t)$ \cite{bellec2018long,salaj2021spike,gangulySpikeFrequencyAdaptation2024a}. Each spike of the neuron leads to an increase of this threshold, which implements the negative feedback loop mentioned above and leads to SFA (Fig.~\ref{fig:fig1}c). 
The performance of these models clearly surpassed those achieved with networks of LIF neurons while being highly efficient on neuromorphic hardware with orders of magnitudes energy savings when compared to implementations on CPUs or graphical processing units (GPUs) \cite{davies2021advancing}.

While threshold adaptation implements SFA, the resulting neuron model still performs a leaky integration of input currents and does not exhibit the typical sub-threshold membrane potential oscillations of class 2 neurons. 
Hence, more recent work considered networks of neurons with a form of adaptation often referred to as sub-threshold or current-based adaptation. The second state variable $w(t)$ is interpreted as negative adaptation current that is increased not only by neuron spikes but also by the sub-threshold membrane potential itself. This sub-threshold feedback leads to complex oscillatory membrane potential dynamics (Fig.~\ref{fig:fig1}b).
Interestingly, simulation studies have shown that SNNs equipped with sub-threshold adaptation achieve significantly better performances than SNNs with threshold adaptation \cite{bittar2022surrogate,gangulySpikeFrequencyAdaptation2024a,higuchi2024balanced}.  

Although the converging evidence suggests that networks of adaptive LIF neurons are superior to LIF networks for neuromorphic applications, there are still many questions open. First, to achieve top performance, usually all neuron parameters are trained together with the synaptic weights. Changes of the neuron parameters however can quickly lead to unstable models which disrupts training. To avoid instabilities, parameter bounds have to be defined and fine-tuned. If these bounds are too wide, the network can become unstable, if they are too narrow, one cannot utilize the full computational expressivity. 

\revision{In this work, we} show that this problem is not inherent to the neuron model but rather caused by the standard discrete-time formulation of the continuous neuron dynamics which is based on the Euler-Forward discretization method. \revision{Despite mere stability issues, we identified a plethora of drawbacks arising from the application of the widely used Euler-Forward discretization. These include, inter alia, unintended interdependencies between neuron parameters, deviations of the discrete model dynamics from its continuous counterpart, limitations in neuron expressibility, strong dependence between the discretization time step and the neuron dynamics, non-trivial divergence boundaries. Our thorough theoretical analysis reveals that the alternative, evenly simple Symplectic-Euler discretization method remarkably alleviates the drawbacks of the Euler-Forward method almost entirely, without additional computational cost or implementation complexity. While this discretization method could in principle be applied to an entire family of multi-dimensional neuron models, we mainly focus our theoretical and empirical analyses on a specific adaptive neuron model that recently gained traction.} Using this insight, we demonstrate the power of adaptation by showing that our improved adaptive RSNNs outperform the state-of-the-art on spiking speech recognition datasets as well as an ECG dataset. We then show that the superiority of adaptive RSNNs is not limited to classification tasks but extends to the prediction and generation of complex time series.
Second, there is a lack of understanding why sub-threshold adaptation is so powerful in RSNNs. We thoroughly analyze the computational dynamics in single adaptive LIF neurons, as well as in networks of such neurons. Our analysis suggests that adaptive LIF neurons are especially capable of detecting temporal changes in input spike density, while being robust to shifts of total spike rate. Hence, adaptive RSNNs are well suited to analyze the temporal properties of input sequences. 
Third, high-performance SNNs are usually trained using normalization techniques such as batch normalization or batch normalization through time \cite{guoMembranePotentialBatch2023,jiangTABTemporalAccumulated2023,zhengGoingDeeperDirectlyTrained2021,vicente-solaSpikingNeuralNetworks2023}.
These methods however complicate the training process and the implementation of networks on neuromorphic hardware. 
We show that adaptation has a previously unrecognized benefit on network optimization. Since adaptation inherently stabilizes network activity, we hypothesized that explicit normalization is not necessary in adaptive RSNNs. In fact, all our results were obtained without explicit normalization techniques. We test this hypothesis and show that in contrast to LIF networks, networks of adaptive LIF neurons can tolerate substantial shifts in the mean input strength as well as substantial levels of background noise even when these perturbations were not observed during training.

\section*{Results}\label{sec:results}
\subsection*{Adaptive LIF neurons}
\label{subsec:adlif_neurons}
The leaky integrate-and-fire (LIF) neuron model \cite{gerstner2014neuronal} evolved as the gold standard for spiking neural networks due to its simplicity and suitability for low-power neuromorphic implementation \cite{rathi2023exploring}. The continuous-time equation for the membrane potential of the LIF neuron at time $t$ is given by 
\begin{equation}\label{eq:LIF}
    \tau_u \dot{u}  =-u(t)+I(t),
\end{equation}
where $\tau_u$ is the membrane time constant, and $I(t)=\sum_{i}\theta_i x_i(t)$ the input current composed of the sum of neuron inputs $x_i(t)$ scaled by the corresponding synaptic weights $\theta_i$. A dot above a variable denotes its derivative with respect to time. 
If the membrane potential $u(t)$ crosses the spike threshold $\vartheta$ from below, a spike is emitted and $u$ is reset to the reset potential. 
In the absence of input, $u(t)$ decays exponentially to zero. The LIF neuron equation models so-called integrating class 1 neurons in the brain (Fig.~\ref{fig:fig1}a), which are integrating incoming currents in a leaky manner. 
The simple first-order dynamics however does not allow the LIF model to account for another class of neurons frequently occurring in the brain: resonating/oscillating class 2 neurons (Fig.~\ref{fig:fig1}a,b). In contrast to integrators, such neurons exhibit oscillatory behavior in response to stimulation, giving rise to interesting properties entirely neglected by LIF neurons. Such oscillatory behavior is often modelled by adding a second time-varying variable --- the adaptation current $w(t)$ --- to the neuron state \cite{izhikevich2001resonate,brunelFiringrateResonanceGeneralized2003,bittar2022surrogate,deckers2024co}. 
The resulting neuron model, which we refer to as adaptive leaky integrate-and-fire (adLIF) neuron, has significant advantages over LIF neurons in terms of feature detection capabilities and gradient propagation properties, as we show in the next few sections. 
The adLIF model is described in terms of two coupled differential equations 
\begin{align}
    \tau_u \dot{u} & =-u(t)+I(t)-w(t)  \label{eq:continuous1} \\
    \tau_w \dot{w} &= -w(t)+au(t) + bz(t), \label{eq:continuous2}
\end{align}
where $\tau_w$ is the adaptation time constant and $a$ and $b$ are adaptation parameters, defining the behavior of the neuron. When comparing the LIF equation \eqref{eq:LIF} with equation \eqref{eq:continuous1}, we see that the latter resembles the LIF dynamics where the adaptation current $w(t)$ is subtracted, with its dynamics defined in equation \eqref{eq:continuous2}. 

The parameter $a \in  \mathbb{R}$ scales the coupling of the membrane potential $u(t)$ with the adaptation current $w(t)$. The negative feedback loop between $u(t)$ and $w(t)$ defined by Equations \eqref{eq:continuous1} and \eqref{eq:continuous2} leads to oscillations of the membrane potential for large enough $a$, see Fig.~\ref{fig:fig1}b. The oscillation can be characterized by the decay rate $r$ and the period $P=\frac{1}{f}$ given by the inverse of the intrinsic frequency $f$. As we discuss later in the manuscript, $f$ characterizes the frequency tuning of the neuron, whereas the decay rate $r$ is an indicator of its stability and time scale. 
\begin{figure}[!t]
    \centering
    \includegraphics[width=\textwidth]{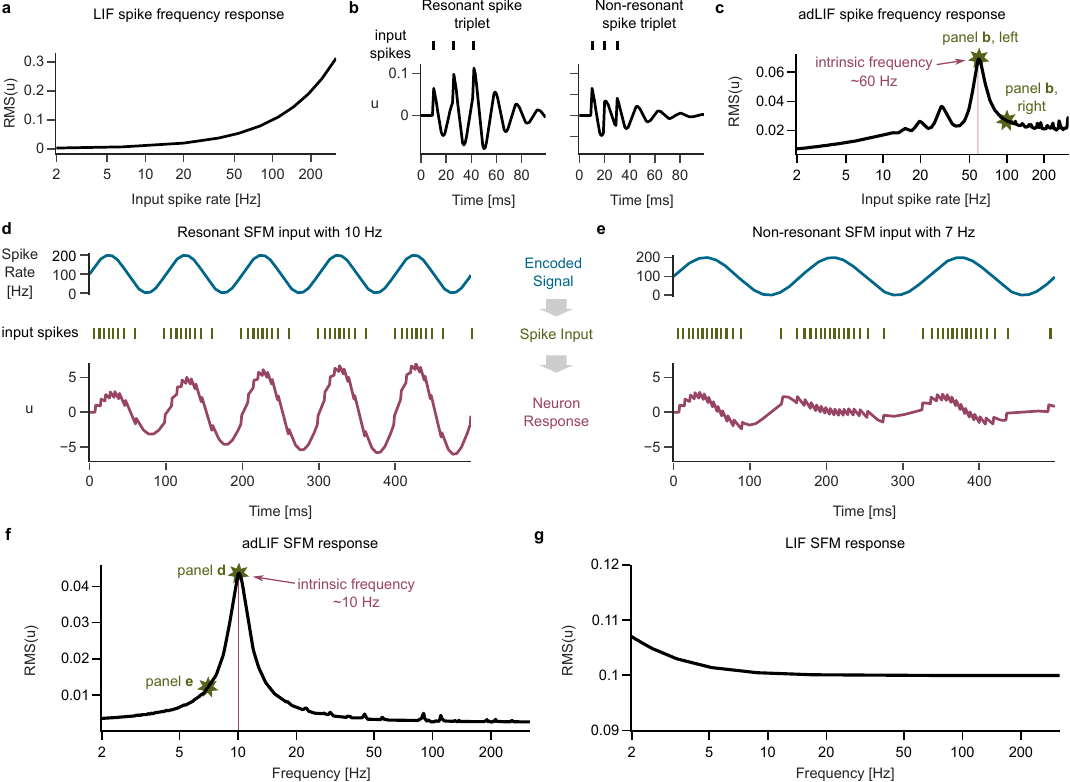}
    \caption{\textbf{The adLIF neuron model shows frequency-selective resonating behavior.} \textbf{a)} Voltage response (root mean squared membrane potential over $10$ seconds) of a LIF neuron in response to tonic spike input of different rates. \textbf{b)} Membrane potential response of an adLIF neuron with intrinsic frequency of $f\approx60$ Hz to an input spike triplet at $60$ Hz (left) and $100$ Hz (right). \textbf{c)} Voltage response of the same adLIF neuron as in panels b and c for tonic spiking input at different rates. \textbf{d)}  A $10$ Hz sinusoidal input signal (top) is encoded as an input spike train (middle) through spike frequency modulation (SFM), see main text for details. Membrane potential response of an adLIF neuron with intrinsic frequency $f\approx10$ Hz (bottom). \textbf{e)} Same as panel e, but for a sinusoidal input at $7$ Hz. \textbf{f)} Voltage response of an adLIF neuron to SFM-encoded 
    sinusoidal input at various frequencies.  \textbf{g)} Same as panel f, but for a LIF neuron. See \meth~for parameters and input generation.}
    \label{fig:fig2}
\end{figure}

The parameter $b \ge 0$ weights the feed-back from the neuron's output spike $z(t)$ onto the adaptation variable $w(t)$. Hence, each spike has an inhibitory effect on the membrane potential, which leads to spike frequency adaptation (SFA) \citep{gangulySpikeFrequencyAdaptation2024a}, see Fig.~\ref{fig:fig1}c. We refer to this auto-feed-back governed by parameter $b$ as spike-triggered adaptation in the following.
SFA has also been implemented directly using an adaptive firing threshold that is increased with every output spike \cite{yinAccurateEfficientTimedomain2021,lsnn_igi}. In contrast to the adLIF model, these models do not exhibit membrane potential oscillations. 

The adLIF model combines both membrane potential oscillations and SFA in one single neuron model (Fig.~\ref{fig:fig1}d).
Depending on the parameters, adLIF neurons can exhibit oscillations of diverse frequencies and decay rates (Fig.~\ref{fig:fig1}e), and are equivalent to LIF neurons for $a,b=0$, where neither oscillations nor spike-triggered adaptation occur. A reduced variant of the adLIF neuron is given by the resonate-and-fire  neuron \cite{izhikevich2001resonate}.

Originally developed to efficiently replicate firing patterns of biological neurons, the adLIF model recently gained attention due to significant performance gains over vanilla LIF neurons in several benchmark tasks, despite its little computational overhead \citep{bittar2022surrogate, gangulySpikeFrequencyAdaptation2024a,deckers2024co}. In particular, gradient-based training of networks of adLIF neurons on spatio-temporal processing tasks appears to synergize well with oscillatory dynamics. However, these empirical findings are so far not accompanied by a good understanding of the reasons for this superiority.

When comparing the responses of the LIF and adLIF neuron, an important computational consequence of membrane potential oscillations has been noted: In contrast to the LIF neuron, which responds with higher amplitude of $u$ to higher input spike frequency (Fig.~\ref{fig:fig2}a), the adLIF neuron is most strongly excited if the frequency of input spikes matches the intrinsic frequency $f$ of the neuron (Fig.~\ref{fig:fig2}b-c, see also \cite{izhikevich2001resonate,brunelFiringrateResonanceGeneralized2003,higuchi2024balanced}). To demonstrate this resonance phenomenon, we show the membrane potential of an adLIF neuron with intrinsic frequency $f=60$ Hz for an input spike triplet exactly at this intrinsic frequency $f$ (Fig.~\ref{fig:fig2}b, left), compared to a spike triplet of higher rate (Fig.~\ref{fig:fig2}b, right). The resulting amplitude of the membrane potential $u$ is higher in the former case, indicating resonance. 
Fig.~\ref{fig:fig2}c shows that the neuron exhibits a frequency selectivity specifically for its intrinsic frequency $f$. 

We took this analysis a step further and asked whether this resonance could account for frequencies in the input spike train that are not directly encoded by spike rate, but rather by slow changes of the spike rate over time, a coding scheme previously termed spike frequency modulation (SFM) \cite{sfmwang}. As a guiding example, we encoded a slow-varying sinusoidal signal as a spike train, where the magnitude of the signal at a certain time is given by the local spike rate, shown in Fig.~\ref{fig:fig2}d. The spike rate thereby varied between $0$ Hz and $200$ Hz, whereas the underlying, encoded sinus signal oscillated with a constant frequency of $10$ Hz. Again, we see increased response of the membrane potential $u$ over time in the case of the $10$ Hz input compared to a slower $7$ Hz sinus signal (Fig.~\ref{fig:fig2}d,e), due to resonance with the adLIF neuron, see also Fig.~\ref{fig:fig2}f. In contrast, the corresponding membrane voltage response amplitude of a LIF neuron is almost indifferent to the intrinsic frequency of the underlying sinus input, see Fig.~\ref{fig:fig2}g.
This shows that in contrast to the LIF neuron, the adLIF neuron model is sensitive to the longer-term temporal structure, i.e. variation of the input signal. In Section \enameref{subsec:inductive_bias}, we highlight the importance of this frequency-dependence of neuron responses as a key ingredient for the powerful feature detection capabilities of networks of adLIF neurons.
\subsection*{The Symplectic-Euler discretized adLIF neuron}
In the previous section, we defined the LIF and adLIF neuron models via continuous-time, ordinary differential equations.
In practice, it is however standard to discretize the continuous-time dynamics of the spiking neuron model. This allows not only to use powerful auto-differentiation capabilities of machine learning software packages such as TensorFlow \citep{tensorflow2015-whitepaper} or PyTorch\citep{anselPyTorchFasterMachine2024}, but also for implementation of such neuron models in discrete-time operating neuromorphic hardware \cite{orchardEfficientNeuromorphicSignal2021}. Discretization of the LIF neuron model Eq.~\eqref{eq:LIF} is straight-forward.
In contrast, for the adLIF model, the interdependency of the two state variables during a discrete time step $\Delta t$ cannot be taken into account exactly in a simple manner (the exact solution involves a matrix exponential). 
Nevertheless, for efficient simulation and hardware implementations, simple update equations are needed. Therefore, approximate discrete update equations for the membrane potential $u$ and the adaptation current $w$ are usually obtained by the Euler-Forward method \cite{bittar2022surrogate}. In the following we analyze discretization methods for adLIF neurons through the lens of dynamical systems analysis. We find that the Euler-Forward method is problematic, and propose the utilization of a more stable alternative discretization method. 

A common approach to study dynamical systems is through the state-space representation, which recently gained popularity in the field of deep learning \cite{guCombiningRecurrentConvolutional2021,guMambaLinearTimeSequence2023a}. Re-formulation of spiking neuron models in a canonical state-space representation provides a convenient unified way to study their dynamical properties. The continuous-time equations of the adLIF neuron Eq.~\eqref{eq:continuous1}, \eqref{eq:continuous2} can be re-written in such a state-space representation as a 2-dimensional linear time-invariant (LTI) system with state vector $\bm{s}$ as 
\begin{align}
\dot {\bm{s}}(t) = 
    \begin{pmatrix}
    \dot{u}(t) \\
    \dot{w}(t)
    \end{pmatrix} &= A \bm{s}(t) + B  \bm{x}(t) \nonumber \\
    &= \underbrace{\begin{pmatrix}
     -\frac{1}{\tau_u} & -\frac{1}{\tau_u}\\
    \frac{a}{\tau_w} & -\frac{1}{\tau_w}
    \end{pmatrix}}_{A}
    \begin{pmatrix}
    u(t) \\
    w(t)
    \end{pmatrix} + \underbrace{
    \begin{pmatrix}
    \frac{1}{\tau_u} & 0\\
    0 & \frac{b}{\tau_w} \\
    \end{pmatrix}}_{B} \begin{pmatrix}
        I(t) \\
        z(t)
    \end{pmatrix} ,
    \label{eq:continuous_lti}
\end{align}
with system matrix $A$ and input matrix $B$. This equation only describes the sub-threshold dynamics of the neuron (i.e., it holds as long as the threshold is not reached). The reset can be accounted for by the threshold condition: When the voltage crosses the firing threshold $\vartheta$, an output spike is elicited and the neuron is reset. 
The goal of discretization is to obtain discrete-time update equations of the form 
\begin{equation}
\label{eq:discrete_lds}
    \bm{s}[k] = 
    \begin{pmatrix}
    u[k] \\
    w[k]
    \end{pmatrix} = \Bar{A} \bm{s}[k-1] + \Bar{B}  \bm{x}[k],
\end{equation}
where $f[k]$ denotes the value of state variable $f$ at discrete time step $k$, i.e., $f[k] \equiv f(\Delta t k)$ for discrete time increment $\Delta t$ and integer-valued $k>0$. Here, $\Bar{A}$ and $\Bar{B}$ denote the state and input matrix of the discrete time system respectively. In the SNN literature, the most commonly used approach to obtain the discrete approximation to the continuous system from Eq.~\eqref{eq:continuous_lti} is the Euler-Forward method \cite{bittar2022surrogate,deckers2024co,higuchi2024balanced}, which results in update equations
\begin{subequations}
\begin{align}
\label{eq:euler_u}
    \hat{u}[k] &= \alpha u[k-1] + (1-\alpha) \left(-w[k-1] + I[k]\right) \\
    w[k] &= \beta w[k-1] + (1-\beta) \left(a u[k-1] + b S[k]\right),
\end{align}
\end{subequations}
where $\hat{u}$ denotes the membrane potential before the reset is applied, $\alpha = 1-\frac{\Delta t}{\tau_u}$, and $\beta = 1-\frac{\Delta t}{\tau_w}$.  The spike output of the neuron is given by
\begin{equation}
    S[k] = \begin{cases}
        1 & \text{if } \hat{u}[k] > u_\mathrm{th} \\
        0 & \text{otherwise.}
    \end{cases}
\end{equation}
Finally, $u[k]$ is obtained by applying the reset to $\hat{u}$ via
\begin{equation}
\label{eq:reset}
    u[k] = \hat{u}[k] \cdot (1-S[k]).
\end{equation}
This Euler-Forward discretization yields the discrete state-space matrices
\begin{equation}
    \eulerA =  \begin{pmatrix}
     \alpha & -(1-\alpha)\\
    a(1-\beta) & \beta
    \end{pmatrix} \hspace{30pt}
    \eulerB =  \begin{pmatrix}
    (1-\alpha) & 0\\
    0 & b(1-\beta) \\
    \end{pmatrix}
    \label{eq:euler_lti}
\end{equation}
 for $\Bar{A}$ and $\Bar{B}$ in Eq.~\eqref{eq:discrete_lds}. In practice (see for example \cite{bittar2022surrogate}), the coefficients $\alpha$ and $\beta$ are often replaced by exponential decay terms $\alpha = \exp\left(-\frac{\Delta t}{\tau_{u}}\right)$ and $\beta = \exp\left(-\frac{\Delta t}{\tau_{w}}\right)$ akin to the LIF discretization, as it is exact in the latter case. However, for adLIF neurons, the Euler-Forward approximation is quite imprecise which can quickly result in unstable and diverging behavior of the system, as we will show below. A better approximation is given by the bilinear discretization method, a standard method also used in state space models \cite{guCombiningRecurrentConvolutional2021}, which is however computationally more demanding. An alternative is the Symplectic-Euler (SE) method \cite{geometric_numerical_integration} that has previously been used in non-spiking oscillatory systems \cite{effenbergerFunctionalRoleOscillatory2023b}.
 We found that the Symplectic-Euler (SE) discretization provides major benefits in terms of stability, expressivity, and trainability of the adLIF neuron, while being computationally as efficient as Euler-Forward. The SE method has been shown to preserve the energy in Hamiltonian systems, a desirable property of a discretization of such systems \cite{geometric_numerical_integration}. As we show below, the improved stability of the SE method still applies to the adLIF neuron model, even though it is non-Hamiltonian.
 The SE method is similar to the Euler-Forward method, with the only difference that one computes the state variable $w[k]$ from $u[k]$ instead of $u[k-1]$, resulting in the discrete dynamics 
\begin{subequations}
\begin{align}
    \hat{u}[k] &= \alpha u[k-1] + (1-\alpha) \left(-w[k-1] + I[k]\right) \label{eq:u_improved_update} \\
    w[k] &= \beta w[k-1] + (1-\beta) (a u[k] + b S[k]).
    \label{eq:w_improved_update}
\end{align}
\end{subequations}
We refer to this neuron model as the SE-adLIF model in order to distinguish it from the Euler-Forward discretized model.
 Note that the reset mechanism from Eq.~\eqref{eq:reset} is applied to obtain $u[k]$ from $\hat{u}[k]$ before computing $w[k]$. While it is also possible to apply the reset after computing $w[k]$, we found that the above described way yields the best performance. 
For the sub-threshold dynamics, this leads to update matrices (see Section \enameref{subsec:meth_update_matrices_for_se} in \meth)
\begin{equation}
    \seA =  \begin{pmatrix}
     \alpha & -(1-\alpha)\\
    a(1-\beta)\alpha & \quad \beta - a(1 - \beta)(1 - \alpha)\\
    \end{pmatrix} \hspace{30pt}
    \seB =  \begin{pmatrix}
    (1-\alpha) & 0\\
    0 & b(1-\beta) \\
    \end{pmatrix}
    \label{eq:se_lti}
\end{equation}
for the discrete state-space formulation given by Eq.~\eqref{eq:discrete_lds}. 
\subsection*{Stability analysis of discretized adLIF models}
\label{subsec:stability}
\revision{A desirable characteristic of a discretization method is its ability to maintain a close alignment between the discretized system and the continuous ground-truth. In contrast to SE, the EF discretization exhibits a pronounced dependence of this alignment on the discretization time step $\Delta t$, thereby reducing its robustness.
This is visualized in Fig.~\ref{fig:stability}a, where we discretized an adLIF neuron with the EF method (left) and the SE method (right) using $3$ different discretization time steps $\Delta t$. We observed that the EF-discretized neuron clearly diverges for larger values of $\Delta t$, in this example even for $\Delta t = 1$ ms. In contrast, the same neuron discretized with the Symplectic Euler method is robust to the choice of $\Delta t$. The divergence of the neuron can be quantified by its decay rate $r$, which gives the exponential decay of the envelope of a neuron's membrane potential $u(t)$ (see also Fig.~\ref{fig:fig1}b). For $r<1$, the neuron stably decays to a resting-state equilibrium. However, if this decay rate exceeds $1$, the neuron becomes unstable and its membrane potential grows indefinitely, as observable for the EF-adLIF neuron with $\Delta t = 1$ in Fig.~\ref{fig:stability}a. When comparing the relationship of the decay rate $r$ with respect to discretization time step $\Delta t$, as visualized in Fig.~\ref{fig:stability}b, the favorable adherence of the SE-discretized adLIF to the continuous model is evident. The SE-adLIF decay rate is independent of $\Delta t$ and evaluates to $r\approx 0.972$, which is the decay rate of the continuous model. For the EF discretization in contrast, $r$ grows along with increasing $\Delta t$, resulting in discretized neurons exceeding the stability boundary at $r=1$. As the computational cost of training SNNs via BPTT increases with smaller discretization time steps due to longer sequence lengths, the SE discretization is clearly favorable over EF, since it ensures stability and adherence to the continuous ground truth when $\Delta t$ is large.} 
\revision{The adherence of SE-adLIF to the continuous adLIF model is not limited to its robustness to the choice of $\Delta t$. For a given $\Delta t = 1$ ms, SE-adLIF follows the characteristics of the continuous model with respect to its parameters $\tau_u$, $\tau_w$, and $a$ much closer. While time constants $\tau_u$ and $\tau_w$ affect the decay rate $r$ of the adLIF neuron, parameter $a$ determines the frequency of oscillation. }
Since the continuous-time adLIF neuron model Eq.~\eqref{eq:continuous1} is inherently stable for $a\geq -1$ (see Section \enameref{subsubsec:continuous_stability_bounds} in \meth~for proof), it is a desirable property of a discretization method to preserve this stability for all possible parameterizations. 
This is shown empirically in Fig.~\ref{fig:stability}c. We instantiated $1{,}000$ different neurons for both discretization methods, Euler-Forward and SE, as well as the continuous model, in a grid-like manner over a reasonable parameter range and plot their calculated frequencies and decay rates in Fig.~\ref{fig:stability}c. While the continuous system  is stable for all considered parameter combinations (middle panel), the Euler-Forward approximation (left panel) is unstable for many parameterizations (decay rate $r>1$). In contrast, for the SE discretization (right panel), all parameter combinations resulted in stable neuron dynamics ($r<1$). These empirical results show that the SE discretization more closely follows the stability properties of the continuous model, whereas the Euler-Forward method deviates drastically from both. How can this discrepancy be explained? 

When analyzing the discretized neurons, one has to calculate $f$ and $r$ directly from the discrete system by calculating the eigenvalues $\lambda_{1,2}$ of the state transition matrix $\Bar{A}$. This allows to study the behavior and the stability of the adLIF neuron model for different discretizations. Two cases have to be differentiated: If the eigenvalues are complex, the membrane potential exhibits oscillations, whereas if they are real, no oscillations occur and the neuron behaves similar to a LIF neuron. In the complex case, we can write the eigenvalues in polar form as $\lambda_{1,2}=r e^{\pm j\phi}$, 
where $j$ denotes the imaginary unit.
Hence, the eigenvalues are complex conjugates, the decay rate $r$ is given by their modulus (see Fig.~\ref{fig:stability}d) and $\phi$ is obtained as the argument of $\lambda_1$ ($\text{arg}(\lambda_2) = -\phi$). Intuitively, the angle $\phi$ is the rotation of the neuron state with each time step $\Delta t$ in radians, and hence determines the frequency of the oscillation. One thus obtains the intrinsic frequency $f$ in Hertz as $f=\frac{\phi}{2\pi\Delta t}$.
In the case of real eigenvalues, $r$ is given by the magnitude of the largest eigenvalue. AdLIF neurons can thereby represent underdamped (complex eigenvalues), critically damped (equal real eigenvalues), and overdamped (non-equal real eigenvalues) systems via different parameterizations. Note, that only in the underdamped case the neuron can oscillate.

\begin{figure}[!t]
    \centering
    \includegraphics[width=\textwidth]{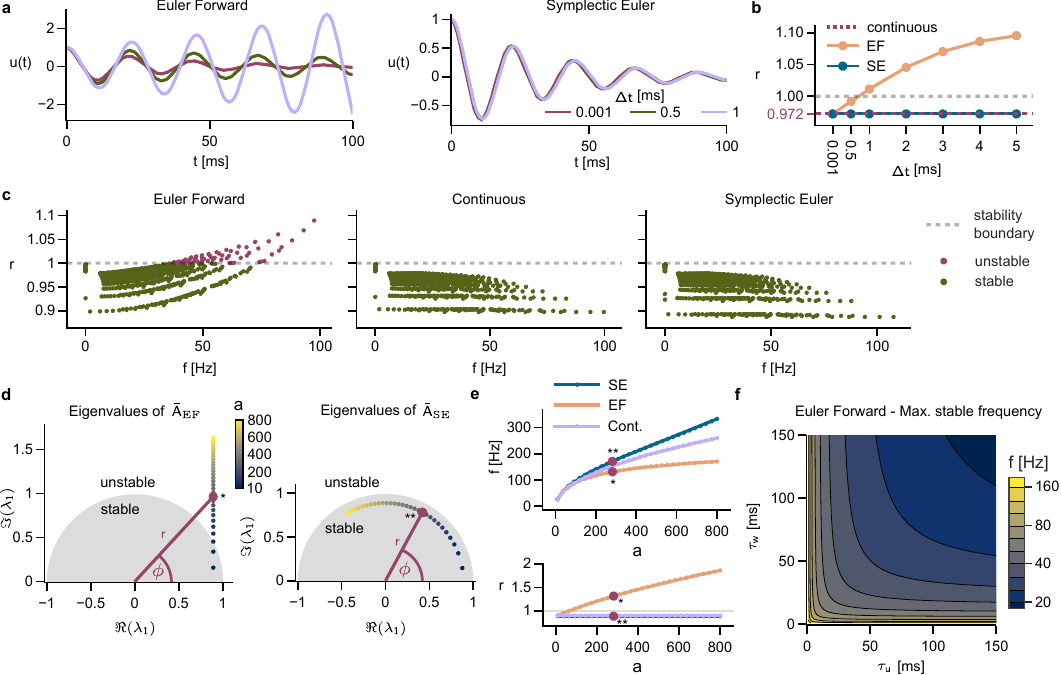}
    \caption{\textbf{Stability of adLIF discretizations.}  \revision{\textbf{a)} Membrane potential $u(t)$ over time for an EF-adLIF (left) and a SE-adLIF (right) neuron for different discretization time steps $\Delta t \in \{0.001,0.5,1\}$. Both neurons have the same parameters ($\tau_u=25$ ms, $\tau_w=60$ ms, $a=120$).} 
    \revision{\textbf{b)} Relationship between the decay rate $r$ and discretization time step $\Delta t$ for adLIF models with different discretizations, EF and SE. All decay rates are calculated with respect to $1$ ms, a decay rate of $r=0.9$ hence represents a decrease in magnitude of $10\%$ every $1$ ms. The decay rate ($r=0.972$) of the equivalently parameterized continuous model is highlighted. Same neuron parameters as in panel a.}
    \textbf{c)} Intrinsic frequency $f$ and per-timestep decay rate $r$ of $1{,}000$ different parameterizations of adLIF neurons for Euler-Forward discretization (left), SE (right),  and the continuous model (middle). The horizontal dotted line at $r=1.0$ marks the stability bound. Instances above this line diverge due to exponential growth.  Parameter ranges are uniformly distributed over the intervals $a \in [0,120]$, $\tau_u \in [5,25]$ ms and $\tau_w \in [60,300]$ ms. 
    \textbf{d)} Eigenvalues of $\eulerA$ (left) and $\seA$ (right) plotted in the complex plane for fixed $\tau_u=25$ ms, $\tau_w=60$ ms and varying $a\in[10,800]$. Decay rate $r$ as modulus of the eigenvalue $\lambda_1$ and angle $\phi$ as argument of $\lambda_1$ are shown for $a=282$, marked with * and ** for EF and SE respectively. The grey half-circle denotes the stable region of $r \le 1$. 
    \textbf{e)} Relationship of parameter $a$  to intrinsic frequency $f$ (top) and decay rate $r$ (bottom) for the same $\tau_u$ and $\tau_w$ as in panel d. Points with * and ** denote the corresponding eigenvalues from panel b. Recall the linear relationship $f=\frac{\phi}{2\pi\Delta t}$ between angle $\phi$ and $f$ of the discrete models. Horizontal grey line in bottom panel denotes stability boundary of $r=1$. Values for $r$ of continuous model ($r=0.897$) and SE-adLIF ($r=0.887$) are constant w.r.t.~$a$ (SE-adLIF values for $r$ not visible due to near-perfect fit to the continuous model).
    \textbf{f)} Maximum admissible frequency for stable dynamics for Euler-Forward discretization with $\Delta t = 1$ ms over different values of $\tau_u$ and $\tau_w$, where $a$ is set to the maximum stable value $a_\mathrm{max}$ (see main text and Section \enameref{subsubsec:meth_frequency_ranges} in \meth).}
    \label{fig:stability}
\end{figure}

As described above, the eigenvalues $\lambda_{1,2}$ of state transition matrices $\eulerA$ and $\seA$ are determining the stability of the discrete neurons. We can directly observe the origin of instability for the EF-adLIF by plotting eigenvalues for different neuron parameters in the complex plane.
In Fig.~\ref{fig:stability}d, we show some eigenvalues of $\eulerA$ and $\seA$ in the case of fixed time constants $\tau_u$ and $\tau_w$ and varying parameter $a$. In the complex plane, the stability boundary appears as circle, separating the stable ($r<1$, grey area) from the unstable region ($r>1$).
Our analysis in Section \enameref{meth:euler_forward_stab} in \meth~shows that for the Euler-Forward method, for fixed time constants $\tau_u$ and $\tau_w$, the real part $\Re(\lambda_{1,2})$ of these eigenvalues is constant and strictly positive with respect to $a$, such that the eigenvalues are aligned along a vertical line in the right half-plane of the complex plane (see Fig.~\ref{fig:stability}d, left panel). As $a$ increases, the imaginary part increases and so does the decay rate $r$.
\revision{As already mentioned, in the continuous adLIF model, parameter $a$ only affects the frequency of oscillation, but not the decay rate. The SE-adLIF model adheres to this property, since the modulus of the eigenvalues does not change with respect to $a$. For EF-adLIF however, parameter $a$ exhibits an undesired side-effect on the modulus. }
Hence, for the Euler-Forward discretized neuron, the eigenvalues overshoot the stability boundary $r=1$ for increasing $a$. This leads to a drastically reduced range of the angle $\phi$ and therefore a reduced range of the intrinsic frequency $f$ where the neuron is stable.
In contrast, for the SE discretized model, the parameter $a$ controls only the angle $\phi$ of the eigenvalues (see Section \enameref{meth:se_stab} in \meth) and hence the intrinsic frequency $f$, but not the decay rate $r$. The decay rate is given by $r=\sqrt{\alpha \beta}$ (see Eq.~\eqref{eq:se_r_bound_complex}) and is hence guaranteed to stay within the stability bound $r<1$ for all $\tau_w>0$ and $\tau_u>0$, see Fig.~\ref{fig:stability}d, right panel and Fig.~\ref{fig:stability}e.

We analytically calculated stability bounds for both the Euler-Forward and SE discretization, see \meth~Sections \enameref{meth:euler_forward_stab} and \ref{meth:se_stab}. For each given tuple of time constants $\tau_u$ and $\tau_w$ we calculated a corresponding $a_\mathrm{max}$, that is, the maximum value of the parameter $a$ for which the model is still stable.  
This analysis shows that the SE discretization allows the neuron to utilize the full frequency bandwidth up to the Nyquist frequency at $\frac{1}{2 \Delta t}$, at which aliasing occurs. Since we used a discretization time step of $\Delta t = 1$ ms for Fig.~\ref{fig:stability}, the Nyquist frequency is $500$ Hz. Theorem \ref{theo:se_stab} below summarizes the full frequency coverage of SE-adLIF and the stability within this frequency range (see  \meth, Section \enameref{meth:proof_theo_se_stab} for a proof).
\begin{theorem}
\label{theo:se_stab}
Let $(\tau_u,\tau_w,a)$ be the parameters of an SE-adLIF neuron according to Eq.~\eqref{eq:se_lti}. For any frequency $f \in (0, f_N]$ where $f_N = \frac{1}{2 \Delta t}$ is the Nyquist frequency, and for any $\tau_u, \tau_w > 0$, there exists a unique parameter $a$ such that the neuron has intrinsic frequency $f$. For any such parameter combination, the neuron in the sub-threshold regime is asymptotically stable with decay rate $r = \sqrt{\alpha\beta} < 1$ where $\alpha = e^{-\frac{\Delta t}{\tau_u}}$ and $\beta = e^{-\frac{\Delta t}{\tau_w}}$.
\end{theorem}

The upper frequency bound of adLIF neurons using the Euler-Forward discretization is illustrated in Fig.~\ref{fig:stability}f.  We can observe that for the Euler-Forward method, the maximum admissible frequency for stable dynamics converges toward zero as $\tau_u$ and $\tau_w$ increase (see Section \enameref{subsubsec:meth_frequency_ranges} in \meth~for proof). 

\revision{
An immediate advantage of using the SE-adLIF over EF-adLIF is the guaranteed stability over the entire range of possible oscillation frequencies. This property in particular comes into play when the neuron parameters $\tau_u$, $\tau_w$, $a$, and $b$  are trained. While for most tasks only a sub-range of this viable frequency range might be required, it is guaranteed that the SE-adLIF neuron is stable for any such frequency. \revision{In other words, for SE-adLIF neurons with frequencies below the Nyquist frequency, no unstable parameter configurations exists.} This is not the case for the EF-adLIF neuron: even in instances of very low oscillation frequencies the stability boundary might be overshot (see Fig. \ref{fig:stability}c,d). Later in the manuscript (see Section \enameref{sec:spring-mass} and Fig.~\ref{fig:oscillator_task}g), we discuss the relationship between the neuron frequency range and the performance in an oscillator regression task. 
In our simulations, we clip $a$ to fixed constant upper and lower bounds, given by task-dependent hyperparameters, independent of parameters $\tau_u$ and $\tau_w$.  While this constraint suffices for most tasks, it introduces a trade-off between the decay of the neuron and the oscillation frequency. This can be observed in Fig.~\ref{fig:stability}c (right), where neurons with a high frequency are restricted to a fast decay. 
A trivial extension to the SE-adLIF model would be to dynamically adjust the upper bound for parameter $a$ with respect to parameters $\tau_u$ and $\tau_w$ by computing $a_{\mathrm{max}}^\text{SE}$ (as defined by Eq.~\eqref{eq:a_se_max} in \meth) after each training step and clipping $a$ to the interval $[0, a_{\mathrm{max}}^\text{SE}]$. This would allow the neuron model to utilize the entire frequency range for any combination of $\tau_u$ and $\tau_w$. This extension is not possible for the EF-adLIF neuron, since using $a_{\mathrm{max}}^\text{EF}$ (as defined by Eq.~\eqref{eq:ef_complex_upper_bound} in \meth) as an upper bound instead of a constant value would still not allow the neuron to use the entire frequency range. This exact case, where $a$ is set to $a_{\mathrm{max}}^\text{EF}$, is shown in Fig.~\ref{fig:stability}f.
}

\revision{The favorable stability properties induced by the SE discretization should generalize well to other neuron models with two bi-directionally coupled neuron states. Two examples for such neuron models are the adaptive exponential integrate-and-fire (AdEx) \cite{bretteAdaptiveExponentialIntegrateandFire2005} model  and the Balanced Harmonic Resonate-and-Fire (BHRF) \cite{higuchi2024balanced} model. For the latter, we observed that applying the SE-discretization not only alleviates the necessity of the frequency-dependent divergence boundary, which was introduced by the authors to ensure stability of the model, but also recovers the direct relationship between neuronal parameters $\omega$ and $b$ and the effective oscillation frequency $\omega_\text{eff}$ and effective damping coefficient $b_\text{eff}$ of the discretized neuron. Details can be found in \nameref{supnotes1} and \nameref{supfig1}. }
\subsection*{Improved performance of SE-discretized adaptive RSNNs}
\revision{
We first evaluated how recurrent networks of the described adLIF neurons perform in comparison to classical vanilla LIF networks.
We compared LIF and SE-adLIF networks} on two commonly used audio benchmark datasets: Spiking Heidelberg Digits (SHD) \cite{cramerHeidelbergSpikingData2022} and Spiking Speech Commands (SSC) \cite{cramerHeidelbergSpikingData2022}
, as well as an ECG dataset previously used to test ALIF neurons \cite{yinAccurateEfficientTimedomain2021}. \revision{For LIF baselines, we used results from previously reported studies as well as additional simulations to ensure identical setups and comparable parameter counts. }

We obtained our results by constructing a recurrently connected SNN composed of one or two layers (depending on the task) of adLIF (resp. LIF) neurons, followed by a layer of leaky integrator (LI) neurons to provide a real-valued network output. We trained both the adLIF and LIF SNNs using BPTT with surrogate gradients \cite{leeTrainingDeepSpiking2016,bellec2018long,neftci2019surrogate,slayer}. We used a dropout rate of $15\%$ except noted otherwise, but otherwise no regularization, normalization, or data augmentation methods. The trained parameters included the synaptic weights $\theta$, as well as all neuron parameters $a$, $b$, $\tau_u$ and $\tau_w$ in the SE-adLIF case, and membrane time constants $\tau$ in the LIF case. The parameters were not shared across neurons such that each neuron could have individual parameter values. Neurons were initialized heterogeneously, such that for each neuron the initial values of these parameters were chosen randomly from a uniform distribution over a pre-defined range. Heterogeneity has previously been shown to improve the performance of SNNs \cite{perez-nievesNeuralHeterogeneityPromotes2021}. We applied a reparametrization technique for the training of time constants $\tau_u$ and $\tau_w$ and parameters $a$ and $b$ for the SE-adLIF model, as well as the membrane time constants $\tau$ for the LIF models, see \meth~for details.

\revision{Table \ref{tab:lif_results} summarizes both the baselines from prior studies and our own results. While LIF networks in our simulations performed better than all previously reported LIF baselines, they still performed significantly worse than the SE-adLIF networks with the same or lesser parameter counts across all tasks. This result provides clear empirical support for the superiority of adLIF networks over LIF networks, both in terms of parameter efficiency and overall performance.} 
\begin{table}[t]
\centering
\caption{\revision{\textbf{Comparison of recurrent LIF and adLIF networks on spike-encoded speech recognition datasets.} The 'N.A.' indicates lack of information from the authors. 'This work' indicates our experimental results. The reported results correspond to the accuracy on the corresponding test set.}
}
\begin{tabular}{cllll}
\text{} & \text{Model} & \text{Publication} & \text{\#Params} & \text{Test Acc. [\%]} \\  \hline
\multirow{6}{*}{\rotatebox[origin=c]{90}{SHD}} 
& \text{LIF} & Cramer et al. 2022 \cite{cramerHeidelbergSpikingData2022} & N.A. & $83.2 \pm 1.3$ \\ 
& \text{LIF} & Deckers et al. 2024 \cite{deckers2024co} & $37.9$k & $84.49$ \\ 
& \text{LIF} & Bittar \& Garner  2022 \cite{bittar2022surrogate} & $141$k & $87.04$ \\ 
& \text{LIF} & Bittar \& Garner  2022 \cite{bittar2022surrogate} & $3.8$M & $89.29$ \\ 
& \text{LIF} & This work & $450$k & $90.27 \pm 0.73$  \\ 
& \textbf{SE-adLIF} & This work & $450$k & $\mathbf{95.81 \pm 0.56}$ \\ \hline

\multirow{6}{*}{\rotatebox[origin=c]{90}{SSC}} 
& \text{LIF} & Cramer et al. 2022 \cite{cramerHeidelbergSpikingData2022} & N.A. & $50.9 \pm 1.1$ \\ 
& \text{LIF} & Deckers et al. 2024 \cite{deckers2024co} & $0.34$M & $71.76$ \\ 
& \text{LIF} & Bittar \& Garner  2022 \cite{bittar2022surrogate} & $141$k & $66.67$ \\ 
& \text{LIF} & Bittar \& Garner  2022 \cite{bittar2022surrogate} & $1.1$M & $68.14$ \\ 
& \text{LIF} &  This work  & $1.7$M  & $75.23$ \\ 
& \textbf{SE-adLIF} & This work &  $1.6$M & $\mathbf{80.44 \pm 0.26}$ \\ \hline

\multirow{2}{*}{\rotatebox[origin=c]{90}{ECG}} 
& \text{LIF} & This work & $1.8$k & $77.8 \pm 1.85$ \\ 
& \textbf{SE-adLIF} & This work & $1.8$k & $\mathbf{86.88 \pm 0.40}$ \\ \hline
\end{tabular}
\label{tab:lif_results}
\end{table}

In the previous sections, we discussed the theoretical advantages of SE discretization over the more commonly used Euler-Forward discretization of the adLIF neuron model. \revision{The discussed theoretical advantages of SE, for example near-independence of the neuron dynamics from the discretization time step $\Delta t$ or the closer adherence to the continuous model, give indirect practical benefits when dealing with such neurons. It is not clear however, whether the utilization of the SE discretization can provide an improvement in performance over the commonly used EF method. To answer this question, we not only compared both the EF-adLIF and the SE-adLIF model to each other, but also to state-of-the-art spiking neural networks:} a model with threshold adaptation (ALIF) \cite{yinAccurateEfficientTimedomain2021}, a constrained variant of the adLIF neuron model, similar to the model in this study, but with  differences in discretization and neuron formulation (cAdLIF) \cite{deckers2024co}, another adLIF network but with batch normalization (RadLIF) \cite{bittar2022surrogate}, a feed-forward model with delays implemented as temporal convolutions (DCLS-Delays) \cite{hammouamri2023learning}, and the balanced resonate-and-fire neuron model (BHRF) \cite{higuchi2024balanced}, a variant of the resonate-and-fire neuron \cite{izhikevich2001resonate} where output spikes do not disrupt the phase of the membrane potential oscillation. 

In Table \ref{tab:results} we report the test accuracy \revision{of the various models} on the corresponding test sets. SHD does not define a dedicated validation set and previous work reported performances for networks validated on the test set, which is methodologically not clean. We therefore report results for two validation variants for SHD: with validation on the test set (to ensure comparability) and with validation on a fraction of the training set.
\begin{table}[t]
\caption{\textbf{Comparison of accuracies of different models on the test sets of SHD, SSC and ECG.}  The asterisk (*) of SHD* denotes that a fraction of the training data was used as a validation set. BHRF used $5\%$, whereas we used $20\%$ of the training data for validation. We report the test accuracy for the training epoch with the highest validation accuracy. For our model, SE-adLIF, we report the mean and standard deviation on $20$ different random initializations. Rec.: model contains recurrent synaptic connections; \#Runs: number of runs with different random seeds to obtain the results. Test accuracy is provided as mean $\pm$ SD.  
}
\centering
\begin{tabular}{clllll}
\text{} & \text{Model} & \text{Rec.} & \text{\#Params} & \text{\#Runs} & \text{Test Acc. [\%]} \\ \hline
\multirow{7}{*}{\rotatebox[origin=c]{90}{SHD}} 
& \text{LIF} & \cmark & $0.45$M & $10$ & $90.27 \pm 0.73$ \\ 
& \text{ALIF \cite{yinAccurateEfficientTimedomain2021}} & \cmark & $0.14$M & $1$ & $90.4$ \\ 
& \text{cAdLIF \cite{deckers2024co}} & \xmark & $38.7$k & $10$ & $94.19$ \\ 
& \text{RadLIF \cite{bittar2022surrogate}} & \cmark & $3.9$M & - & $94.62$ \\ 
& \text{DCLS-Delays \cite{hammouamri2023learning}} & \xmark & $0.2$M & $10$ & $95.07 \pm 0.24$ \\ 
& \text{EF-adLIF (2 layers)} & \cmark & $0.45$M & $20$ & $94.68 \pm 0.57$ \\
& \text{SE-adLIF} (1 layer) & \cmark & $37.5$k & 20 & $94.59 \pm 0.27$ \\
& \text{SE-adLIF} (2 layers) & \cmark & $0.45$M & 20 & $95.81 \pm 0.56$ \\  \hline

\multirow{3}{*}{\rotatebox[origin=c]{90}{SHD*}} 
& \text{BHRF \cite{higuchi2024balanced}} & \cmark & $0.1$M & $5$ & $92.7 \pm 0.7$ \\ 
& \text{SE-adLIF} (1 layer) & \cmark & $37.5$k & 20 & $93.18 \pm 0.74$ \\ 
& \text{SE-adLIF} (2 layers) & \cmark & $0.45$M & 20 & $93.79 \pm 0.76$ \\ \hline

\multirow{7}{*}{\rotatebox[origin=c]{90}{SSC}} 
& \text{LIF} & \cmark & $1.7$M & $1$ & $75.23$ \\ 
& \text{ALIF \cite{yinAccurateEfficientTimedomain2021}} & \cmark & $0.78$M & $1$ & $74.2$ \\ 
& \text{RadLIF \cite{bittar2022surrogate}} & \cmark & $3.9$M & - & $77.4$ \\ 
& \text{cAdLIF \cite{deckers2024co}} & \xmark & $0.35$M & $1$ & $77.5$ \\ 
& \text{DCLS-Delays \cite{hammouamri2023learning}} & \xmark & $2.5$M & $5$ & $80.69 \pm 0.21$ \\ 
& \text{EF-adLIF (2 layers)} & \cmark & $1.6$M & $5$ & $79.80 \pm 0.18$ \\ 
& \text{SE-adLIF (2 layers)} & \cmark & $1.6$M & $5$ & $80.44 \pm 0.26$ \\ \hline

\multirow{7}{*}{\rotatebox[origin=c]{90}{ECG}} 
& \text{LIF (1 layer)} & \cmark & $1.8$k & 5 & $77.8 \pm 1.85$ \\ 
& \text{ALIF \cite{yinAccurateEfficientTimedomain2021}} & \cmark & $1.8$k & $1$ & $85.9$ \\ 
& \text{BHRF \cite{higuchi2024balanced}} & \cmark & $1.8$k & $5$ & $87.0 \pm 0.4$ \\ 
& \text{EF-adLIF (1 layer)} & \cmark & $1.8$k & $20$ & $87.10 \pm 0.46$ \\ 
& \text{SE-adLIF (1 layer)} & \cmark & $1.8$k & $20$ & $86.88 \pm 0.40$ \\ 
& \text{EF-adLIF (2 layers)} & \cmark & $4.6$k & $20$ & $87.83 \pm 0.51$ \\ 
& \text{SE-adLIF (2 layers)} & \cmark & $4.6$k & $20$ & $88.18 \pm 0.36$ \\ \hline
\end{tabular}
\label{tab:results}
\end{table}

For all considered datasets, recurrent SE-discretized adLIF networks performed better than previously considered recurrent SNNs. For SSC, their performance was slightly below that of the DCLS model \cite{hammouamri2023learning}, a feed-forward network using extensive delays trained via dilated convolutions. \revision{Unlike our model, however, DCLS employs temporal convolutions to implement delays and incorporates batch normalization. These properties make the DCLS model less suitable for neuromorphic use cases. Nevertheless, we included it in our results table for comparison, as the delays in neural connectivity provide an interesting orthogonal complement to the enhanced somatic dynamics of the models studied in our work (see also \enameref{sec:discussion}). }
Networks composed of SE-discretized adLIF neurons (SE-adLIF networks) performed significantly better than those based on Euler-forward discretization (EF-adLIF networks) on SHD and SSC (significance values for a two-tailed t-test were $p<0.000001$ for SHD and $p<0.005$ for SSC). Small networks with a single recurrent layer on ECG performed on-par ($p=0.115$), while SE-adLIF networks significantly improved over EF-adLIF networks when larger networks with two recurrent layers were used ($p<0.02$).
We found that EF-adLIF networks suffered from severe instabilities if neuron parameters were not constrained to values in which the decay rate exceeds the critical boundary of $r=1$, resulting in instabilities for example in the ECG task, see \nameref{supnotes2} and \nameref{suptab1}. The SE method is hence the preferred choice when adLIF neurons are used in a discretized form. 

AdLIF neurons could, depending on their parameters, exhibit many different experimentally observed neuronal dynamics \cite{gerstner2014neuronal}. We wondered whether networks trained on spatio-temporal classification tasks utilized the diverse dynamical behaviors of adLIF neurons. To that end, we investigated the resulting parameterizations of adLIF neurons in networks trained on SHD, and indeed found a heterogeneous landscape of neuron parameterizations, see \nameref{supnotes3} and \nameref{supfig2}.
\subsection*{\revision{Accurate prediction of dynamical system trajectories}}
\label{sec:spring-mass}
The benchmark tasks considered above were restricted to classification problems where the network was required to predict a class label. 
We next asked whether the rich neuron dynamics of adaptive neurons could be utilized in a generative mode where the network has to produce complex time-varying dynamical patterns. To that end, we considered a task in which networks had to generate the dynamics
of a system of $4$ masses, interconnected by springs with different spring constants, see Fig.~\ref{fig:oscillator_task}a. 
Each training sequence consisted of the masses' trajectory over time for \qty{500}{\milli\second} (Fig.~\ref{fig:oscillator_task}b) from a randomly sampled initial condition of this $4$-degree-of-freedom dynamical system, where the displacement $x_i$ of each mass $i$ was encoded via a real-valued input current. During the first half of the sequence, the model was trained to produce "single-step predictions", that is, it received the mass displacements $\bm{x}[k] \in \mathbb{R}^4$ as input at each time step $k$ and had to predict the displacements $\bm{x}[k+1]$. In the second half, the model auto-regressed, i.e. it used its own prediction $\hat{x}[k]$ to predict the next state $\bm{x}[k+1]$ (see Fig.~\ref{fig:oscillator_task}c).
Through this second phase, we tested if the network was able to accurately maintain a stable representation of the evolving system by measuring the deviation from the ground truth over time. 
\begin{figure}[!t]
    \centering
    \includegraphics[width=\textwidth]{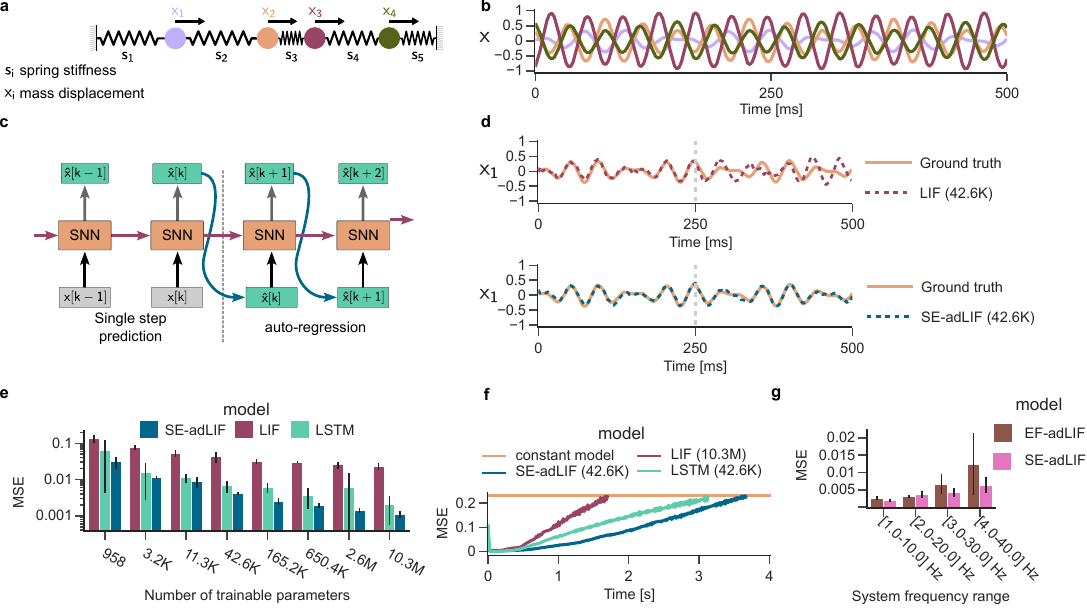}
    \caption{\textbf{Prediction and generation of complex oscillatory dynamics.} \textbf{a)} Schematic of a 4-degree-of-freedom spring-mass system. $x_1$ to $x_4$ represent the displacements of the four masses. \textbf{a)} Example displacement dynamics generated over a period of \qty{500}{\milli\second}.  
    \textbf{c)} Illustration of the auto-regression task. For the first  \qty{250}{\milli\second}, the network received the true displacements $\bm{x}[k]$ and predict the next displacement $\hat{\bm{x}}[k+1]$. After \qty{250}{\milli\second}, the model generates the displacements by using its own predictions from the previous time step in an autoregressive manner. \textbf{d)} Displacement predictions for mass $x_1$, by a LIF (top) and adLIF (bottom) network with $42.6$K trainable parameters.
    \textbf{e)} Mean squared error (MSE) in logarithmic scale during the auto-regression period for LIF, adLIF, and LSTM networks of various sizes (mean and STD over 5 unique randomly generated spring-mass systems).
    \textbf{f)} Divergence of generated dynamics in the auto-regressive phase (starting after \qty{250}{\milli\second}). We report the MSE over time averaged over a \qty{25}{\milli\second} time-window. The constant model corresponds to the average MSE over time for a model that constantly predicts zero as displacement. 
    \textbf{g)} Mean squared error (MSE) during the auto-regression period for adLIF networks discretized with the Euler-Forward (brown) and Symplectic-Euler (pink) method on spring-mass systems with different frequency ranges.} 
    \label{fig:oscillator_task}
\end{figure}
Note that in the spring-mass system the states are described by the displacement and velocity of the masses but only displacement information $\bm{x}[k]$ was available to the network. Hence, it is impossible to accurately predict the displacements of the masses at time $k+1$ from the displacements at time $k$ alone.
The network must therefore learn to keep track of the longer-time dynamics of the system.

In Fig.~\ref{fig:oscillator_task}d, we show the ground truth of displacement $x_1$ for mass $1$, as well as the prediction of the displacement by an SE-adLIF network with a single hidden layer of $200$ neurons and a single-layer LIF network with the same number of parameters. 
After time step $t=\qty{250}{\milli\second}$, the auto-regression phase starts. These plots exemplify how the LIF network roughly followed the dynamics during the one-step prediction phase, but gradually diverged from the target in the auto-regression phase. In contrast, the generated trajectory of the SE-adLIF network stayed close to the ground truth system throughout the auto-regression phase. 
Fig.~\ref{fig:oscillator_task}e shows the mean squared error (MSE) of several models and model sizes during this autoregressive phase. SE-adLIF networks consistently outperformed LIF networks as well as non-spiking long-short-term memory (LSTM) networks (note the log-scale of the y-axis). 
Moreover, we observed that their performance scaled better with network size ($1.7$ and $1.3$ MSE improvement factor per doubling of the network size for SE-adLIF and LIF networks respectively). 
Fig.~\ref{fig:oscillator_task}f shows how fast the models degrade towards the baseline of a model that constantly outputs zero. We observe that small SE-adLIF networks with $200$ neurons approximated the trajectory of the dynamical system in the auto-regression phase for a much longer duration than the best LIF network with $3200$ neurons. 
Interestingly, when we trained SE-adLIF networks without recurrent connections, their dynamics degraded clearly slower than LIF networks with recurrent connections (\nameref{supfig3}), which underlines the utility of the inductive bias of oscillatory neurons for such generative tasks.

Additionally, we used this setup to compare the Symplectic-Euler discretization (SE-adLIF networks) with the Euler-Forward discretization (EF-adLIF networks). Since in this task the frequency bandwidth can be controlled directly via the spring coefficient, we generated spring-mass systems of increasing maximal frequency.
We trained EF-adLIF networks and SE-adLIF networks under the same range of time-constants ($\tau_u$ and $\tau_w$) and a restricted range for the adaption parameter $a$.
For Euler-Forward, $a$ was restricted between $0$ and $a_\mathrm{max}$, where $a_\mathrm{max}$ is the maximal parameter value for $a$ that is stable under this discretization, resulting in a $[0, 30]$ Hz range of frequencies that can be represented by the neurons for the chosen range of time-constants.
For Symplectic-Euler, all frequencies below the Nyquist frequency are stable, so we simply choose $a_{\mathrm{max}}$ to achieve a frequency range of $[0, 60]$ Hz. 
The results are shown in Fig.~\ref{fig:oscillator_task}g. As expected, the two methods have similar performance at low frequencies. For dynamics with a larger frequency bandwidth however, EF-adLIF networks performed significantly worse. Additionally, the increased variance of the error indicates stability problems.
These experimental results support our claim that the wider stability region of the SE-adLIF network allows the model to converge over a wide range of data frequencies.

\subsection*{\revision{High-fidelity neuromorphic audio compression}}
\revision{
In all experiments so far we observed a significant superiority of the adLIF neuron over the LIF neuron, both in terms of parameter efficiency and overall performance. Yet, it is unclear how these observed improvements transfer from benchmarks and toy tasks to real-world neuromorphic applications. To take a step towards answering this question, we compared the performance of adLIF and LIF neurons in the task of raw audio compression. The goal of this task is to first compress and then transmit a raw audio signal as energy-efficient as possible, while sacrificing as little signal quality as possible. Similar setups with neuromorphic processing and spike-based transmission have been discussed in various studies \cite{rothSpikeBasedSensingCommunication2022,wuNeuromorphicWirelessSplit2024,ke2024,chenNeuromorphicWirelessCognition2023a,savazziMutualInformationAnalysis2024} as promising research direction for low-power IoT applications with smart wireless sensors. Our study is, to the best of our knowledge, the first to address audio compression using plain RSNNs that deliberately exclude batch normalization, temporal convolutions, and transformer-based architectures to maintain compatibility with standard neuromorphic processors. 
}

\begin{figure}[!t]
    \centering
    \includegraphics[width=0.5\textwidth]{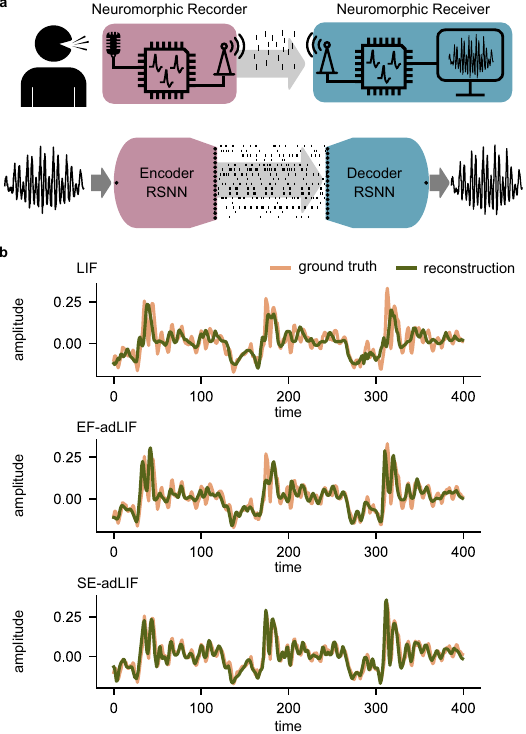}
    \caption{\revision{\textbf{Audio compression task setup and results.} \textbf{a)} (top) The conceptual hardware setup considered in the audio compression task. See main text and \meth~ for details. (bottom) Architecture of our simulated SNNs, closely resembling the conceptual neuromorphic setup. The raw waveform input to the encoder RSNN, the $16$-dimensional spike train between encoder and decoder, as well as the reconstructed waveform were taken from our simulations with SE-adLIF networks. \textbf{b)} Comparison between the ground truth and the reconstructed waveform for the LIF (top), EF-adLIF (middle), and SE-adLIF network (bottom). }}
    \label{fig:audio_compression_task}
\end{figure}
\revision{
The setup is illustrated in Fig.~\ref{fig:audio_compression_task}a: We conceptually consider a small device containing a digital neuromorphic chip that receives raw, unprocessed audio from a microphone. For our study, we used SNN simulations and did not implement this setup in real hardware, but simulated the SNNs that would run on such chips, see Fig.~\ref{fig:audio_compression_task}a, bottom. In the conceptual setup, the chip processes the waveform by implementing a small SNN with a bottleneck output layer consisting of very few neurons. It sends the spike-encoded audio data through a sparse wireless communication channel \cite{savazziMutualInformationAnalysis2024} to a receiving device. This receiving device, a second neuromorphic processor, could in principle perform arbitrary post-processing on the spike-encoded data. In our simulations we considered the most general case, which is the reconstruction of the ground truth waveform from the sparse spikes. For the spike-based communication we assume low-latency pulse-driven radio transmission, for example IR-UWB \cite{fernandesRecentAdvancesIRUWB2010}, that features adaptive energy consumption, depending on the presence of input signal. In the absence of an input signal (i.e. silence), almost no energy is consumed by the transmitting device, hence this technology is a promising candidate for ultra low-power neuromorphic sensing devices \cite{chenNeuromorphicWirelessCognition2023a}. In contrast, if conventional frame-based digital transmission is used, the transmission rate is constant and power is consumed at a constant rate. In pulse-driven spike encoding, the timing of spikes is implicitly encoded by the timing of the emitted radio pulse, such that the spike timing does not need to be explicitly transmitted as payload. }

\revision{Audio compression requires balancing the quality of the reconstructed signal against the data transmission rate at the bottleneck. In our simulations, we constrained the encoder SNN to very sparse spiking activity at the output layer to achieve low-bandwidth transmission. This aligns with our goal of ultra-low-energy processing, as energy consumption in neuromorphic systems is directly tied to spike rate. 
We considered as few as $16$ output neurons for the encoder, regularized to not exceed a total maximum spike rate of $6$k spikes per second. Assuming that a single spike is equivalent to $1$ bit, the upper regularization bound for the average data transmission rate between encoder and decoder is $6$ kbps. Under this constraint, we compared the quality of the reconstructed audio signal of LIF and adLIF neurons with common audio compression codecs \cite{valin2012definition,dietz2015overview} and the state-of-the-art Residual Vector Quantization (RVQ) method \cite{defossez2022high} evaluated at the same bandwidth of $6$ kbps. The results are shown in Table \ref{tab:speech_task_model_comparison}.
Details on the task setup and the simulations can be found in Section \enameref{subsec:audio_compression_details} in \meth. 
We provide uncompressed audio samples (Supplementary Audio 1) and reconstructions using the LIF (Supplementary Audio 2), EF-adLIF (Supplementary Audio 3), and SE-adLIF networks (Supplementary Audio 4) as supplements to this article.
Our findings indicate that adaptive LIF neurons achieve a significantly higher reconstruction quality than vanilla LIF neurons (see Fig.~\ref{fig:audio_compression_task}b for an example). Moreover, SE-discretized adLIF neurons yield significantly higher reconstruction quality, as demonstrated by both quantitative metrics (Table \ref{tab:speech_task_model_comparison}) and visual waveform comparisons (Fig.~\ref{fig:audio_compression_task}b). Given that waveform data spans a broad frequency spectrum, the stability of the SE-adLIF neurons over the entire frequency range --- as discussed in Section \enameref{subsec:stability} --- offers a clear advantage over EF-adLIF. Since LIF neurons lack intrinsic membrane potential oscillations, they depend heavily on recurrent network dynamics to detect, encode and decode oscillatory patterns in their input. We observed the same phenomenon in the oscillatory task in Fig.~\ref{fig:oscillator_task}, where LIF neurons performed poorly. 
While the SE-adLIF model achieved the best performance in terms of scale-invariant signal-to-noise ratio \cite{luo2018tasnet},
the state-of-the-art neural network model RVQ exhibits better performance on the VISQOL measure \cite{hines2015visqol}, but at the cost of a much larger model with a $>27 \times$ increase in the number of parameters. The compactness of the SE-adLIF model allowed audio compression and decoding in $1.3 \times$ real-time on a consumer single tread CPU (AMD Ryzen 7 5800H).
In Table \ref{tab:speech_task_model_comparison}, we also report performances of two standard audio codecs (OPUS and EVS) on our test set, showing that the SE-adLIF networks achieve competitive performance.
In the next few sections, we further explore the inductive bias introduced by oscillatory membrane potentials. 
}
\begin{table}[t]
\centering
\caption{\revision{\textbf{Performance comparison between LIF, EF-adLIF, and SE-adLIF networks with state-of-the-art audio codecs in the audio compression task.} We used the same network layout for all neuron models and report mean and standard devivation on the LibriTTS clean test set \cite{zen2019libritts} across $4$ runs except for LIF models, for which only 2 runs converged and were taken into account. 
See Section \enameref{subsec:audio_compression_details} in \meth~ for details. SI-SNR denotes the scale-invariant signal-to-noise ratio \cite{luo2018tasnet}, VISQOL the value obtained from the Virtual Speech Quality Objective Listener \cite{hines2015visqol}. OPUS and EVS are audio codecs, hence no parameter counts are provided. 
Sam.~Rate refers to the respective sampling rate of the audio signal. Avg.~Bandw.~refers to the average number of spikes per second (Sp./s) at the bottleneck layer for the spiking networks, and the data transmission rate for the other, frame-based methods.}}
\label{tab:speech_task_model_comparison}
\begin{tabular}{clllll}
\text{Model} & Sam. Rate & \text{\#Params} & Avg. Bandw. & \text{SI-SNR \cite{luo2018tasnet}} & \text{VISQOL \cite{hines2015visqol}} \\  \hline
LIF &  \qty{24}{\kHz} & 543k & $4.6$ kSp./s & $6.90 \pm 0.05$ & $3.46 \pm 0.01$ \\

EF-adLIF & \qty{24}{\kHz}& 543k & $5.0$ kSp./s & $7.80 \pm 0.04$ & $3.82 \pm 0.01$ \\
SE-adLIF & \qty{24}{\kHz} & 543k & $5.2$ kSp./s & $8.39 \pm 0.05$ & $3.91 \pm 0.004$ \\

RVQ  \cite{defossez2022high} & \qty{24}{\kHz} & 14.8M & $6$ kbps & $4.14$   & $4.38$  \\ 
OPUS  \cite{valin2012definition} & \qty{24}{\kHz}& - &  $6$ kbps & $1.35$ & $3.69$  \\
EVS  \cite{dietz2015overview} & \qty{16}{\kHz} & - &  $6$ kbps & $5.51$& $4.11$\\
\hline
\end{tabular}
\end{table}

\subsection*{Computational properties of adLIF networks}
\label{subsec:inductive_bias}

Our empirical results above demonstrate the superiority of oscillatory neuron dynamics over pure leaky integration in spiking neural networks, which is in line with prior studies \cite{bellec2018long,salaj2021spike,bittar2022surrogate,gangulySpikeFrequencyAdaptation2024a,higuchi2024balanced}. In the following, we analyze the reasons behind this superiority.

\paragraph{Adaptation provides an inductive bias for temporal feature detection.}
%
In gradient-based training of neural networks, the gradient determines to which features of the input a network 'tunes' to. Hence, understanding how gradients depend on certain input features contributes to the understanding of network learning dynamics. When a recurrent SNN is trained with BPTT, the gradient propagates through the network via two different pathways: the recurrent synaptic connections and the implicit neuron-internal recurrence of the neuron state $s[k]$.
\begin{figure}[!t]
    \centering
    \includegraphics[width=\textwidth]{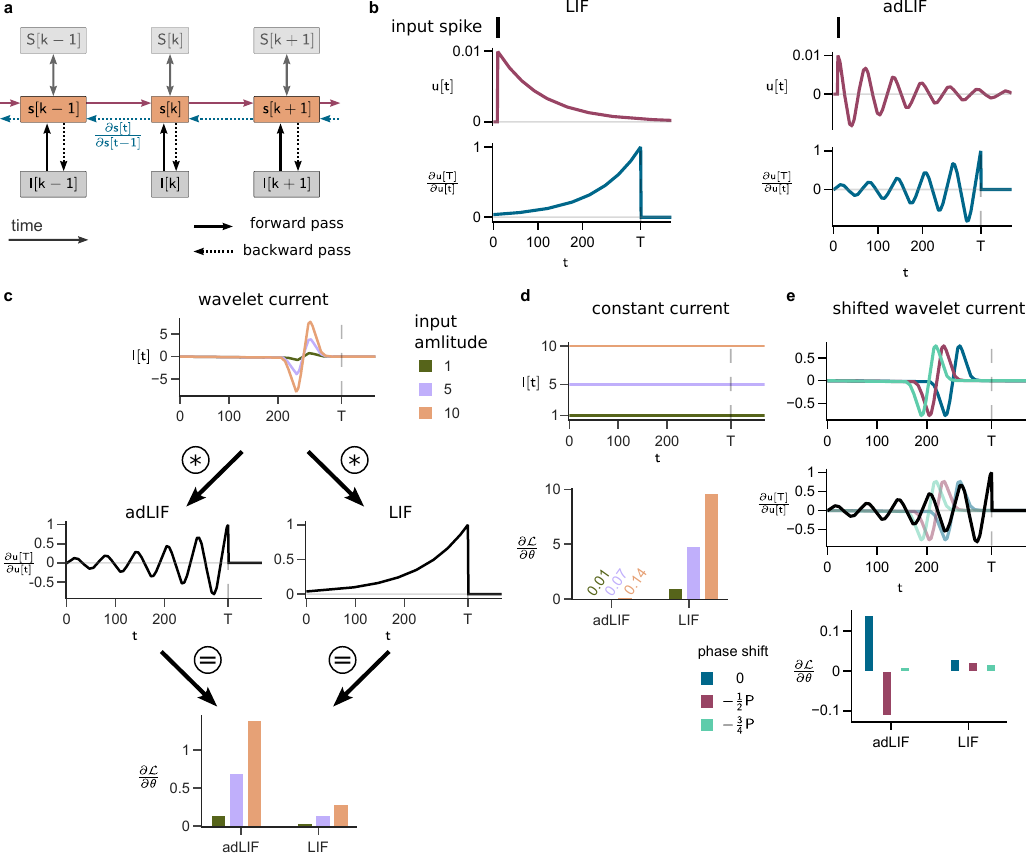}
    \caption{\textbf{Inductive bias of adLIF gradients.}  \textbf{a)} Computational graph showing the state-to-state derivative $\frac{\partial s[t]}{\partial s[t-1]}$ back-propagating through time. $\mathbf{s}[k]$ denotes the state vector Eq.~\eqref{eq:discrete_lds} \textbf{b)} Response of the membrane potential of a LIF neuron (left) and an adLIF neuron (right) to a single input spike. The shape of the derivative $\frac{\partial u[T]}{\partial u[t]}$ (bottom) matches the reversed impulse response function. \textbf{c)} Comparison of the derivative $\frac{\partial \mathcal{L}}{\partial \theta}$ for a wavelet input current. The multiplication from Eq.~\eqref{eq:gradient} of the input current with the state derivative is schematically illustrated for both, the adLIF and the LIF case. The frequency of the wavelet approximately matches the intrinsic frequency of the oscillation of the membrane potential oscillation of the adLIF neuron. The bar plot on the bottom shows the derivative $\frac{\partial \mathcal{L}}{\partial \theta}$ for both neurons, where color indicates input amplitude. \textbf{d)} Same as panel c but for a constant input current. \textbf{e)} Same as panel c but for different positions of the wavelet current. Middle plot shows the alignment between input and back-propagating derivative $\frac{\partial u[T]}{\partial u[t]}$ for the adLIF neuron. Input wavelet is given with a phase shift of $0$,$-\frac{1}{2}P$ and $-\frac{3}{4}P$ with respect to the period $P$ of the adLIF neuron oscillation.}
    \label{fig:gradient_magnitude}
\end{figure}
For the following analysis, we ignored the explicit recurrent synaptic connections and focused on how the backward gradient of the neuron state  determines the magnitude of weight updates in different input scenarios, see Fig.~\ref{fig:gradient_magnitude}a.

Consider the derivative $\frac{\partial u[T]}{\partial u[k]}$ of the membrane potential at a time step $T$ with respect to the membrane potential at some prior time step $k$. Intuitively, this derivative indicates how small perturbations of $u[k]$ influence $u[T]$. 
Fig.~\ref{fig:gradient_magnitude}b
shows this derivative for a LIF neuron (left) and an adLIF neuron (right). Because this derivative is the reverse of the model's forward impulse response, it exhibits oscillations in the case of the adLIF neuron and reversed leaky integration for a LIF neuron. Consider a LIF or an adLIF neuron with a single synapse with weight $\theta$ and input $I[k]$. A loss signal $\frac{\partial L}{\partial u[T]}$ (set to $1$ in our illustrative example) is provided at time-step $T$. The resulting gradient $\frac{\partial L}{\partial \theta}$, used to compute the update of synaptic weight $\theta$, is given by
\begin{equation}
\label{eq:gradient}
    \frac{\partial L}{\partial \theta} \propto \frac{\partial L}{\partial u[T]} \sum_{k=1}^T \frac{\partial u[T]}{\partial u[k]} I[k].
\end{equation}
This equation makes explicit that the weight change is proportional to the correlation between the input currents $I(T), I(T-1), I(T-2), \dots$ and the internal derivatives $\frac{\partial u[T]}{\partial u[T-1]}, \frac{\partial u[T]}{\partial u[T-2]}, \frac{\partial u[T]}{\partial u[T-3]}, \dots$. 
This is illustrated in Fig.~\ref{fig:gradient_magnitude}c for 
temporal input currents --- realized as wavelets --- at different amplitudes. The magnitudes of the resulting gradients for the LIF and adLIF model are quite complementary. 
The correlation between the wavelet and the oscillations of the adLIF neuron's state-derivative results in a strongly amplitude-dependent gradient.
In contrast, the leaky integration of the LIF neuron averages the positive and negative region in the input wave. 
The situation changes drastically for a constant input current, see Fig.~\ref{fig:gradient_magnitude}d. 
For LIF neurons, the gradient $\frac{\partial L}{\partial \theta}$ strongly increases with increasing input current magnitude in this scenario. In contrast, the gradient of the adLIF neuron only weakly depends on the magnitude of the constant input current, in fact the gradient is nearly non-existent. This can be explained by the balance between positive and negative regions of the adLIF gradient (compare with Fig.~\ref{fig:gradient_magnitude}b), resulting in almost zero if multiplied with a constant and summed over time. 

The temporal sensitivity of the adLIF gradient is even more evident when we consider the gradients for different positions of a wavelet current (Fig.~\ref{fig:gradient_magnitude}e). The sign and magnitude of the gradient strongly depend on the position of the wavelet for the adLIF neuron, but not for the LIF neuron. If the wavelet input is aligned with the oscillation of the back-propagating derivative, the resulting gradient is strongly positive. For a half-period ($-\frac{1}{2}P$) phase shift, the gradient is negative and for a $-\frac{3}{4}P$ phase shift, the resulting gradient is low in magnitude due to misalignment of oscillation and input. This gradient encourages the neuron to detect temporal features in the input, that is, temporally local changes in the input, either as changes in the spike rate (e.g. Fig.~\ref{fig:fig2}d) or in the input current (e.g. Fig.~\ref{fig:gradient_magnitude}c,e) with specific timing.
This sensitivity hence provides an inductive bias for spatio-temporal sequence processing tasks. 
\paragraph{Networks of adLIF neurons tune to high-fidelity temporal features.}
Through the rich dynamics and the consequential inductive bias towards learning temporal structure in the input, adLIF neurons should be well-suited for tasks in which spatio-temporal feature extraction is necessary. In order to investigate how well temporal input structure can be exploited by networks of adLIF neurons as compared to networks of LIF neurons, we considered a conceptual task that can be viewed as prototypical temporal pattern detection. We refer to this task as the burst sequence detection (BSD) task.

In the BSD task, a network has to classify temporal patterns of bursts from a population of $n$ input neurons, see Fig.~\ref{fig:optimization_based_sampling}a.
\begin{figure}[!t]
    \centering
    \includegraphics[width=\textwidth]{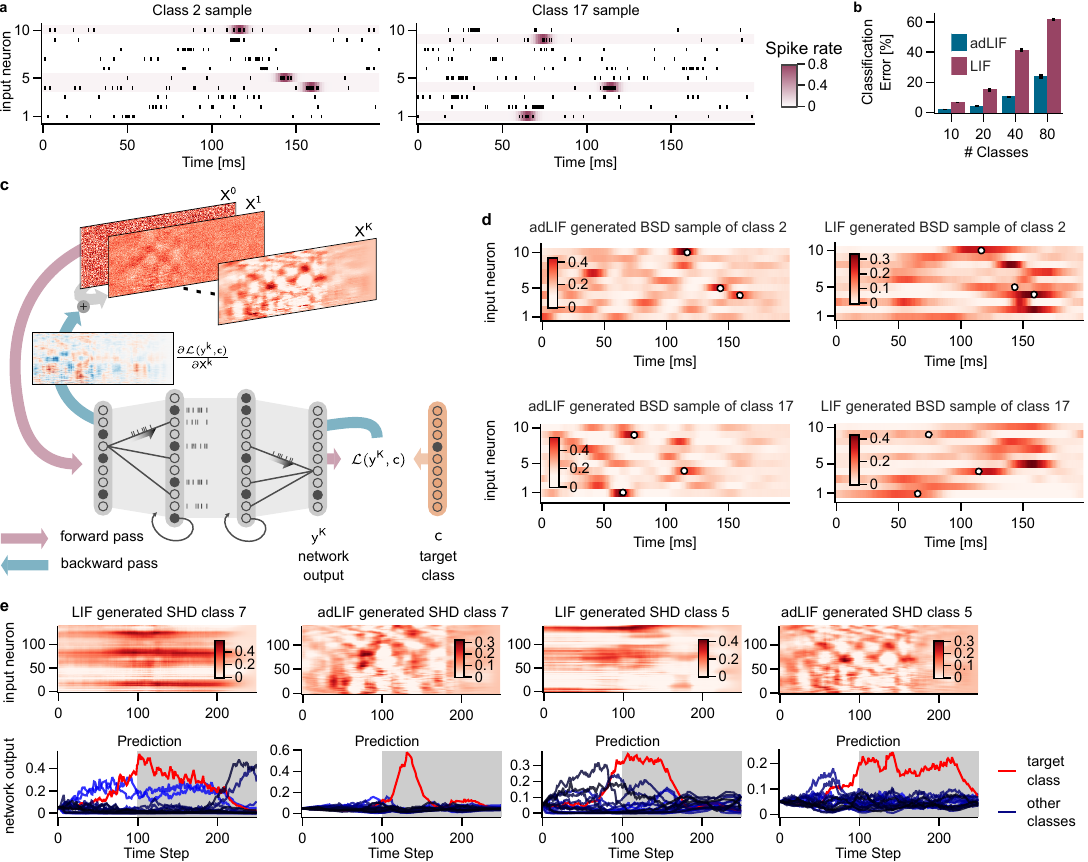}
    \caption{\textbf{Temporal feature detection in adLIF networks.}  \textbf{a)} Two samples of classes $2$ and $17$ of the burst sequence detection (BSD) task (see main text). \textbf{b)} Classification error of adLIF and LIF networks with equal parameter count for different numbers of classes in the BSD task. \textbf{c)} Schematic illustration of network feature visualization. An initial noise sample $\bm{x_0}$ is passed through a trained network with frozen network parameters. The classification loss of the network output with respect to some predefined target class $c$ is computed and back-propagated through the network to obtain the gradient $\nabla_{X} L(X,c)|_{X^0}$ of the loss with respect to input $X^0$. This gradient is applied to the sample and the procedure is repeated to obtain a final sample $X^K$ after $K=400$ iterations.  \textbf{d)} Samples generated by the feature visualization procedure from panel c from networks trained on the $20$-class BSD task. White dots denote the locations of the class-descriptive bursts. We generated samples for classes $2$ and $17$ which were the most misclassified classes of LIF and adLIF networks respectively.
    \textbf{e)} Samples generated from an adLIF network and a LIF network trained on SHD for different target classes $c$ (top) and the corresponding network output over time (bottom). The grey shaded area (at $t>100$) denotes the time span relevant for the loss, all outputs before this time span were ignored, see \meth~for details. 
    }    
    \label{fig:optimization_based_sampling}
\end{figure}
This task is motivated from neuroscientific experiments which show the importance of sequences of transient increases of spike rates in cortex \cite{harvey2012choice,LISMAN199738}.
A class in this task is defined by a specific pre-defined temporal sequence of bursts across a fixed sub-population of three of these neurons. Other neurons emit a random burst at a random time each and additionally, neurons fire with a background rate of $50$ Hz. Bursts were implemented as smooth, transient increases of firing rate  resulting in approximately $7$ spikes per burst, see Section \enameref{subsec:meth_bsd_details} in \meth~for details. We tested single-layer recurrent adLIF networks ($510$ neurons) and single-layer recurrent LIF networks with the same number of trainable parameters. AdLIF networks clearly outperformed LIF networks on this task, see Fig.~\ref{fig:optimization_based_sampling}b. For the case of $10$ classes, the adLIF network reached an average classification error of $2.31\%$ on the test set, whereas the LIF network only achieved a test error of $6.96\%$ despite a low training error ($<2\%$). 

In order to evaluate to what extent the computations of the SNNs considered relied on temporal features of the input, we used a technique commonly applied to artificial neural networks that allows to visualize the input features that cause the network to predict a certain class \cite{erhan2009visualizing}.
The idea of this optimization-based feature visualization procedure is to generate an artificial data sample $X^*$ that maximally drives the network output towards a pre-defined target class $c$:
\begin{equation}
\label{eq:obs-objective}
    X^* = \arg \min_{X} L(X,c),
\end{equation}
where $L(X,c)$ denotes the loss of the network output for input $X$ and target class $c$. In order to estimate $X^*$, one starts with a uniform noise input $X^0$ and updates the input using gradient descent to minimize the loss, i.e., the input $X^{k+1}$ after update ${k+1}$ is given by
\begin{equation}
    X^{k+1} = X^k - \frac{\eta}{\zeta^k} \nabla_{X} L(X,c)|_{X^k},
\end{equation}
where $\eta$ is the update step size, $\nabla_{X} f(X)|_{X^k}$ denotes the gradient of $f$ with respect to $X$ evaluated at $X^k$, and $\zeta_k$ is a normalization factor, see Fig.~\ref{fig:optimization_based_sampling}c and Section \enameref{subsec:meth_opt_based_feature_viz_details} in \meth~for details. After each update, we applied additional regularization to the data sample $X^{k+1}$ (see \meth~for details). We repeated this procedure for $K=400$ iterations, such that the final $X^K$ yielded a very strong prediction for class $c$.
%
We performed this feature visualization for the BSD task and for the SHD task. Fig.~\ref{fig:optimization_based_sampling}d shows the resulting artificial samples $X^K$ for an adLIF network and a LIF network trained on the BSD task. The class-defining burst timings are indicated as black circles with white filling.  

The adLIF-generated samples exhibited a strong temporal structure that captures the relevant temporal structure of the class. This can be observed visually by comparing the position of strong activations with the class-defining burst timings. In contrast, the samples generated from the trained LIF network displayed less precise resemblance of class-descriptive features, and showed less temporal variation. This gap of specificity of the features in the generated samples of adLIF versus LIF networks might explain the performance gap between these: While the less precise temporal tuning of LIF networks suffices to achieve a high accuracy on the training data, it falls behind in terms of generalization on the test set, due to confusion of temporal features from different classes.
Interestingly, the same analysis for an ALIF network \cite{yinAccurateEfficientTimedomain2021} with threshold adaptation revealed that the temporal tuning of ALIF networks is comparable to that of LIF networks, indicating the importance of oscillatory dynamics for temporal feature detection, see \nameref{supnotes4} and \nameref{supfig4}.

Similar results were obtained for SHD, Fig.~\ref{fig:optimization_based_sampling}e, where we applied $K=200$ iterations. The underlying class-descriptive features in the SHD task are less clearly visible, since it is instantiated from natural speech recordings. Nevertheless, one can clearly observe richer temporal structure of the adLIF-generated samples. These samples indicate, that the two network models tune to very different features of the input, which is again in alignment with the above reported inductive bias of the gradient. While LIF networks tend to tune to certain spike rates of different input neurons over prolonged durations, adLIF neurons rather tune to local variations of the spike rates.
In summary, our analysis supports the hypothesis that the superior performance of SNNs based on adLIF neurons on datasets like SHD stems from the fact that temporal features can effectively be learned and detected. 
\paragraph{Inherent normalization properties of adLIF neurons.}
Artificial neural networks as well as spiking neural networks are often trained using normalization techniques that normalize the input to the network layers\revision{, for example along the spatial dimension of a single batch }\cite{batchnorm}. \revision{As inputs to SNNs are in general temporal sequences, normalizations over both the spatial and temporal dimensions have been introduced \cite{guoMembranePotentialBatch2023,jiangTABTemporalAccumulated2023,zhengGoingDeeperDirectlyTrained2021,vicente-solaSpikingNeuralNetworks2023}}. \revision{Such normalization techniques, in particular over the temporal dimension, are however problematic from an implementation perspective, especially when such networks are deployed on neuromorphic hardware and the entire sequence is not known upfront.}
In contrast, all results reported in this article have been achieved without an explicit normalization technique, indicating that \revision{normalization over the temporal dimension} is not necessary for networks of adLIF neurons.
 
We argue that good performance without normalization is possible due to the negative feedback loop through the adaption current in adLIF neurons (Eqs.~\eqref{eq:continuous1}, \eqref{eq:continuous2}) which inherently stabilizes neuron responses as long as neurons are in the stable regime. In addition, the oscillatory sub-threshold response (Fig.~\ref{fig:fig1}b) tends to filter out constant offsets in the forward pass. Similarly, the oscillating gradient (Fig.~\ref{fig:gradient_magnitude}b) tends to filter out constant activation offsets during training, thus stabilizing training. 
To test the stabilizing effect of adLIF neurons, we investigated how a constant offset in the input during test time effects the accuracy of RSNNs. To this end, we tested LIF and adLIF networks trained on the \revision{clean} SHD dataset (Table \ref{tab:results}) on biased SHD test examples which we obtained by adding a constant offset to each input dimension. More precisely, the biased sample $\hat{\bm{x}}[k]$ at time step $k$ was given by $\hat{\bm{x}}[k] = \bm{x}[k] + \kappa \Bar{x}$, where $\Bar{x}$ is the mean over all input dimensions and time steps and $\kappa\ge 0$ scales the bias strength.
The classification accuracy of the network on these biased test samples as a function of the bias coefficient $\kappa$ is shown in Fig.~\ref{fig:bias_offset}a. 
\begin{figure}[!t]
    \centering
    \includegraphics[width=\textwidth]{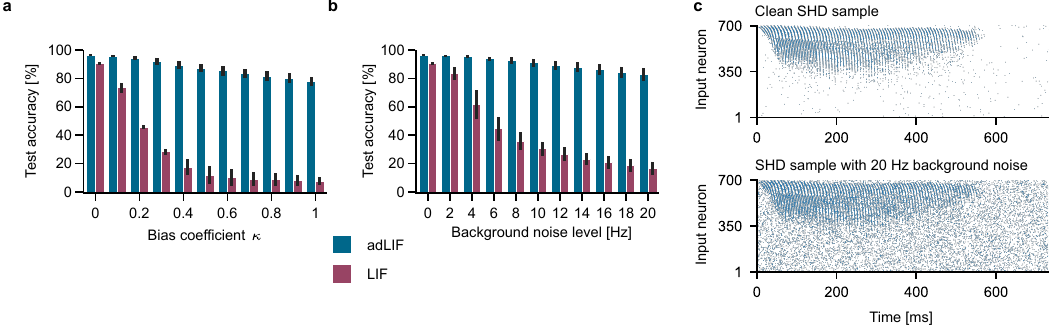}
    \caption{\textbf{Networks of adLIF neurons are robust to perturbation of the input.} \textbf{a)} Test accuracy of adLIF and LIF networks trained on SHD but tested on test examples with a constant input offset of strength $\kappa$ (see main text; bars show mean over $5$ trained network instances, black vertical lines denote STD). \textbf{b)} Same as panel a, but instead of constant offset, random spikes with a specific average rate were added to the data.  \textbf{c)} Spike pattern of a raw, unperturbed SHD sample (top) in comparison to the same SHD sample with added random background noise of a rate of \qty{20}{Hz} (rightmost bar in panel b).}
    \label{fig:bias_offset}
\end{figure}
Surprisingly, even for large bias values the adLIF network maintained a high classification accuracy. In comparison, the accuracy of a LIF network dropped rapidly with increased $\kappa$.

In a second experiment, instead of a constant bias, we added background noise via random spikes to the raw input data. Again, the adLIF network was surprisingly robust to this noise even for high background noise rates  robustness between adLIF and LIF networks, see Fig.~\ref{fig:bias_offset}b,c.

These results support the claim that normalization methods are not necessary to train noise-robust high-performance adLIF networks, which represents a substantial advantage of this model over LIF-based SNNs in neuromorphic applications.
\section*{Discussion}
\label{sec:discussion}
Spiking neural network models are the basis of many neuromorphic systems. For a long time, these networks used leaky integrate-and-fire neurons as their fundamental computational units. 
More recent work has shown that networks of adaptive spiking neurons outperform LIF networks in spatio-temporal processing tasks. However, a deep understanding of the mechanisms that underlie their superiority was lacking. In this article, we investigated the underpinnings of the computational capabilities of networks of adaptive spiking neurons. 
\revision{We first demonstrated, both analytically and empirically, why the commonly used Euler-Forward (EF) discretization is problematic for multi-state neuron models. Specifically, we showed that the EF-discretization leads to substantial deviations in the neuron model dynamics compared to its continuous counterpart (see Fig.~\ref{fig:stability}c) and that the magnitude of these deviations strongly depends on the discretization time step (Fig.~\ref{fig:stability}a,b). Moreover, parameter configurations that yield stable dynamics in the continuous model often result in divergent dynamics when applied to the EF-discretized model. Another unintended consequence of EF discretization is the introduction of interdependencies between parameters that are otherwise decoupled. We briefly examine this effect on the Balanced Harmonic Resonate-and-Fire (BHRF) model \cite{higuchi2024balanced} in \nameref{supnotes1} and \nameref{supfig1}. Finally, we demonstrated that all of these drawbacks can be eliminated, without incurring additional computational cost, by using the Symplectic Euler discretization instead of EF.}

Many digital neuromorphic chips implement time-discretized SNNs \cite{merolla2014million,furber2014spinnaker,orchardEfficientNeuromorphicSignal2021,frenkel2022reckon}. Adaptive neurons are an attractive model for such systems as they only add a single additional state variable per neuron. As synaptic connections are typically dominating implementation costs, this approximate doubling of resources needed for neuron dynamics is well justified by the improved performance. For example, Fig.~\ref{fig:oscillator_task}f shows that there exist tasks where adLIF networks can achieve superior performance to LIF networks with orders of magnitudes less parameters.
Compared to the standard Euler-discretized adLIF model, the computational operations needed to implement the SE-adLIF model are identical. Hence, the improved stability properties of this model are practically for free. 

Our analysis in Section \enameref{subsec:inductive_bias} indicates that adLIF should be well-suited to learn relevant temporal features from input sequences. Interestingly, our investigations revealed that LIF networks are surprisingly weak in that respect. This is witnessed by its low performance at the burst sequence detection task (Fig.~\ref{fig:optimization_based_sampling}) as well as by our input-feature analysis of trained LIF networks (Fig.~\ref{fig:optimization_based_sampling} panels d and f).
This is surprising, as theoretical arguments suggest that SNNs in general should be efficient in temporal computing tasks \cite{maass1997networks}. Our analysis suggests that the gradients in LIF networks fail to detect such temporal features, while for adLIF networks, these gradients are actually biased towards those. Our empirical evaluation supports this claim as adLIF networks excel at the  burst sequence detection task (Fig.~\ref{fig:optimization_based_sampling}), a task which we designed specifically to test temporal feature detection capabilities of SNNs (Fig.~\ref{fig:optimization_based_sampling}). The input-feature analysis for trained adLIF networks further supports this view (Fig.~\ref{fig:optimization_based_sampling} panels b and d). 

The auto-regressive task on a complex spring-mass system (Fig.~\ref{fig:oscillator_task}) can also be seen as a conceptual task designed to investigate the capabilities of SNNs to predict the behavior of complex oscillatory systems. Note that, although the trajectories of masses that have to be predicted are periodic, the period is very long due to the complex interactions of the four masses. This conceptual task is of high relevance as oscillations are ubiquitous in physical systems and biology, for example in limb movement patterns \cite{kilnerHumanCorticalMuscle2000,sanesOscillationsLocalField1993}. The superiority of adLIF networks in this task can be explained by the principles of physics-informed neural networks \cite{raissiPhysicsinformedNeuralNetworks2019}. In this framework, underlying physical laws of some training data are molded into the architecture of neural networks, such that the functions learned by these networks naturally follow these laws. Obviously, the oscillatory sub-threshold behavior of adLIF neurons fits well to the oscillatory dynamics of the spring-mass system. \revision{Taking one step further, we showed that the observed advantage of adaptive LIF neurons in this oscillatory toy task successfully transfers to the more complex, real-world inspired task of audio compression (see Fig.~\ref{fig:audio_compression_task}), providing a promising perspective for adaptive neurons in neuromorphic applications. Interestingly, we observed that the LIF networks were particularly sensitive to the choice of hyperparameters in this task in contrast to the SE-adLIF networks, due to the long sequence length of $2,560$ time steps. }

Finally, we empirically demonstrated the robustness of adLIF networks towards perturbations in the input, showcasing their invariance to shifts in the mean input strength. Surprisingly, the accuracy of adLIF networks on the SHD dataset maintained a high level ($>80\%$) even after doubling the mean input strength, whereas LIF networks drop to an accuracy of $10\%$ already for a much lower increase. We argue that this inheritance alleviates the need for layer normalization techniques in other types of artificial and spiking neural networks. Indeed, all our results were achieved without explicit normalization techniques. This finding is in particular relevant for neuromorphic implementations of SNNs, as explicit normalization is hard to implement in neuromorphic hardware, in particular for recurrent SNNs.

Oscillatory neural network dynamics have been studied not only in the context of SNNs. Several works \cite{rusch2021coupled,dubininFadingMemoryInductive2024b,effenbergerFunctionalRoleOscillatory2023b} identified favorable properties adopted by models utilizing some form of oscillations. Rusch et al.~\cite{rusch2021coupled} for example studied an RNN architecture in which recurrent dynamics were given by the equation of motion of a damped harmonic oscillator. The authors found, that these oscillations not only alleviate the vanishing and exploding gradient problem \citep{glorot}, but also perform well on a large variety of benchmarks. Effenberger et al.~\cite{effenbergerFunctionalRoleOscillatory2023b} proposed oscillating networks as a model for cortical columns.
In the field of artificial neural networks, the recent advent of state space models \cite{voelkerLegendreMemoryUnits2019b,guMambaLinearTimeSequence2023a,guptaDiagonalStateSpaces2022,smithSimplifiedStateSpace2023a} and linear recurrent networks \cite{orvietoResurrectingRecurrentNeural2023a} introduces a paradigm shift in sequence processing, where information is transported through constrained linear state transitions instead of being recurrently propagated between nonlinear neurons, as was previously the case in traditional recurrent neural networks \cite{elmanFindingStructureTime1990}. 
Similar to SNNs, state space models are obtained by discretizing continuous ordinary differential equations (ODEs) to form recurrent neural networks. Although spiking neuron models have always been derived from discretizing differential equations to obtain recurrent linear state transitions \cite{maass2001pulsed}, earlier neuron models, such as the LIF neuron, lack the temporal dynamics necessary to effectively propagate time-sensitive information. The relation between SNNs and state space models has recently been discussed \cite{balRethinkingSpikingNeural2024,schoneScalableEventbyeventProcessing2024a}.

\revision{The adaptation discussed in this article directly acts on the dynamics of the neuron state. To keep the model simple and in order to study the effects of oscillatory dynamics in a clean manner, we did not include other recently proposed model extensions that can improve network performance. For example, oscillatory neuron dynamics can be combined with synaptic delays \cite{hammouamri2023learning} or dendritic processing \cite{zheng2024temporal,ferrand2023context}. Deckers et al.~\cite{deckers2024co} reported promising results for a model that combines a variant of the adLIF model with synaptic delays. Further studies of such combinations, especially with the SE-adLIF model, constitute an interesting direction for future research.}

In summary, we have shown that networks of adaptive LIF neurons provide a powerful model of computation for neuromorphic systems. Stability issues during training can provably be avoided by the use of a suitable discretization method. Our results indicate that the properties of these neurons, in particular their sub-threshold oscillatory response, provide the basis for their spatio-temporal processing capabilities. 

\section*{Methods}
\label{sec11}
\subsection*{Details for simulations in Figure \ref{fig:fig2}}
\label{subsec:resonator_exp_details}For plots in Fig.~\ref{fig:fig2}b-c, we used the adLIF neuron parameters $\tau_w=60$ ms, $\tau_u=15$ ms, $a=120$, $\vartheta=\infty$. For Fig.~\ref{fig:fig2}d-f the adLIF neuron parameters were $\tau_w=200$ ms, $\tau_u=125$ ms, $a=100$, $\vartheta=\infty$. For Fig.~\ref{fig:fig2}a and g, parameters for the LIF neuron were $\tau_u=125$ ms and $\vartheta=\infty$. 

The spike trains for Fig.~\ref{fig:fig2}d and e were generated deterministically in the following way. We first computed a spike rate $s[t] \in [0,0.2]$ for each time step $t$ according to $s[t] = 0.2 (0.5 + 0.5 \sin\left(2\pi t F \Delta t\right))$, where $F$ is the oscillation frequency of the sinus signal in Hz ($F=10$ Hz for Fig.~\ref{fig:fig2}d, $F=7$ Hz for Fig.~\ref{fig:fig2}e), $\Delta t = 1$ ms the sampling time step, and $t \in [0, ... ,T]$. Then, the spike train $S[t]$ was computed by cumulatively summing over the spike rates in an integrate-and-fire manner with $v[t] = v[t-1] + s[t] - S[t]$, where $S[t] = \Theta(v[t] - 1)$ with Heaviside step function $\Theta$ and $v[0]=0$. 
\subsection*{Derivation of matrices $\seA$ and $\seB$ for the SE-adLIF neuron}
\label{subsec:meth_update_matrices_for_se}
Here, we provide the derivation of the matrices $\seA$ and $\seB$ for the SE-discretized adLIF neuron given in Eq.~\eqref{eq:se_lti}.
To rewrite the state update equations of the SE-adLIF neuron model, given by
\begin{subequations}
\begin{align}
    \hat{u}[k] &= \alpha u[k-1] + (1-\alpha) \left(-w[k-1] + I[k]\right) \label{eq:u_improved_update_meth} \\
    w[k] &= \beta w[k-1] + (1-\beta) (a u[k] + b S[k]),
    \label{eq:w_improved_update_meth}
\end{align}
\end{subequations}
into the canonical state-space representation
\begin{equation}
    \bm{s}[k] = 
    \begin{pmatrix}
    u[k] \\
    w[k]
    \end{pmatrix} = \seA \bm{s}[k-1] + \seB  \bm{x}[k],
\end{equation}
we substitute $u[k]$ in Eq.~\eqref{eq:w_improved_update_meth} by $\hat{u}[k]$ from Eq.~\eqref{eq:u_improved_update_meth}. 
Since we study the sub-threshold dynamics of the neuron (assuming $S[k]=0$), we can substitute $\hat{u}[k]$ by $u[k]$, since $u[k] = \hat{u}[k] \cdot (1-S[k])$. This yields the update equation for $w[k]$ given by
\begin{equation}
        w[k] = \beta w[k-1] + (1-\beta) \left(a \underbrace{\left(\alpha u[k-1] + (1-\alpha) \left(-w[k-1] + I[k]\right)\right)}_{\hat{u}[k]} + b S[k]\right).
\end{equation}
Since this formulation gives the new state of adaptation variable $w[k]$ as function of the previous states $u[k-1]$ and $w[k-1]$, it can be transformed into a matrix formulation
\begin{equation}
    \begin{pmatrix}
    u[k] \\
    w[k]
    \end{pmatrix} = \underbrace{\begin{pmatrix}
     \alpha & -(1-\alpha)\\
    a(1-\beta)\alpha & \quad \beta - a(1 - \beta)(1 - \alpha)\\
    \end{pmatrix}}_{\seA} \begin{pmatrix}
    u[k-1] \\
    w[k-1]
    \end{pmatrix}  + \underbrace{\begin{pmatrix}
    (1-\alpha) & 0\\
    0 & b(1-\beta) \\
    \end{pmatrix}}_{\seB}  \begin{pmatrix}
    I[k] \\
    S[k]
    \end{pmatrix} .
\end{equation}
\subsection*{Proof of stability bounds for the continuous adLIF model}
\label{subsubsec:continuous_stability_bounds}
For all following analyses in Sections \enameref{subsubsec:continuous_stability_bounds} to \enameref{subsubsec:meth_frequency_ranges}, we consider the subthreshold regime, i.e., we assume a spike threshold $\vartheta = \infty$ such that $S[k] = 0$ for all $k$, and we assume no external inputs $I$.

In this section, we prove that the continuous-time adLIF neuron exhibits stable sub-threshold dynamics for all $a > -1$. 
\begin{lemma}
    The continuous adLIF neuron model from Eqs.~\eqref{eq:continuous1} and \eqref{eq:continuous2} is stable in the sub-threshold regime for all $\tau_u > 0, \tau_w > 0, a> -1$.
\end{lemma}
\begin{proof}
In general, a continuous-time linear dynamical system $\dot{x}=Ax$ is Lyapunov-stable, if the real parts $\Re(\lambda_{1,2})$ of both eigenvalues $\lambda_{1,2}$ of matrix $A$ satisfy $\Re(\lambda_{1,2}) \leq 0$ \cite{linear-systems-theory}. For the adLIF model, the matrix $A$ is given by
\begin{equation}
    A = \begin{pmatrix}
     -\frac{1}{\tau_u} & -\frac{1}{\tau_u}\\
    \frac{a}{\tau_w} & -\frac{1}{\tau_w}
    \end{pmatrix},
\end{equation}
with eigenvalues
\begin{align}
    \lambda_1 &= \frac{ - \tau_u - \tau_w + \sqrt{-4 a \tau_u \tau_w + (\tau_u - \tau_w)^2}}{2 \tau_u \tau_w} \\
    \lambda_2 &= \frac{- \tau_u - \tau_w - \sqrt{-4 a \tau_u \tau_w + (\tau_u - \tau_w)^2}}{2 \tau_u \tau_w} \label{meth:cont_eigvals1}.
\end{align}
In the complex-valued case (where the discriminant $-4 a \tau_u \tau_w + (\tau_u - \tau_w)^2 < 0$), the real part $\Re(\lambda_{1,2})$ is given by
\begin{equation}
    \Re(\lambda_{1,2}) = \frac{- \tau_u - \tau_w}{2 \tau_u \tau_w},
\end{equation}
where  $\Re(\lambda_{1,2}) < 0$ since $\tau_u,\tau_w>0$. 
In the case of real eigenvalues, $\lambda_2<0$ is always true (since Eq.~\eqref{meth:cont_eigvals1} only contains negative terms), whereas $\lambda_1 \leq 0$ is only true if
\begin{align}
    \sqrt{-4 a \tau_u \tau_w + (\tau_u - \tau_w)^2} \leq \tau_u + \tau_w \\
    a \geq - \frac{(\tau_u + \tau_w)^2 - (\tau_u - \tau_w)^2}{4 \tau_u \tau_w}
\end{align}
Hence, the continuous-time adLIF neuron is stable for all
\begin{equation}
    a \geq - 1.
\end{equation}
This proves, that for $a \geq -1$, the continuous-time adLIF neuron is Lyapunov stable.
\end{proof}
\subsection*{Derivation of stability bounds for EF-adLIF}
\label{meth:euler_forward_stab}
In this section we derive the stability bounds of the EF-adLIF neuron with respect to parameter $a$. These bounds provide the basis for the subsequent proofs.
The state update equations of the Euler-Forward discretized adLIF neuron are given by
\begin{equation}
\dot {\bm{s}}(t) = 
    \begin{pmatrix}
    \dot{u}(t) \\
    \dot{w}(t)
    \end{pmatrix} = \eulerA \bm{s}(t) + \eulerB \bm{x}(t) \nonumber 
\end{equation}
with
\begin{equation}
    \eulerA =  \begin{pmatrix}
     \alpha & -(1-\alpha)\\
    a(1-\beta) & \beta
    \end{pmatrix} \hspace{30pt}
    \eulerB =  \begin{pmatrix}
    (1-\alpha) & 0\\
    0 & b(1-\beta) \\
    \end{pmatrix}
\end{equation}
with $\alpha = e^{-\frac{\Delta t}{\tau_u}}$ and $\beta = e^{-\frac{\Delta t}{\tau_w}}$. This system is asymptotically stable, if the spectral radius $\rho$, given by $\rho = \max\left(|\lambda_1|, |\lambda_2|\right)$ is less than $1$. We differentiate between two cases: complex-valued and real-valued eigenvalues. We assume given time constants $\tau_u, \tau_w > 0$ and calculate the bound as function of parameter $a$.
To simplify the notation, we introduce $\Bar{\alpha} = (1 - \alpha)$ and $\Bar{\beta} = (1 - \beta)$.
\begin{lemma}
\label{lem:ef_complexe_stability_range} 
The EF-adLIF model is asymptotically stable and oscillating in the sub-threshold regime, for $a \in~(a_0^\text{EF}, a_\mathrm{max}^\text{EF})$, with $a_0^\text{EF} = \frac{(\alpha - \beta)^2 }{4(1 - \alpha)(1 - \beta)}$ and $a_\mathrm{max}^\text{EF} = \frac{1 - \alpha\beta}{(1 - \alpha)(1 - \beta)}$.  The real part of the eigenvalues is strictly positive and given by $\Re(\lambda_{1,2}) = \frac{\alpha + \beta}{2} > 0$.
\end{lemma}
\begin{proof}
The characteristic polynomial $\chi(\eulerA)$ is given by
\begin{equation}
    \chi(\eulerA) = \lambda^2 + \lambda\left(- \alpha - \beta\right) + \beta\alpha + a\Bar{\alpha}\Bar{\beta},
\end{equation}
with discriminant
\begin{align}
    \Delta_{\lambda} & = (\alpha - \beta)^2 - 4a\Bar{\alpha}\Bar{\beta}.
\end{align}

$\chi(\eulerA)$ admits complex solutions for $\Delta_\lambda < 0$ depending on $a$, with
\begin{align}
     a &> \frac{(\alpha - \beta)^2 }{4\Bar{\alpha}\Bar{\beta}} = a_0^\text{EF},
\end{align}
Complex roots of $\chi(\eulerA)$ yield the complex-conjugate eigenvalues
\begin{equation}
\label{eq:ef_eigenvalues}
    \lambda_{1,2} = \frac{\alpha + \beta}{2} \pm i \frac{\sqrt{-\Delta_\lambda}}{2}.
\end{equation}
This proves that the complex eigenvalues have a strictly positive real part given by $\Re(\lambda_{1,2}) = \frac{\alpha + \beta}{2} > 0$.
In that case the spectral radius $\rho$ is defined as $\rho(\eulerA) = |\lambda_1| = |\lambda_2|$, where $|z| = \sqrt{\mathrm{Re}(z)^2 + \mathrm{Im}(z)^2}$ is the modulus of the complex number with
\begin{align}
    \rho(\eulerA) &= \frac{1}{2}\sqrt{\left(\beta + \alpha\right)^2 - \Delta_\lambda} \\
    &= \frac{1}{2}\sqrt{\left(\beta + \alpha\right)^2 - \left(\beta -\alpha\right)^2 + 4\revision{a}\Bar{\alpha}}\Bar{\beta}\\
    &= \sqrt{\alpha\beta + a\Bar{\alpha}\Bar{\beta}}.
\end{align}
In the complex regime, the system is stable when $\rho(\eulerA) < 1$, and thus, 
\begin{equation}\label{eq:ef_complex_upper_bound}
    a < \frac{1 - \alpha\beta}{\Bar{\alpha}\Bar{\beta}} = a_\mathrm{max}^\text{EF},
\end{equation}
where $a_\mathrm{max}^\text{EF}$ gives the upper stability bound.

The system is asymptotically stable in the sub-threshold regime for $a \in~(\frac{(\alpha - \beta)^2 }{4\Bar{\alpha}\Bar{\beta}}, \frac{1 - \alpha\beta}{\Bar{\alpha}\Bar{\beta}})$. In this bound, the system has complex eigenvalues and thus admits oscillations. 
\end{proof}
\begin{lemma}
\label{lem:ef_real_stability_range} 
EF-adLIF is asymptotically stable and not oscillating in the sub-threshold regime for $a \in (-1, a_0^\text{EF}]$ with $a_0^\text{EF} = \frac{(\alpha - \beta)^2 }{4\Bar{\alpha}\Bar{\beta}}$.
\end{lemma}
\begin{proof}
The characteristic polynomial $\chi(\eulerA)$ is given by
\begin{equation}
    \chi(\eulerA) = \lambda^2 + \lambda\left(- \alpha - \beta\right) + \beta\alpha + a\Bar{\alpha}\Bar{\beta},
\end{equation}
with discriminant
\begin{align}
    \Delta_{\lambda} & = (\alpha - \beta)^2 - 4a\Bar{\alpha}\Bar{\beta}
\end{align}
$\chi(\eulerA)$ admits real solutions for $\Delta_\lambda \geq 0$ depending on $a$, with
\begin{align}
    \Delta_\lambda & \geq 0\\
    \label{eq:ef_real_criterium}
     (\alpha - \beta)^2  & \geq 4a\Bar{\alpha}\Bar{\beta}\\
     a & \leq \frac{(\alpha - \beta)^2 }{4\Bar{\alpha}\Bar{\beta}} = a_0^\text{EF},
\end{align}
 $\chi(\eulerA)$ thus admits real solutions in the interval $\left(\infty, a_0^\text{EF}\right]$ with real eigenvalues
 \begin{equation}
    \lambda_{1,2} = \frac{\alpha + \beta \pm \sqrt{\Delta_\lambda}}{2}.
\end{equation}
For $a = a_0^\text{EF}$, the singular real root of $\chi(\eulerA)$ is
\begin{equation}\label{eq:ef_sing_root_real}
    \lambda_0 = \frac{\alpha + \beta}{2} < 1.
\end{equation}
For lower values of $a$ we have the following asymptotic behavior,
\begin{align}
    \lim_{a \to -\infty} \lambda_1 &= \frac{\alpha + \beta + \sqrt{\Delta_\lambda}}{2} =\infty\\
    \lim_{a \to -\infty} \lambda_2 &= \frac{\alpha + \beta - \sqrt{\Delta_\lambda}}{2} = -\infty,
\end{align}
where $\lambda_1$ is the term with the higher absolute value and hence determines the spectral radius $\rho(\eulerA)$. It exists a value $a_\mathrm{min}^\text{EF}$ such that $\lambda_1(a_\mathrm{min}^\text{EF}) = 1$,
\begin{align}\label{eq:ef_real_lower_bound}
    \lambda_1 &= \frac{\alpha + \beta + \sqrt{(\alpha - \beta)^2 - 4a\Bar{\alpha}\Bar{\beta}}}{2} = 1,\\
    a_\mathrm{min}^\text{EF} &= \frac{1 - \alpha -\beta + \alpha\beta}{-\Bar{\alpha}\Bar{\beta}} = \frac{\Bar{\alpha}\Bar{\beta}}{-\Bar{\alpha}\Bar{\beta}} = -1,
\end{align}
and the system is unstable for $a < a_\mathrm{min}^\text{EF} = -1$.
The system is thus asymptotically stable in the sub-threshold regime for $a \in (-1, a_0^\text{EF}]$. In this bound, the system has real eigenvalues and thus doesn't admit oscillations. 
\end{proof}

\begin{corollary}\label{thm:stability_ef}
    In the sub-threshold regime (i.e. $\vartheta = \infty$), for $\tau_u,\tau_w, \Delta t > 0$, $\alpha = e^{-\frac{\Delta t}{\tau_u}}$ and $\beta = e^{-\frac{\Delta t}{\tau_w}}$, EF-adLIF is asymptotically stable for $a \in (-1, a_{\mathrm{max}}^\text{EF})$ with $a_{\mathrm{max}}^\text{EF} = \frac{1 - \alpha\beta}{(1 - \beta)(1 - \alpha)}$.
\end{corollary}
    This is a consequence of Lemma \ref{lem:ef_complexe_stability_range} and Lemma \ref{lem:ef_real_stability_range}.
\subsection*{Derivation of the stability bounds of SE-adLIF}
\label{meth:se_stab}
Analogously to Section \enameref{meth:euler_forward_stab}, we can compute the stability bounds for the SE-discretized adLIF neuron (SE-adLIF) with respect to parameters $\tau_u$, $\tau_w$, and $a$. Again, we assume time constants $\tau_u, \tau_w > 0$ as given and calculate the stability bounds with respect to parameter $a$. Recall state transition matrix $\seA$ from Eq.~\eqref{eq:se_lti}:
\begin{equation}
\seA =  \begin{pmatrix}
     \alpha & -(1-\alpha)\\
    a(1-\beta)\alpha & \quad \beta - a(1 - \beta)(1 - \alpha)\\
    \end{pmatrix}
\end{equation}
with $\alpha = e^{-\frac{\Delta t}{\tau_u}}$ and $\beta = e^{-\frac{\Delta t}{\tau_w}}$.
To simplify notation, we again introduce $\Bar{\alpha} = (1 - \alpha)$ and $\Bar{\beta} = (1 - \beta)$.
\begin{lemma}
\label{lem:se_complex_stability_range} 
SE-adLIF is stable and oscillating in the sub-threshold regime for $a \in (a_{1}^\text{SE},a_{2}^\text{SE})$, with $a_{1}^\text{SE} = \frac{\left(\sqrt{\beta} - \sqrt{\alpha}\right)^2}{\Bar{\beta}\Bar{\alpha}}$ and $a_{2}^\text{SE} = \frac{\left(\sqrt{\beta} + \sqrt{\alpha}\right)^2}{\Bar{\beta}\Bar{\alpha}}$. The spectral radius is independent of $a$ and given by $\rho = \sqrt{\beta\alpha} < 1$.
\end{lemma}
\begin{proof}
The characteristic polynomial of $\chi(\seA)$ is
\begin{equation}
    \chi(\seA) = \lambda^2 + \lambda\left(\Bar{\beta}\Bar{\alpha}a - \beta - \alpha\right) + \beta\alpha,
\end{equation}
with discriminant $\Delta_{\lambda}$ given by
\begin{align}
    \Delta_{\lambda} & = \left(\Bar{\beta}\Bar{\alpha}a - \beta -\alpha\right)^2 - 4\beta\alpha \nonumber\\
    & = \Bar{\beta}\Bar{\alpha}\left[\Bar{\beta}\Bar{\alpha}a^2 - 2(\beta +\alpha)a + \frac{(\beta - \alpha)^2}{\Bar{\beta}\Bar{\alpha}}\right].
    \label{eq:disc_lambda_se}
\end{align}
 $\chi(\seA)$ gives complex solutions for $\Delta_\lambda < 0$, which depends on the range of negative values of the polynomial part $p(a)$ of $\Delta_{\lambda}$, given by
\begin{equation}
    p(a) = \Bar{\beta}\Bar{\alpha}a^2 - 2(\beta +\alpha)a + \frac{(\beta - \alpha)^2}{\Bar{\beta}\Bar{\alpha}}.
\end{equation}
From this polynomial, we can again compute a discriminant $\Delta_a$ as
\begin{align}
    \Delta_a = 16\beta\alpha.
\end{align}
Since $\Delta_a>0$ the roots of this polynomial are given by two values $a_1^\text{SE}$ and $a_2^\text{SE}$ according to
\begin{equation}
\label{eq:a1_a2_se}
    a_{1}^\text{SE} = \frac{\left(\sqrt{\beta} - \sqrt{\alpha}\right)^2}{\Bar{\beta}\Bar{\alpha}},\quad a_{2}^\text{SE} = \frac{\left(\sqrt{\beta} + \sqrt{\alpha}\right)^2}{\Bar{\beta}\Bar{\alpha}}.
\end{equation}
As the coefficient $\Bar{\beta}\Bar{\alpha}$ in Eq.~\eqref{eq:disc_lambda_se} is always positive, $\Delta_\lambda$ is negative when $a_1^\text{SE} < a < a_2^\text{SE}$ resulting in complex-conjugate eigenvalues $\lambda_{1,2}$, given by the roots of $\chi(\seA)$, implying oscillatory behavior of the membrane potential.
These eigenvalues are given by
\begin{equation}
\label{eq:eigvals_se}
    \lambda_{1,2} = -\frac{\left(\Bar{\beta}\Bar{\alpha}a - \beta - \alpha\right)}{2} \pm i \frac{\sqrt{-\Delta_\lambda}}{2}.
\end{equation}
In that case the spectral radius is defined as $\rho(A) = |\lambda_1| = |\lambda_2|$, where $|z| = \sqrt{\mathrm{Re}(z)^2 + \mathrm{Im}(z)^2}$ is the modulus of the complex eigenvalues, such that
\begin{align}
    r = \rho(A) &= \frac{1}{2}\sqrt{\left(\Bar{\beta}\Bar{\alpha}a - \beta - \alpha\right)^2 - \Delta_\lambda} \\
    & =  \frac{1}{2}\sqrt{\left(\Bar{\beta}\Bar{\alpha}a - \beta - \alpha\right)^2  - \left(\Bar{\beta}\Bar{\alpha}a - \beta - \alpha\right)^2 + 4\beta\alpha}\\
    &= \sqrt{\beta\alpha}.
    \label{eq:se_r_bound_complex}
\end{align}
Hence, the spectral radius, which we also refer to as the decay rate $r$ in the main text, is $\sqrt{\beta\alpha}$ which is always strictly less than $1$ due to $\beta, \alpha \in \left(0, 1\right)$. Therefore, the SE-adLIF neuron is stable over the entire range of parameters (provided $\tau_u, \tau_w > 0$) where the matrix $\seA$ exhibits complex eigenvalues.
\end{proof}
\begin{lemma}
\label{lem:se_real_stability_range} 
SE-adLIF is stable and not oscillating in the sub-threshold regime for $a \in (-1, a_{1}^\text{SE}] \cup [a_{2}^\text{SE}, a_{\mathrm{max}}^\text{SE})$ with $a_{\mathrm{max}}^\text{SE} = \frac{(1 + \beta)(1 + \alpha)}{(1 - \beta)(1 - \alpha)}$ and $a_{1}^\text{SE}$, $a_{2}^\text{SE}$ as defined in Lemma \ref{lem:se_complex_stability_range}.
\end{lemma}
\begin{proof}
The characteristic polynomial of $\chi(\seA)$ is
\begin{equation}
    \chi(\seA) = \lambda^2 + \lambda\left(\Bar{\beta}\Bar{\alpha}a - \beta - \alpha\right) + \beta\alpha,
\end{equation}
with discriminant $\Delta_{\lambda}$ given by
\begin{align}
    \Delta_{\lambda} & = \left(\Bar{\beta}\Bar{\alpha}a - \beta -\alpha\right)^2 - 4\beta\alpha \nonumber\\
    & = \Bar{\beta}\Bar{\alpha}\left[\Bar{\beta}\Bar{\alpha}a^2 - 2(\beta +\alpha)a + \frac{(\beta - \alpha)^2}{\Bar{\beta}\Bar{\alpha}}\right].
\end{align}

The derivations from Lemma \ref{lem:se_complex_stability_range} imply that $\chi(\seA)$ admits real solutions in the intervals $\left(-\infty, a_1^\text{SE}\right]$ and $\left[a_2^\text{SE}, \infty\right)$, yielding real-valued eigenvalues,
\begin{equation}
    \lambda_{1,2} = \frac{-\left(\Bar{\beta}\Bar{\alpha}a - \beta - \alpha\right) \pm \sqrt{\Delta_\lambda}}{2}.
\end{equation}
For $a = a_1^\text{SE}$ and $a = a_2^\text{SE}$, $\Delta_\lambda = 0$ and we have two singular solutions,
\begin{align}
    \lambda^0_1 &= -\frac{\left(\Bar{\beta}\Bar{\alpha}a_{1}^\text{SE} - \beta - \alpha\right)}{2} = \sqrt{\beta\alpha} \\
    \lambda^0_2 &= -\frac{\left(\Bar{\beta}\Bar{\alpha}a_{2}^\text{SE} - \beta - \alpha\right)}{2} = -\sqrt{\beta\alpha}.
\end{align}
For both singular solutions the spectral radius is $\rho(\seA) = |r^0_1| = |r^0_2| = \sqrt{\beta\alpha}$ which is less than one, resulting in asymptotic stability.

In order to study the stability for $a < a_1^\text{SE}$ and $a > a_2^\text{SE}$ we need to determine the asymptotic behavior of $\lambda_{1,2}$ for $a \to -\infty$ and $a \to \infty$, as it allows us to determine which eigenvalue will constrain the stability of the system.
In the following propositions, we prove the stability of the system in the intervals  $(-1, a_{1}^\text{SE}]$ and $[a_{2}^\text{SE}, a_{\mathrm{max}}^\text{SE})$ independently.

\begin{prop} SE-adLIF is asymptotically stable for $a \in (-1, a_{1}^\text{SE}]$.

For $a < a_1^\text{SE}$, $\rho(\seA)$ is determined by $|\lambda_1|$ as for $a \to -\infty$, we have
\begin{align}
    \lim_{a \to -\infty} \lambda_1 &= \lim_{a \to -\infty} \frac{-\left(\Bar{\beta}\Bar{\alpha}a - \beta - \alpha\right) + \sqrt{\left(\Bar{\beta}\Bar{\alpha}a - \beta - \alpha\right)^2 - 4\beta\alpha}}{2} = \infty \\
    \lim_{a \to -\infty} \lambda_2 &= \lim_{a \to -\infty} \frac{-\left(\Bar{\beta}\Bar{\alpha}a - \beta - \alpha\right) - \sqrt{\left(\Bar{\beta}\Bar{\alpha}a - \beta - \alpha\right)^2 - 4\beta\alpha}}{2} = 0.
\end{align}
Thus $\lambda^0_1 < 1 < \lim\limits_{a \to -\infty} \lambda_1$. We can find a value $a_{\mathrm{min}}^\text{SE}$ such that $\lambda_1(a_{\mathrm{min}}^\text{SE}) = 1$, given by
\begin{align}
\lambda_1 &= \frac{ -\left(\Bar{\beta}\Bar{\alpha}a_{\mathrm{min}}^\text{SE} - \beta -\alpha\right) + \sqrt{\left(\Bar{\beta}\Bar{\alpha}a_{\mathrm{min}}^\text{SE} - \beta - \alpha\right)^2 - 4\beta\alpha}}{2} = 1\\
a_{\mathrm{min}}^\text{SE} &= \frac{-\beta\alpha + \beta + \alpha - 1}{\Bar{\beta}\Bar{\alpha}} = \frac{-\Bar{\beta}\Bar{\alpha}}{\Bar{\beta}\Bar{\alpha}} = -1.
\end{align}
\end{prop}
\begin{prop}  SE-adLIF is asymptotically stable for $a \in [a_{2}^\text{SE}, a_{\mathrm{max}}^\text{SE})$ with $a_{\mathrm{max}}^\text{SE} = \frac{(1 + \beta)(1 + \alpha)}{(1-\beta)(1-\alpha)}$.

For $a > a_2^\text{SE}$, $\rho(\seA)$ is determined by $|\lambda_2|$ as for $a \to \infty$, we have
\begin{align}
    \lim_{a \to \infty} \lambda_1 &= \lim_{a \to \infty} \frac{-\left(\Bar{\beta}\Bar{\alpha}a - \beta - \alpha\right) + \sqrt{\left(\Bar{\beta}\Bar{\alpha}a - \beta - \alpha\right)^2 - 4\beta\alpha}}{2} = 0\\
    \lim_{a \to \infty} \lambda_2 &= \lim_{a \to \infty} \frac{-\left(\Bar{\beta}\Bar{\alpha}a - \beta - \alpha\right) - \sqrt{\left(\Bar{\beta}\Bar{\alpha}a - \beta - \alpha\right)^2 - 4\beta\alpha}}{2} = -\infty.
\end{align}
Thus $\lim\limits_{a \to \infty} \lambda_2 < -1 < \lambda^0_2$.  We can find a value $a_{\mathrm{max}}^\text{SE}$ such that $\lambda_2(a_{\mathrm{max}}^\text{SE}) = -1$, given by
\begin{align}
\lambda_2 &= \frac{ -\left(\Bar{\beta}\Bar{\alpha}a_{\mathrm{max}}^\text{SE} - \beta -\alpha\right) - \sqrt{\left(\Bar{\beta}\Bar{\alpha}a_{\mathrm{max}}^\text{SE} - \beta - \alpha\right)^2 - 4\beta\alpha}}{2} = -1\\
a_{\mathrm{max}}^\text{SE} &= \frac{1 + \beta\alpha + \beta + \alpha}{\Bar{\beta}\Bar{\alpha}} = \frac{(1 + \alpha)(1 + \beta)}{\Bar{\beta}\Bar{\alpha}}
\label{eq:a_se_max}
\end{align}
 For all, $a > a_{\mathrm{max}}^\text{SE}$, it follows that $\rho(\seA) = |\lambda_2| > 1$ and the system is unstable.
\end{prop}
\end{proof}
\begin{corollary}\label{thm:se_proof_stability}
    In the sub-threshold regime (i.e. $\vartheta = \infty$), for $\tau_u,\tau_w, \Delta t > 0$ and $\alpha = e^{-\frac{\Delta t}{\tau_u}}$ and $\beta = e^{-\frac{\Delta t}{\tau_w}}$, the Symplectic-Euler discretized adLIF neuron (SE-adLIF) is asymptotically stable for all $a \in~(-1, a_{\mathrm{max}}^\text{SE})$ with $a_{\mathrm{max}}^\text{SE} = \frac{(1 + \beta)(1 + \alpha)}{(1-\beta)(1-\alpha)}$.
\end{corollary}
This is a consequence of Lemma \ref{lem:se_complex_stability_range} and Lemma \ref{lem:se_real_stability_range}.
\subsection*{Proof of Theorem \ref{theo:se_stab}}
\label{meth:proof_theo_se_stab}

Theorem \ref{theo:se_stab} states that for each choice of $\tau_u > 0$, $\tau_w > 0$, and each intrinsic frequency $f \in [0, f_N]$, there is a unique parameter value $[a^{\text{SE}}_1, a^{\text{SE}}_2]$ such that the SE-adLIF neuron with these parameters has intrinsic frequency $f$ and vice versa, while being asymptotically stable in the sub-threshold regime. In other words, the neuron model can exhibit the full range of intrinsic frequencies for any setting of $\tau_u > 0$, $\tau_w > 0$ in a stable manner. In the following we proof Theorem \ref{theo:se_stab}.

\begin{proof}
    We first show that for an arbitrary intrinsic frequency $f \in (0, f_N]$, there exists a value $a$ such that an SE-adLIF neuron with arbitrary $\tau_u>0$ and $\tau_w>0$ oscillates with $f$. To show this, we consider an SE-adLIF neuron in the sub-threshold regime with an arbitrary choice of $\tau_u > 0$ and $\tau_w > 0$. Let $g_{\tau_u,\tau_w}: [a^{\text{SE}}_1, a^{\text{SE}}_2] \rightarrow [0, f_N]$ be the function that maps parameter $a$ of the neuron to the neuron's intrinsic frequency $f$. We prove in the following that $g_{\tau_u,\tau_w}$ is a bijection.

   As proven in the Lemma \ref{lem:se_complex_stability_range}, in the range $a \in (a_1^\text{SE}, a_{2}^\text{SE})$, with $a_{1}^\text{SE} = \frac{\left(\sqrt{\beta} - \sqrt{\alpha}\right)^2}{\Bar{\beta}\Bar{\alpha}}$ and $a_{2}^\text{SE} = \frac{\left(\sqrt{\beta} + \sqrt{\alpha}\right)^2}{\Bar{\beta}\Bar{\alpha}}$, $\seA$ has complex eigenvalues $\lambda_{1,2}$ given by Eq.~\eqref{eq:eigvals_se}.

Let $a \in [a_1^\text{SE}, a_2^\text{SE}]$ for an arbitrary choice of $\tau_u > 0$ and $\tau_w > 0$, and $\lambda_1(a)$ be the complex eigenvalue of $\seA$ as function of $a$ with positive argument $arg(\lambda_1) \geq 0$.

We claim there is a natural bijection $\phi(a)$ between $[a_1^\text{SE}, a_2^\text{SE}]$ and $[0, \pi]$. We first show that $\cos(\phi(a))$ is a bijection between $[a_1^\text{SE}, a_2^\text{SE}]$ and $[-1, 1]$. The bijection to  $[0, \pi]$ follows directly from $\arccos(x): [-1, 1] \to [0, \pi]$ defined as a bijective function on its principal values $[0, \pi]$.

We have the trigonometric relation $\Re(\lambda_1(a)) = r \cos(\phi(a))$ with $r = |\lambda_1(a)| = \sqrt{\beta\alpha}$. $\cos(\phi(a))) = \frac{\Re(\lambda_1(a))}{\sqrt{\beta\alpha}}$ is a bijection between $[a_1^\text{SE}, a_2^\text{SE}]$ and $[-1, 1]$.
$\cos(\phi(a)))$ is surjective, since it is continuous in $a$, $\cos\left(\phi(a_1^\text{SE}))\right) = 1$, and $\cos\left(\phi(a_2^\text{SE}))\right) = -1$.
$\cos(\phi(a))$ is injective, as for $a \in (a_1^\text{SE}, a_2^\text{SE})$, $\Re(\lambda_1(a))$ is a strictly decreasing continuous function, which follows from the fact that its derivative $\Re^{\prime}(\lambda_1(a)) = - \frac{\Bar{\beta}\Bar{\alpha}}{2} < 0$ for all $a \in (a_1^\text{SE}, a_2^\text{SE})$, and $r = \sqrt{\beta\alpha}$ is a positive constant. 

It follows that $\phi(a) = \arccos{\frac{\Re\left(\lambda_1(a)\right)}{\sqrt{\beta\alpha}}}$ defines a bijective function between $[a_1^\text{SE}, a_2^\text{SE}]$ and $[0, \pi]$, with $\phi(a_1^\text{SE}) = \arccos(1) =  0$ and $\phi(a_2^\text{SE}) = \arccos(-1) = \pi$.

Since the frequency in Hertz is defined by $f = \frac{\phi}{2\pi\Delta t}$, 
we have shown that
\begin{align}
    g_{\tau_u,\tau_w}: [a^{\text{SE}}_1, a^{\text{SE}}_2] &\to [0, f_N] \\ \nonumber
    a &\mapsto \frac{\phi(a)}{2\pi\Delta t}
\end{align}
is a bijective function with  $g_{\tau_u,\tau_w}(a_1^\text{SE}) = 0$ and $g_{\tau_u,\tau_w}(a_2^\text{SE}) = f_N$, the Nyquist frequency. 

Hence, we have shown that for an arbitrary intrinsic frequency $f \in (0, f_N]$, there exists a value $a$ such that an SE-adLIF neuron with arbitrary $\tau_u>0$ and $\tau_w>0$ oscillates with $f$.

Second, we have to show that an SE-adLIF neuron with arbitrary parameters $(a, \tau_u>0, \tau_w>0)$ and intrinsic oscillation frequency $f \in (0, f_N]$ is asymptotically stable with decay rate $r = \sqrt{\alpha\beta} < 1$ where $\alpha = e^{-\frac{\Delta t}{\tau_u}}$ and $\beta = e^{-\frac{\Delta t}{\tau_w}}$. Above we have shown that for parameter $a \in [a_1^{\text{SE}}, a_2^{\text{SE}}]$ a bijective mapping to each frequency $f \in (0, f_N]$ exists. The asymptotic stability follows from this fact in combination with Lemma \ref{lem:se_complex_stability_range}.

\end{proof}

\subsection*{Stable ranges for intrinsic frequencies of EF-adLIF}
\label{subsubsec:meth_frequency_ranges}
We consider the parameterizations of EF-adLIF and SE-adLIF neurons where the discretized transition matrix $\seA$ has complex eigenvalues. In that case, the neuron exhibits oscillations of intrinsic frequency $f$ determined by the angle $\phi=\text{arg}\left(\lambda_{1}\right)$ of the complex eigenvalue $\lambda_{1}$, see also Sections \enameref{meth:euler_forward_stab} and \enameref{meth:se_stab}. This angle determines the angle of 'rotation' of the state for each time step, from which we can infer the intrinsic frequency by $f = \frac{\phi}{2\pi}f_S$ where $f_S=\frac{1}{\Delta t}$ is the sampling frequency. We define the Nyquist frequency as half the sampling frequency $f_N = \frac{f_S}{2}$.

\begin{lemma}\label{lem:ef_bandwidth}
EF-adLIF neurons can oscillate with intrinsic frequencies $f$ bounded by $f \in [0, \frac{f_N}{2})$
\end{lemma}
\begin{proof}
We proved in Lemma \ref{lem:ef_complexe_stability_range} that in the range $a \in (a_0^\text{EF}, a_{\max}^\text{EF})$, with $a_0^\text{EF} = \frac{(\alpha - \beta)^2 }{4\Bar{\alpha}\Bar{\beta}}$ and $a_{\max}^\text{EF} = \frac{1 - \alpha\beta}{\Bar{\alpha}\Bar{\beta}}$, the EF-adLIF neuron is asymptotically stable and has complex eigenvalues $\lambda_{1,2}$ given by Eq.~\eqref{eq:ef_eigenvalues}. 

$\lambda_{1}$ is restricted to the right half-plane of the complex plane, which follows directly from the fact that, by definition $\Im(\lambda_1) > 0$ and $\Re(\lambda_1) = \frac{\beta + \alpha}{2} > 0$, for all $\tau_u > 0$ and $\tau_w > 0$. Hence, the argument is restricted to $0 \leq \phi < \frac{\pi}{2}$.
Since the frequency in Hertz is defined by $f = \frac{\phi}{2\pi\Delta t}$, this results in an upper bound on the oscillation frequency of $f<\frac{f_N}{2}$ with  $f_N$ the Nyquist frequency.

Note, that this upper bound is approached in the limit $\tau_u, \tau_w \rightarrow 0$, but the maximum frequency is much lower for realistic values of $\tau_u$ and $\tau_w$, as shown in Fig.~\ref{fig:stability}f.
\end{proof}

\begin{lemma}\label{lem:ef_bandwidth_decay}
    For EF-adLIF, as $\tau_u$ and $\tau_w$ increase, the stable frequency bandwidth of EF-adLIF asymptotically converges towards $0$.
\end{lemma}
\begin{proof}
Recall that $\alpha = e^{-\frac{\Delta t}{\tau_u}}$ and $\beta = e^{-\frac{\Delta t}{\tau_w}}$ for $\tau_u > 0$ and $\tau_w > 0$.
We evaluate $\Im(\lambda_1)$ at the stability boundary $a_{\max}^\text{EF}$ and obtain $\Im\left(\lambda_1(a_{\max}^\text{EF})\right) = \frac{\sqrt{-(\alpha + \beta)^2 + 4}}{2}$. At this point the modulus is $r = 1$ and the maximum radial frequency is thus,
\begin{equation}
    \phi_{\max} = \arcsin{\frac{\sqrt{-(\alpha + \beta)^2 + 4}}{2}}.
\end{equation}
As $\lim\limits_{\tau \to \infty} e^{-\frac{\Delta t}{\tau}} = 1$, it is clear that $\lim\limits_{\tau_u,\tau_w \to \infty} \phi_{\max} = 0$.
\end{proof}
\subsection*{Benchmark Datasets and Preprocessing for Tables \ref{tab:lif_results} and \ref{tab:results}}
\label{subsec:datasets_and_preprocessing}
The SHD dataset was preprocessed by sum-pooling spikes temporally using bins of \qty{4}{\milli\second} (as in \cite{higuchi2024balanced}) and spatially using a bin size of $5$ channels, such that its input dimension was reduced from $700$ to $140$ channels, as in \cite{hammouamri2023learning}. Note, that any resulting preprocessed sample $X\in \mathbb{N_+}^{T\times140}$ of length $T$ thereby has integer-valued entries, where each entry $x_{kj}$ denotes the number of spikes occurring during the $k$-th $4$ ms time window within the $j$-th group of $5$ channels in the raw data. We padded samples that were shorter than $250$ time steps to a minimum length of $250$ with zeros to ensure that the network has enough 'time' for a decision, but kept longer sequences as they were. We applied the same temporal and spatial pooling to the SSC dataset, but zero-padded the samples to a minimum length of $300$ time steps, since the relevant part of the data usually appears later in the sequence in SSC. 
For the ECG dataset, we used the preprocessed files from \cite{yinAccurateEfficientTimedomain2021}, where the two-channel ECG signals from the QT database \cite{ecg-qtdb} were preprocessed using a level-crossing encoding. For details refer to \cite{yinAccurateEfficientTimedomain2021}, Methods section. We considered two cases of the SHD dataset, one where we validated on the test set and chose the best epoch based on this validation (=test) accuracy, and one where we used $20\%$ of the training set as held-out validation set and performed testing on the test set, using the weights of the epoch with the highest validation accuracy. For SSC, a distinct validation set was provided. For ECG, we used a fraction of $5\%$ of samples from the training set for validation.

\subsection*{Training and Hyperparameter Search Details for all Tasks}
\label{subsec:training_and_hyerparams}
\paragraph{Optimizer and surrogate gradient} We trained the SNNs using the SLAYER surrogate gradient \cite{slayer} defined by $\frac{\partial S}{\partial v} = \frac{c \alpha}{2\exp\left(\alpha|v|\right)}$ with $\alpha$ and $c$ according to Table \ref{tab:hyperparameters} for all experiments. We found that careful choice of the scale parameter $c$ is crucial to achieve good performance for all networks in which the SLAYER gradient is used. A too high $c$ results in an exploding gradient, whereas a too small $c$ can result in vanishing gradients. We trained all networks with back-propagation through time with minibatches using PyTorch \cite{anselPyTorchFasterMachine2024}. We used the ADAM \cite{kingma2017adam} optimization algorithm for all experiments with $\beta_1=0.9$, $\beta_2=0.999$ and $\epsilon=10^{-8}$. For the LIF and adLIF models we detached the spike from the gradient during the reset, such that $u[k] = \hat{u}[k] \cdot (1-\text{sg}\left(S[k]\right))$ where sg is the stop-gradient function with $\text{sg}(x) = x$ and $\frac{\partial}{\partial x}\text{sg}(x) = 0$. \revision{We applied gradient clipping and rescaled the gradient if it exceeded a norm of $10$ (audio compression task) or $1.5$ (all other tasks).}
\revision{\paragraph{Network Ouput Layer} The output layer of the network consisted of leaky integrator (LI) neurons of the same number as classes in the task at hand. The LI membrane potential at time step $k$ is given by
\begin{equation}
    u[k] = \gamma u[k-1] + \left(1- \gamma\right) I[k]
\end{equation}
with $\gamma = \exp\left(-\frac{\Delta t}{\tau_\mathrm{out}}\right)$, time constant $\tau_\mathrm{out}$ and input current $I[k]$. LI neurons don't emit spikes, hence they lack a threshold and reset mechanism, instead their output is their membrane potential $u[k]$.}
\paragraph{Loss functions} In all tasks, the last layer of the network consisted of leaky integrator neurons that match the number of classes of the corresponding task, or the number of masses in case of the oscillatory dynamical system trajectory prediction task. \revision{For the audio reconstruction task we quantized the waveform into discrete bins and used a multi-component loss (please refer to Section \enameref{subsec:audio_compression_details})}. For the SHD and BSD tasks, the loss function was given by $L=\text{CrossEntropy}\left(\sum_t \text{softmax}(\bm{y}_t), \bm{c} \right)$ for one-hot encoded target class $\bm{c}$ and network output $\bm{y}^t$ at time step $t$. We discarded the network output of the first $10$ time steps for the calculation of the loss for SHD, whereas for the BSD task we discarded the output of the first $80\%$ of time steps. For SSC, we used the loss function $L=\text{CrossEntropy}\left(\text{softmax}(\sum_t \bm{y}_t), \bm{c} \right)$. Again, we discarded the network output of the first $10$ time steps.
For the ECG dataset, the loss was computed on a per-time-step level as $L=\sum_t \text{CrossEntropy}\left(\text{softmax}(\bm{y}_t), \bm{c}_t \right)$ where $\bm{c}_t$ is the label per time step $t$. 
For the trajectory prediction task, we used the mean-squared-error (MSE) loss over the temporal sequence, $L=\frac{1}{n\times T}\sum^{T}_{t=1}\sum^{n}_{j=1}\left(\overset{*}{y}^{t}_j - y_j^t\right)^{2}$ where $T$ is the number of time steps in the sequence, $n$ the number of masses and $\overset{\ast}{\bm{y}}^t$ is the ground truth of the masses' displacements.
\paragraph{Hyperparameter tuning} The hyperparameters for the adLIF network were tuned using a mixture of manual tuning and the Hyperband algorithm \cite{hyperband}, mainly on the SHD classification task. Since the search space of hyperparameters of the adLIF model is quite large, we neither ran exhaustive searches on the ranges for time constants $\tau_u$ and $\tau_w$, nor the ranges of parameters $a$ and $b$ (for details on how we train these parameters see next section). We selected the ranges based on our stability analyses to ensure stable neurons, and on previous empirical results on similar models \cite{bittar2022surrogate}. For SSC
and ECG we manually tuned only the number of neurons, the learning rate and the SLAYER gradient scale $c$ and kept the values for other hyperparameters same as for the SHD task. We found, that the out-of-the-box performance for these hyperparameters was already very competitive, providing a solid choice as a starting point.
For the trajectory prediction task, we performed the hyperparameter search using the Hyperband algorithm \cite{hyperband}.
Only the LIF model was highly sensitive to its hyperparameters \--- we found that the time constants $\tau_u$ and the hyperparameters of the SLAYER gradient function were critical for learning.
For the LSTM network, we found that the learning rate should be lower than for LIF and adLIF, and similar to \cite{jozefowicz2015empirical}, we found that for the model to converge, the forget gate of LSTM should be initially biased to one.
For EF-adLIF and SE-adLIF, the performances were mostly invariant to hyperparameter changes, so we conserved the hyperparameters found for SHD and only restricted the parameters $a$ to correspond to the frequency bandwidths described in \enameref{sec:spring-mass}.
\paragraph{Reparameterization and initialization} For all models, we trained the time constants $\tau_u$ and $\tau_w$ (except for LIF, where no $\tau_w$ occurs) in addition to the synaptic weights. \revision{For all tasks except the audio compression task} we reparameterized the time constants via
\begin{equation}
\label{eq:tau_reparam}
    \tau_x = \tau_x^\text{min} + \theta_x \left(\tau_x^\text{max} - \tau_x^\text{min}\right),
\end{equation}
with $x \in \{u,w\}$, trained parameters $\theta_x$ clipped to the interval $\theta_x \in [0,1]$ during training and $\tau_x^\text{min}$ and $\tau_x^\text{max}$ as hyperparameters according to Table \ref{tab:hyperparameters}. We found this reparameterization useful if the ADAM optimizer is used, since its dynamic adjustment of the learning rates expects all parameters to be roughly in the same order of magnitude, which is usually not the case for joint training of neuron time constants (order of $10^1$ to $10^2$) and synaptic weights (order of $10^0$). \revision{In the audio compression task, we found it beneficial to use a different reparameterization scheme, where we directly parameterized $\alpha$ and $\beta$ according to:
\begin{align}
    \alpha &= \sigma\left(\theta_\alpha\right) \exp\left(-\frac{\Delta t}{\tau_u^\text{min}}\right) + (1-\sigma\left(\theta_\alpha\right)) \exp\left(-\frac{\Delta t}{\tau_u^\text{max}}\right), \\
    \beta &= \sigma\left(\theta_\beta\right) \exp\left(-\frac{\Delta t}{\tau_w^\text{min}}\right) + (1-\sigma\left(\theta_\beta\right)) \exp\left(-\frac{\Delta t}{\tau_w^\text{max}}\right),
\end{align}
with parameters $\theta_\alpha$ and $\theta_\beta$ and logistic sigmoid function $\sigma$. }

For the parameters $a$ and $b$ of the adLIF model we applied a \revision{reparameterization similar to Eq.~\eqref{eq:tau_reparam}}, where
\begin{align}
    a &= q \hat{a} \label{eq:reparam_a} \\
    b &= q \hat{b}, \label{eq:reparam_b}
\end{align}
with hyperparameter $q$ and trained parameters $\hat{a}$ and $\hat{b}$, clipped to $\hat{a} \in [0,1]$ and $\hat{b} \in [0, 2]$ during training\revision{, except for the audio compression task, where we clipped them according to Table \ref{tab:hyperparameters_audio_comp}}. We restrict the range of parameters to $a,b>0$ to avoid instabilities caused by a positive feed-back loop between adaptation variable $w$ and membrane potential $u$. This constraint has also been discussed in \cite{deckers2024co}. We initialized the feed-forward weights in all models uniformly in the interval $[-\sqrt{\frac{1}{\text{fan}_\text{in}}},\sqrt{\frac{1}{\text{fan}_\text{in}}}]$ with $\text{fan}_\text{in}$ as the number of inbound synaptic feed-forward connections, and all recurrent weights according to the orthogonal method described in \cite{saxe2014a} with a gain factor of $1$. $\hat{a}$, $\hat{b}$, $\theta_u$ and $\theta_w$ are initialized uniformly over their respective range, as stated above. Note, that these are neuron-level parameters, i.e. each neuron has individual values of these parameters.

\subsection*{Details for the dynamical system trajectory prediction task}
As illustrated in Fig.~\ref{fig:oscillator_task}a, we consider a system of $n$ masses connected with $n+1$ springs where each mass (except the two outermost) is connected to two other masses by a spring. The rightmost and leftmost masses are connected to one mass and the fixed support each (see schematic in Fig.~\ref{fig:oscillator_task}a).
The temporal evolution of the displacements of the masses $\bm{x}(t) \in \mathbb{R}^{n}$ can be written using the equation of motion
\begin{equation}\label{eq:motion_spring_mass}
    M\Ddot{\bm{x}}(t) + S\bm{x}(t) = 0,
\end{equation}
where $M \in \mathbb{R}^{n\times n}$ is a diagonal matrix with diagonal entries corresponding to the masses in kg, while $S \in \mathbb{R}^{n\times n}$ corresponds to the matrix of interaction between the masses determined by the spring coefficients. $S$ is given by 
\begin{equation}
    S_{ij} = 
    \begin{cases} 
    s_i + s_{i+1} & \text{if } i = j \\
    -s_i & \text{if } j = i-1 \\
    -s_{i+1} & \text{if } j = i+1 \\
    0 & \text{otherwise}
    \end{cases}
\end{equation}
with $\bm{s} \in \mathbb{R}^{n+1}$ as the spring coefficients in \qty{}{\newton\per\meter}.
We solve this system by considering velocity vector $\bm{v}(t)= \dot{\bm{x}}(t)$, which results in an equivalent system of $2n$ equations of the form
\begin{equation}
\underbrace{\begin{pmatrix}
     I & \bm{0} \\
    \bm{0} & M
    \end{pmatrix}}_{A}
    \begin{pmatrix}
    \dot{\bm{x}}(t) \\
    \dot{\bm{v}}(t)
    \end{pmatrix} = \underbrace{\begin{pmatrix}
     \bm{0} & I\\
    -S & \bm{0}
    \end{pmatrix}}_{B}
    \begin{pmatrix}
    \bm{x}(t) \\
    \bm{v}(t)
    \end{pmatrix},
\end{equation}
with $\bm{0}$ and $I \in R^{n \times n}$ representing the zero-valued and the identity matrices respectively.
The system has a homogeneous solution of the form
\begin{equation}
    \begin{pmatrix}
    \bm{x}(t) \\
    \bm{v}(t)
    \end{pmatrix} = \exp{(A^{-1} Bt)} \begin{pmatrix}
    \bm{x}(0) \\
    \bm{v}(0)
    \end{pmatrix},
\end{equation}
where $\bm{x}(0), \bm{v}(0) \in \mathbb{R}^n$ correspond to the initial conditions of displacements and velocities of the masses respectively.

For each independent trial in Fig.~\ref{fig:oscillator_task}e, we constructed an individual spring-mass system with $n=4$ masses and $5$ springs by generating random spring coefficients $s_i$ for each spring $i$ from a uniform distribution over the interval $[500, 10000]$ \qty{}{\newton\per\meter}, whereas the masses' magnitudes were fixed to $1$ kg. The system's parameters were then held fixed throughout the trial. We sampled $4{,}096$ trajectories from this system with random initial conditions to construct the dataset for a trial.
Each network was then trained on this per-trial dataset for $200$ epochs. We repeated the experiment for $5$ independent trials.

The chosen range of spring coefficients resulted in eigenfrequencies of the systems between approximately $2$ to $32$Hz.

For Fig.~\ref{fig:oscillator_task}g, we trained our models on different systems with increasing minimal and maximal frequency.
Since the frequency of oscillations is determined by $S$ and the spring coefficients, we seek to associate a range of frequencies with a range of spring coefficients from which to sample system parameters.
The theoretical frequency bandwidth associated with Eq.~\eqref{eq:motion_spring_mass} can be approximated by calculating the eigenfrequency of the system when all springs are set to the same coefficient $s$, resulting in the matrix  $\hat{S}$ of spring coefficients
\begin{equation}
    \hat{S}_{ij} = 
    \begin{cases} 
    2s & \text{if } i = j \\
    -s & \text{if } j = i-1 \\
    -s & \text{if } j = i+1 \\
    0 & \text{otherwise}.
    \end{cases}
\end{equation}
In the case of unit masses, we can formulate an eigenvalue problem from the system defined by Eq.~\eqref{eq:motion_spring_mass} as
\begin{equation}
    \hat{S}\bm{v}_j = \omega^2_j \bm{v}_j,
\end{equation}
where $\bm{v}_j$ is the eigenvector associated to the eigenvalue $\omega^2_j$ and $\omega_j$ is the corresponding radial frequency.
$\hat{S}$ is a tridiagonal Toeplitz matrix, as such the $j$-th eigenvalue associated to mass $j$ has a closed form solution \cite{kulkarni1999eigenvalues}  
\begin{equation}
\omega^2_j = 4s\sin^{2}{\left(\frac{j\pi}{2\left(n + 1\right)}\right)},
\end{equation}
where $s$ is the spring coefficient and $n$ the number of masses.
The maximal eigenvalue is thus given by
\begin{equation}
    \omega^2_{\max} = 4s\sin^{2}{\left(\frac{n\pi}{2(n + 1)}\right)} \approx 4s,
\end{equation}
and the radial frequency by
\begin{equation}
    \omega_{\max} \approx 2\sqrt{s}.
\end{equation}
The minimal eigenvalue corresponds to 
\begin{equation}
    \omega^2_{\min} = 4s\sin^{2}{\left(\frac{\pi}{2(n + 1)}\right)} \approx 4s \left(\frac{\pi}{2(n+1)}\right)^{2},
\end{equation}
and the radial frequency is thus given by
\begin{equation}
    \omega_{\min} \approx \frac{\pi}{n+1} \sqrt{s}.
\end{equation}
The range of spring coefficients can be determined by setting $\omega_{\min}$ (resp. $\omega_{\max}$) to the desired minimum (resp. maximum) radial frequency and solving for $s$.

For each data sample, corresponding to the displacement trajectory $X \in \mathbb{R}^{n\times T}$, we randomly generated an initial condition consisting of an initial displacement $x_{i,0}$ sampled from a standard normal distribution $\mathcal{N}(0, 1)$ and zero-valued initial velocity, for each mass $i\in [1,n]$. We then simulated the temporal evolution of this system, such that the $k$-th column of $X$ was given by displacements of the masses at time $k\Delta t$ with simulation time step $\Delta t = \qty{2.5}{\milli\second}$ for \qty{500}{\milli\second} of simulation during training, totaling 200 time-steps. The vector $\bm{x}[k]$ in the main text was then given by the $k$-th column of $X$. The velocities were intentionally held out from the training data to increase the task difficulty and enforce the utilization of internal states by the neural network. For Fig.~\ref{fig:oscillator_task}e, the model with $958$ parameters corresponds to an adLIF network of $25$ neurons. We doubled the number of neurons until $3200$, neuron counts of other models (LIF and LSTM) were scaled such that these models match the number of trainable parameters in the corresponding adLIF network.

\subsection*{\revision{Details for the audio compression task}}
\label{subsec:audio_compression_details}
\revision{
For the audio compression task from Fig.~\ref{fig:audio_compression_task} we used the 'train-clean-100' dataset from the LibriTTS corpus \cite{zen2019libritts,panayotovLibrispeechASRCorpus2015}, comprising $53.78$ hours of raw recorded speech data from audiobooks, split into $33{,}200$ samples of varying length in the interval of $[0.16, 32]$ seconds. In this dataset, each sample is encoded with $16$ bit Pulse Code-Modulation (PCM) sampled at \qty{24}{\kHz}. We first rescaled each sample $x$ individually by dividing it by its peak amplitude $|\text{max}(x)|$, ensuring that the rescaled sample had a maximum absolute amplitude of $1$. This rescaling did not change the zero-amplitude level of the data. Next, we segmented each rescaled sample into non-overlapping blocks of $2{,}560$ time steps (equivalent to \qty{106}{\milli\second}) and treated each of these blocks as individual, independent sample $\bm{x}_\text{w} \in \mathbb{R}^{2560}$ for training. Since the task was the reconstruction of the original signal, akin to an autoencoder, the input time series $\bm{x}_\text{w}$ and the target time series $\overset{*}{\bm{y}}_\text{w} \in \mathbb{R}^{2560}$ were equivalent. We use the subscript $\text{w}$ to denote that both $\bm{x}_\text{w}$ and $\overset{*}{\bm{y}}_\text{w}$ represent amplitudes in the waveform domain.}

\revision{To train spiking neural networks for waveform reconstruction, we considered a spectral and a temporal loss.
For the temporal loss, our objective was to maximize the likelihood of the target amplitude $\overset{*}{y}_\text{w}[k+1]$ at the $k+1$-th time step, conditioned on the sequence of preceding amplitude values $(x_\text{w}[1], \dots, x_\text{w}[k])$. This optimization problem can be formalized to
\begin{equation}
    \arg\max_{\theta} P(\overset{*}{\bm{y}}_\text{w};\theta) = \prod^{T-1}_{k=1}  p(\overset{*}{y}_\text{w}[k+1] | x_\text{w}[1], \dots, x_\text{w}[k]; \theta),
\end{equation}
where $\theta$ are the model parameters. In line with recent studies (e.g. Wavenet \cite{van2016wavenet} and sampleRNN \cite{mehriSampleRNNUnconditionalEndtoEnd2017}), we quantize $\overset{*}{\bm{y}}_\text{w}$ into $256$ discrete levels, treating $p$ as a categorical distribution.
The output layer of all our simulated networks consisted of $256$ leaky-integrator (LI) neurons representing the logits associated with each of these discrete levels. 
The target class associated with each time step was defined by the quantization scheme:
\begin{equation}
    \overset{*}{y}_\text{q}[k] = \left\lfloor\frac{f_{A}( \overset{*}{y}_\text{w}[k] + 1)}{\Delta_\text{q}}\right\rceil,
\end{equation}
where $\left\lfloor x \right\rceil$ denotes the rounding operator, $\Delta_q = \frac{2}{2^{8} - 1}$ is the discretization interval, and subscript $q$ indicates the quantized, categorical representation. The transformation function $f_{A}(y_\text{w}[k])$
represents the A-law non-linearity \cite{a_law}, defined as:
\begin{equation}
\label{eq:a_law}
f_{A}(x) = \text{sign}(x)
\begin{cases}
\frac{A |x|}{1 + \ln(A)} & \text{if } |x| < \frac{1}{A}, \vspace{5pt}\\
\frac{1 + \ln(A |x|)}{1 + \ln(A)} & \text{if } |x| \geq \frac{1}{A},
\end{cases}
\end{equation}}
\revision{where $A=86.7$ is a hyperparameter, and $|x|$ denotes the absolute value of $x$. This type of invertible log-space mapping was previously used for audio generation in \cite{van2016wavenet} and is a common technique used in $8$ bit telephony to improve noise robustness \cite{a_law}, as it is less sensitive to low amplitude noise while maintaining high precision for significant amplitude magnitudes. 
The original audio data is encoded with $16$ bit, hence for perfect reconstruction with a categorical distribution this would require $65536$ classes. By using a companding algorithm like the A-law, we limit the effects of reducing our precision to $8$ bit.
The temporal loss is finally defined as
\begin{equation}
    \mathcal{L}_\text{temp}(\bm{y}_\text{q}, \overset{*}{\bm{y}}_\text{q}) = \text{CrossEntropy}(\bm{y}_\text{q}, \overset{*}{\bm{y}}_\text{q})
\end{equation}
where $\bm{y}_\text{q}[k] = \text{softmax}(\frac{\bm{o}[k]}{\tau})$ of model output $\bm{o}[k] \in \mathbb{R}^{256}$  at time $k$ and $\tau \geq 1$ is a temperature hyperparameter, which we initialized at $\tau=10$ and reduced by a factor of $0.95$ every $2000$ training batches, until a minimum value of $1$.}

\revision{From this network output $\bm{y}_\text{q}[k]$ in the quantized domain we obtained the waveform-amplitude $y_\text{w}[k]$ by computing a convex sum over the quantization levels via
\begin{equation}
    y_\text{w}[k] = f^{-1}_A(\sum\limits^{256}_{i=1}(i - 1)\cdot y^i_\text{q}[k]\cdot\Delta_{\text{q}} - 1).
\end{equation}
}
\revision{Here, $f^{-1}_A$ denotes the inverse of the A-law operation (Eq.~\eqref{eq:a_law}) and superscript $i$ denotes the $i$-th entry of the quantized vector.
While temporal objectives, such as the one described above, have been used successfully in important studies \cite{mehriSampleRNNUnconditionalEndtoEnd2017, van2016wavenet}, modern approaches to audio generation and compression generally rely on spectral objectives \cite{zeghidourSoundStreamEndtoEndNeural2021,defossez2022high}.
In this work we found that combining spectral and temporal objectives gives the best results.
From the reconstructed waveform $\bm{y}_\text{w}$ and the target waveform $\overset{*}{\bm{y}}_\text{w}$, we computed a spectral loss based on a multi-resolution Mel-spectral short-term Fourier transform (STFT) loss \cite{gritsenko2020spectral} defined as
\begin{equation}
\label{eq:l_w}
    \mathcal{L}_\text{w} = \frac{1}{6}\sum_{k \in \{2^6, 2^7, \ldots, 2^{11}\}} \left(\mathcal{L}^{k}_{lin} + \mathcal{L}^{k}_{log}\right),
\end{equation}
where
\begin{align}
\mathcal{L}^{k}_{lin}(\bm{y}_\text{v},  \overset{*}{\bm{y}}_\text{w}) &=
\frac{1}{T  N}\| |\text{STFT}^{k}_\text{Mel}(\bm{y}_\text{v})| - |\text{STFT}^{k}_\text{Mel}(\overset{*}{\bm{y}}_\text{w})| \|_1 \\
\mathcal{L}^{k}_{log}(\bm{y}_\text{v}, \overset{*}{\bm{y}}_\text{w}) &=
\frac{1}{T  N} \| \ln(|\text{STFT}^{k}_\text{Mel}(\bm{y}_\text{v})|) - \ln(|\text{STFT}^{k}_\text{Mel}(\overset{*}{\bm{y}}_\text{w})|) \|_1.
\end{align}
Here, $\text{STFT}^{k}_\text{Mel}$ is the unnormalized short-term Fourier transform operator over a $k$-length Hann window using the HTK-variant of the Mel frequency filters.  $||\cdot||_1 $ denotes the  $L1$ norm, $T$ corresponds to the number of STFT frames, and $N$ the number of Mel filters ($N=128$). 
We consider a hop length of $k/4$ and the number of frequencies for the STFT was set to $2048$. The overall loss was defined as  
\begin{equation}
\label{eq:total_loss}
   \mathcal{L}_\text{total}\left(\bm{y}_\text{q}, \overset{*}{\bm{y}}_\text{w}\right) = \mathcal{L}_\text{w}\left(\bm{y}_\text{w},\overset{*}{\bm{y}}_\text{w}\right) + \mathcal{L}_\text{temp}(\bm{y}_\text{q}, \overset{*}{\bm{y}}_\text{q}).
\end{equation}
For every sample, during training, we ignored the model output during the first $50$ time steps (i.e. $2$ ms) as burn-in time for the model.}

\revision{
For all simulations (LIF, SE-adLIF, and EF-adLIF), the networks consisted of two fully-connected encoder layers, where only the second layer included recurrent connections. From the second layer, the first 
$16$ neurons were connected to the decoder network, enforcing the information bottleneck. The decoder network was also a two-layer network, with recurrent connections in both layers, followed by an additional non-recurrent output layer of $256$ leaky integrator (LI) neurons. All remaining hyperparameters are listed in Table \ref{tab:hyperparameters_audio_comp}. For each time step $k$, the network output $\bm{y}[k] \in \mathbb{R}^{256}$ was constructed by passing the output of the LI layer through a softmax. 
}

\revision{
In this task, we employed a small modification to the reset mechanism of the LIF and adLIF neuron models: Instead of resetting the membrane potential to $0$ after a spike, we reset it to a learnable reset potential $u_\text{reset}$. Each neuron thereby had a separate value for $u_\text{reset}$, clipped to the interval $[- \vartheta, \vartheta]$. In addition, we trained the spike threshold $\vartheta \in \mathbb{R}^{+}$ of each individual neuron, instead of treating it as fixed hyperparameter, as done in all our other tasks.
We found that these modifications substantially improved the performance of the SNN networks. Another modification we found beneficial was to delay the target time series by a delay of $20$ time steps ($\approx$\qty{0.8}{\milli\second}), such that the model had more 'time' to encode the waveform into spikes. In other words, at time step $k$, the prediction target for the model was $\overset{*}{\bm{y}}_\text{q}[k-20]$. During training, we applied spike regularization to each layer $l$ of the network according to:
\begin{equation}
    \mathcal{L}_\text{reg}(r^{n}_l) =
\begin{cases}
g_l^+ \left(r^{n}_l -  t_{l}^+\right)^2 & \text{if } r^n_l > t^+_l, \\
g_l^- \left(r^{n}_l - t_{l}^- \right)^2 & \text{if } r^n_l < t^-_l, \\
0 & \text{otherwise},
\end{cases}
\end{equation}
where $r^n_l$ is the average number of spikes of neuron $n$ in layer $l$ over time, $t_{l}^+$, $t_l^-$, $g_l^+$ and $g_l^-$ are layer-specific hyperparameters according to Table \ref{tab:audio_spike_reg}. $\mathcal{L}_\text{reg}$ was added to the total loss from Eq.~\eqref{eq:total_loss}.
}

\begin{table}[!ht]
    \centering
    \caption{\revision{\textbf{Spike regularization parameters for the audio compression task, LIF and AdLIF models} *) for the second encoder layer, the spike regularization was only applied to the first $16$ neurons, which constitute the bottleneck and are the only neurons connected to the decoder.}}
    \label{tab:audio_spike_reg}
    \begin{tabular}{ccccc}
    \hline
        Parameter & \rotatebox{90}{Encoder layer 1} & \rotatebox{90}{Encoder layer 2 *} & \rotatebox{90}{Decoder layer 1} & \rotatebox{90}{Decoder layer 2} \\ \hline
        $t_{l}^-$ & $0.05$ & $0.005$ & $0.05$ & $0.05$  \\
        $t_l^+$ & $0.1$ / $0.5$ (LIF) &  $0.012$ & $0.1$ / $0.5$ (LIF) & $0.6$  \\
        $g_l^-$ & $10$ & $10$ & $10$ & $10$\\
         $g_l^+$ & $10$ & $100$ & $10$ & $10$ \\ \hline         
    \end{tabular}
\end{table}

\subsection*{Details for simulations in Fig.~\ref{fig:gradient_magnitude}}
The neuron parameters for the experiments in Fig.~\ref{fig:gradient_magnitude} were $\tau_u=\qty{100}{ms}$ for LIF and $\tau_u=\qty{100}{ms}$, $\tau_w=\qty{300}{ms}$, $a=300$ for adLIF, resulting in an adLIF neuron with an intrinsic oscillation frequency $f\approx$ \qty{16}{Hz}. We used the first derivative of a Gaussian function as wavelet, which was scaled such that the central oscillation frequency was $\approx \qty{17}{Hz}$. The offsets for the wavelet in Fig.~\ref{fig:gradient_magnitude}e were $-32$ ms and $-48$ ms as $-\frac{1}{2} P$ and $-\frac{3}{4} P$ respectively. The loss was evaluated at time $T=\qty{330}{ms}$.

\subsection*{Details for the Burst Sequence Detection (BSD) task}
\label{subsec:meth_bsd_details}
The BSD task from Fig.~\ref{fig:optimization_based_sampling}a,b is a $20$-class classification task consisting of $8{,}000$ samples. Each binary-valued sample $X \in [0,1]^{T\times N}$ in this task consists of spike trains of $N=10$ input neurons and a time duration of T=\qty{200}{\milli\second} in discrete \qty{1}{ms} time steps. The objective in this task is to classify any sample $X$ based on the appearance of spike bursts of specific neurons at specific timings. The timing and neurons for these class-descriptive bursts were pre-assigned upfront and kept fixed for the generation of the entire dataset. To generate this dataset, we employed a two-step process where we first randomly assigned class-descriptive burst timings for each class, then stochastically sampled data samples based on these pre-assigned timings. In detail, this procedure was as follows:
For the pre-assignment of burst timings, we sampled a random subset $\mathcal{S}_c$ of $3$ input neurons for each class $c$ and sampled a random time point $t^n_c$ uniformly over $[20,170]$ for each neuron $n \in \mathcal{S}_c$ for class $c$. These time points served as the class-descriptive burst timings.

After generating the class-specific burst-timings for all classes, the data samples were generated. To generate a sample of a given class $c$, we first determined the burst timing for each neuron $n\in [1,N]$ as follows. If neuron $n$ is one of the $3$ neurons inside the pre-defined set $\mathcal{S}_c$ of class $c$, then its burst-timing $t^n$ is determined by the pre-assigned timing $t^n_c$. Otherwise, we randomly selected a burst timing $t^n$ uniformly over $[20,170]$ for neuron $n$. In the example shown in Fig.~\ref{fig:optimization_based_sampling}a, a sample of class $2$ is defined by a spike burst of neuron $4$ at \qty{159}{\milli\second}, a burst of neuron $5$ at \qty{143}{\milli\second}, and a burst of neuron $10$ at \qty{117}{\milli\second}. Hence, the set $\mathcal{S}_2$ of class-descriptive neurons for class $2$ is $\{4,5,10\}$ with burst timings $t^4_2=159$, $t_2^5=143$ and $t_2^{10}=117$. Only if all of the neurons assigned to $S_c$ emit a burst at their corresponding timings $t^c_n$, the sample should be classified as class $c$. All other input neurons show distraction bursts at random timings.

In an input sample $X \in [0,1]^{T\times N}$, $x_{t,n}=1$ indicates a spike of input neuron $n$ at time $t$ ($x_{t,n}=0$ if there is no spike).
We generated the spike trains $X$ for input neurons for a given class $c$ as follows. Each $x_{t,n}$ was drawn from a Bernoulli distribution with $p(x_{t,n}=1|c)$ that was obtained as follows.
Bursts were modeled as brief, smooth increases in spike probability at the corresponding times $t^n$. We achieved this by a Gaussian function $f(t,t^n) = \exp\left(\frac{-(t-t^n)^2}{4}\right)$, yielding high spike probability for time steps close to $t^n$, and lower spike probability further away. From this function, we computed the final spike probability $p(x_{t,n}=1|c) = \frac{f(t,t^n)}{\max_{k\in[0,T]} f(k,t^n)} \cdot 0.75 + 0.05$, interpolating between the minimum and maximum spike probabilities of $0.05$ ($50$ Hz) and $0.8$ ($800$ Hz) respectively. This way we ensured that the spike probability went up to $0.8$  during a burst and was approximately $0.05$ otherwise to mimic background noise. 
This procedure was repeated for each individual data sample $X$ to obtain the dataset. We randomly chose a class for each new sample, such that the class cardinalities in the dataset were roughly but not exactly balanced. The whole data set consisted of $8000$ samples. We held out $10\%$ ($800$ samples) of this dataset for validation and $20\%$ ($1600$ samples) for testing. For the experiment with different numbers of classes (Fig.~\ref{fig:optimization_based_sampling}b), we only increased the number of classes, but did not increase the number of samples in the dataset. This resulted in increased difficulty for larger numbers of classes, since less samples per class were present in the training data.
\subsection*{Optimization-based feature visualization}
\label{subsec:meth_opt_based_feature_viz_details}
The visualization of important input features for the trained network shown in Fig.~\ref{fig:optimization_based_sampling} was obtained using optimization-based feature visualization as defined in \cite{erhan2009visualizing}.
We applied the same sampling method to four different networks, one adLIF and one LIF network trained on BSD and one adLIF and one LIF trained on SHD. For both cases we used the same algorithm, but with different parameters, which are listed in Table \ref{tab:opt_sampling_params}. For both tasks, we selected network instances that achieved high accuracy. The sampling procedure was performed as follows. First, we initialized a random input example $X^0\in\mathbb{R}^{T\times C}$ with $C$ as the input dimension and $T$ the input length (number of time steps), both of the same dimension as the data the network was trained on. Here, each $x^0_{t,c}$ was drawn from the uniform distribution $\mathcal{U}(0,1)$. At each iteration $k$, the sample was passed through the network and the loss was calculated akin to the loss function used during training on the specific dataset (SHD or BSD) to obtain the gradient $G_k = \nabla_{X} L(X,c)|_{X^k}$, with respect to a pre-defined target class $c$. This gradient was normalized to obtain $\Delta X^k = \eta\frac{ G_k}{\zeta_k}$ with normalization factor $\zeta_k=\max\left\{|G_k|_\text{max}, \epsilon\right\}$ using an $\epsilon = 10^{-6}$ for numerical stability, and a step size $\eta$. $|\cdot|_\text{max}$ denotes the maximum norm returning the highest absolute value. After the gradient update, we clipped the sample $X^{k+1}$ to the positive range, as the data from both datasets, SHD and BSD, is all-positive. We then applied Gaussian smoothing on the data sample $X^{k+1}$. In the BSD case we applied a $1$D smoothing along the time axis using a Gaussian kernel $g(s)=\frac{1}{\sqrt{2\pi}\sigma}\exp\left(-\frac{s^2}{2\sigma^2}\right)$ via $X^{k+1} \leftarrow \nu X^{k+1} + \gamma (g*X^{k+1})$.  Decay $\nu \in [0,1]$ and smoothing coefficient $\gamma \in [0,1]$ are hyperparameters, $*$ denotes a convolution operation. The $\sigma$ parameter of the Gaussian kernel was linearly decayed from $\sigma_\text{init}$ to $0$ over the $n_\text{iter}$ iterations to decrease the effective regularization applied by the kernel smoothing. For the SHD case we applied a $2$D smoothing to the data sample over both the time dimension and the spatial dimension. For the BSD case, we normalized the sample $X^{k+1}$ with $X^{k+1} \leftarrow X^{k+1} \frac{\mu_\text{all}}{\mu_{k+1}}$, where scalar $\mu_\text{all}$ is the mean spike rate over all data samples and input neurons in the data, and $\mu_{k+1}$ the mean of sample $X^{k+1}$ over both the time and the input neuron axes. For the SHD setup, we ignored the loss for the first $100$ sequence time steps (compare with grey shaded area in Fig.~\ref{fig:optimization_based_sampling}e, whereas for the BSD case we ignored the loss for the first $80\%$ of time steps, which was also done during training of the network on the dataset (see Section \enameref{subsec:training_and_hyerparams}).

\begin{table}[]
    \centering
    
    \caption{\textbf{Hyperparameters for optimization-based feature visualization}}
    \label{tab:opt_sampling_params}
    \begin{tabular}{ccc}
    \hline
        & BSD & SHD \\ \hline
         $n_\text{iter}$ & $400$ & $200$  \\
         $\eta$ & $0.1$ & $0.1$ \\
         $\nu$ & $0.891$& $0.891$ \\
         $\gamma$ & $0.1$ & $0.1$ \\
         $\sigma_\text{init}$ & $5$ & $5$ \\ \hline
         
    \end{tabular}
\end{table}

\subsection*{Hyperparameters}

We present hyperparameters that remained unchanged for all tasks and models in Table \ref{tab:glob_hyperparameters} and the task- and model-specific hyperparameters in Table \ref{tab:hyperparameters}.

\begin{table}[!t]
\centering
    \caption{\textbf{Global hyperparameters} FF Init: initialization of feed-forward synaptic weights. Rec Init: Initialization of recurrent synaptic weights.}
    \label{tab:glob_hyperparameters}
\begin{tabular}{cc}
\hline
$\vartheta$        & $1$                                                                          \\
$\Delta t$         & $1$                                                                          \\
FF Init.           & $\mathcal{U}\left(-\sqrt{\frac{1}{\text{fan}_\text{in}}},\sqrt{\frac{1}{\text{fan}_\text{in}}}\right)$ \\
Rec Init.          & orthogonal \cite{saxe2014a}                                                           \\
\hline
\end{tabular}
\end{table}

\begin{table}[!t]
\centering
    \caption{\textbf{Task- and model-specific hyperparameters} lr: learning rate for ADAM optimizer, Ep.: number of training epochs, $\tau_\text{out}$: Membrane time constants of output layer (leaky integrator), dropout: dropout rate, \revision{$q$: coefficient of the reparameterization of $a$ and $b$ (see Eq.~\eqref{eq:reparam_a} and \eqref{eq:reparam_b}).}
    \label{tab:hyperparameters}}
\adjustbox{max width=\textwidth}{%
\begin{tabular}{cccccccccccccc}
\hline 
                     &          & lr    & \rotatebox{90}{\# neurons} & \rotatebox{90}{\# layers} & q  & \rotatebox{90}{$\alpha$ (SLAYER)}  & \rotatebox{90}{c (SLAYER)} & Ep. & \rotatebox{90}{[$\tau_u^\text{min}, \tau_u^\text{max}$]} & \rotatebox{90}{[$\tau_w^\text{min}, \tau_w^\text{max}$]} &  \rotatebox{90}{dropout} & $\tau_\text{out}$ & \rotatebox{90}{batch size} \\

                     \hline
\multirow{3}{*}{\text{SHD}} & \text{SE} & $0.01$  & $128 / 360$ & $1 / 2$ & $120$ & $5$ & $0.4$ & $300$   & $[5,25]$ & $[60,300]$ &  $15\%$ & $15$   & $256$ \\
                     & \text{EF} & $0.01$  & $360$       & $2$         & $60$ & $5$      & $0.4$    &  $300$   & $[5,25]$   & $[60,300]$                       & $15\%$              & $15$         & $256$      \\
                     & \text{LIF}      & $0.01$  & $360$       & $2$         & -   & $5$    & $0.1$    &  $300$   & $[5,150]$  & -                              & $15\%$              & $15$         & $256$     \\ \hline
                     
\multirow{3}{*}{\text{SSC}} & \text{SE} & $0.006$ & $720$       & $2$         & $120$ & $5$    & $0.4$       & $40$    & $[5,25]$   & $[60,300]$                          & $15\%$              & $15$           & $256$    \\
                     & \text{EF} & $0.006$ & $720$       & $2$         & $60$ & $5$     & $0.4$    & $40$    & $[5,25]$   & $[60,300]$                       & $15\%$              & $15$          & $256$     \\
                     & \text{LIF}      & $0.006$ & $720$       & $2$         & - & $5$      & $0.1$    &  $40$    & $[5,150]$  & -                               & $15\%$              & $15$      & $256$         \\ \hline
                     
\multirow{3}{*}{\text{ECG}} & \text{SE} & $0.01$  & $36$        & $1 / 2$     & $120$ & $5$    & $0.2$      & $400$   & $[5,25]$   & $[60,300]$                          & $15\%$              & $3$          & $64$     \\
                     & \text{EF} & $0.01$  & $36$        & $1$         & $60$ & $5$     & $0.2$    &  $400$   & $[5,25]$   & $[60,300]$                         & $15\%$              & $3$       & $64$        \\
                     & \text{LIF}      & $0.01$  & $36$        & $1$         & -    & $5$    & $0.1$    &  $400$   & $[5,150]$  & -                                & $15\%$              & $3$          & $64$     \\ \hline
                                                 
\multirow{2}{*}{\text{BSD}} & \text{SE} & $0.01$  & $512$       & $1$         & $120$ & $5$    & $0.4$      & $400$   & $[5,25]$   & $[60,300]$                       & $0\%$              & $15$             & $128$         \\
                     & \text{LIF}      & $0.006$ & $510$       & $1$         & - & $5$      & $0.2$    & $400$   & $[5,50]$   & -                          & $0\%$              & $15$          & $128$     \\ \hline
                     
\multirow{3}{*}{\text{spring-mass}} & \text{SE/EF} &   $0.01$    &    $[25, 3200]$       &    $1$      & $228$/$65$ & $5$        &   $0.4$     &   $200$    & $[5,25]$               &     $[60, 300]$                  &              $0\%$           &        $[1,20]$       & $256$  \\
                     & \text{LIF} &   $0.01$               &     $[27, 3202]$      &   $1$     & - & $10$      &  $0.5$     &    $200$   &    $[1,25]$      &  -                      &                  $0\%$        &        $[1,20]$   & $256$      \\
                     & \text{LSTM} & $0.001$          &   $[13, 1600]$        & $1$ & -  & - &   -   &   $200$    &     -     &          -             &           $0\%$             &       $[1,20]$  & $256$        \\ \hline
\end{tabular}
}
\end{table}

\begin{table}[!t]
\centering
\centering
    \caption{\revision{\textbf{Hyperparameters for the audio reconstruction task.} lr: learning rate for ADAM optimizer, which was reduced by a factor of $0.1$ if the SI-SNR on the validation set did not improve after a training epoch, but not further than $10^{-7}$, Ep.: number of training epochs, $\tau_\text{out}$: Membrane time constants of output layer (leaky integrator),  $q, \hat{a}, \hat{b}$: coefficient and clipping ranges of the reparameterization of $a$ and $b$ (see Eq.~\eqref{eq:reparam_a} and \eqref{eq:reparam_b}) respectively. 
    }}
    \label{tab:hyperparameters_audio_comp}
\begin{tabular}{ccccccccccccccc}
\hline 
    model &  lr & \rotatebox{90}{\# neurons per layer } & q & $\hat{a}$ & $\hat{b}$ & \rotatebox{90}{$\alpha$ (SLAYER)}  & \rotatebox{90}{c (SLAYER)} & Ep. & \rotatebox{90}{[$\tau_u^\text{min}, \tau_u^\text{max}$]} & \rotatebox{90}{[$\tau_w^\text{min}, \tau_w^\text{max}$]} & $\tau_\text{out}$ & \rotatebox{90}{batch size} \\ \hline         
    \text{SE-adLIF} & $5\cdot 10^{-4}$ &  $300$ & $120$ & $[0,5]$ & $$[0, 2]$$ & $5$ & $0.4$ & $10$ & $[5, 25]$ & $[30, 300]$ & $[1, 10]$ & $128$\\
    \text{EF-adLIF} & $5\cdot 10^{-4}$ &  $300$ & $20$ & $[0, 1]$  & $[0, 2]$ & $5$ & $0.4$ & $10$ & $[5, 25]$ & $[30, 300]$ & $[1, 10]$ & $128$ \\
    \text{LIF} & $5\cdot 10^{-4}$ & $302$ & $-$ & $-$ & $-$ & $5$ & $0.1$ & $10$ & $[5, 100]$ & $-$ & $[1,10]$ & $128$ \\
\end{tabular}
\end{table}

\section*{Data Availability}

The SHD and SSC datasets are publicly available under \href{https://zenkelab.org/resources/spiking-heidelberg-datasets-shd/}{https://zenkelab.org/resources/spiking-heidelberg-datasets-shd/}. The QTDB ECG dataset is publicly available under \href{https://physionet.org/content/qtdb/1.0.0/}{https://physionet.org/content/qtdb/1.0.0/}. However, we used the preprocessed files from \cite{yinAccurateEfficientTimedomain2021}.
\section*{Code Availability}

The code is available under the CC BY-SA 4.0 license at the repository \href{https://github.com/IGITUGraz/SE-adlif}{https://github.com/IGITUGraz/SE-adlif}.

\printbibliography

\section*{Acknowledgements}

This work has been supported by the “University SAL Labs” initiative of Silicon Austria Labs (SAL) and its Austrian partner universities for applied fundamental research for electronic based systems (M.B., R.F., S.S.), by the European Community's Horizon 2020 FET-Open Programme, grant number 899265, ADOPD (R.L., M.B., R.F), and by NSF EFRI grant \#2318152 (R.L.). This research was funded in whole or in part by the Austrian Science Fund (FWF) [10.55776/COE12] (R.L., M.B.). For the purpose of open access, the author has applied a CC BY public copyright licence to any Author Accepted Manuscript version arising from this
submission.
We thank Felix Effenberger and Sebastian Otte for valuable feedback on the manuscript.

\section*{Author Contributions}

M.B., R.F., and R.L. conceived the idea and proposed the research, M.B., R.F. and S.S. carried out experiments, M.B. and R.F. performed theoretical analyses, M.B., R.F., and R.L. wrote the paper.

\section*{Competing Interests}

The authors declare no conflict of interest.

\pagebreak
\appendix

\section*{\revision{Supplementary Notes 1}}
\label{supnotes1}
The SE discretization can be applied to other neuron models with bi-directional feedback between neuron state variables. We demonstrate its benefits on a successful recent neuron model, the balanced harmonic resonate-and-fire (BHRF) neuron \cite{higuchi2024balanced}. The neuron is described by continuous-time equations of a damped harmonic oscillator as
\begin{align}
    \dot{u}&=-2bu-\omega^2v+I \label{eq:bhrfcontu} \\
    \dot{v}&=u \label{eq:bhrfcontv},
\end{align}
with neuron states $u$ and $v$, input current $I$, frequency parameter $\omega$, and damping coefficient $d$. In contrast to the adLIF neuron model discussed in the main text, where we can compute the oscillation frequency and decay from the state transition matrix, the BHRF neuron model is directly parameterized by the frequency parameter $\omega$ and damping parameter $d$. The parameter $\omega$ thereby directly defines the natural oscillation frequency (which is the oscillation frequency in the absence of damping). Similarly, the damping parameter $d$ directly defines the rate of decay of the neuron. In the original work by Higuchi et al.~\cite{higuchi2024balanced}, the neuron model is discretised using the Euler Forward method to obtain the following discrete state update equations for discrete time step $k$:
\begin{align}
    u[k] &= u[k-1] + \delta \left(-bu[k-1]-\omega^2v[k-1]+I[k]\right) \\
    v[k] &= v[k-1] + \delta u[k-1], \label{eq:bhrf_ef_v}
\end{align}
where $\delta$ is the discretization time step and $I[k]$ is the time-dependent neuron input. The equations above are simplified by omitting some of the mechanisms from the discrete model proposed by the authors to reduce complexity. By using the same approach as described in the main text (Section \emph{Stability analysis of discretized adLIF models}), we can calculate the effective oscillation frequency $\omega_\text{eff}$ and effective damping coefficient $b_\text{eff}$ of the discrete BHRF neuron, and investigate how these two characteristics relate to the neuron parameters $\omega$ and $b$. The effective frequency $\omega_\text{eff}$ can thereby be obtained via the same method as described in the main text, and $b_\text{eff}$ can be computed from the decay rate $r$, by $b_\text{eff} = -\frac{\ln r}{\delta}$. The decay rate $r$ is given by the spectral radius of the state transition matrix of the discretised BHRF neuron (analogous to the method applied to the adLIF model in the main text). The relationship between parameters and effective neuron dynamics is shown in Supplementary Fig.~\ref{fig:bhrf_se}a, where we show $b_\text{eff}$ (top) and $\omega_\text{eff}$ (bottom) with respect to the neuron parameters. 

The top panel of Supplementary Fig.~\ref{fig:bhrf_se}a reveals that the identity-relationship between $b_\text{eff}$ and parameter $b$, which is an inherent property of the continuous-time BHRF neuron from Equations \eqref{eq:bhrfcontu} and \eqref{eq:bhrfcontv}, is disrupted by the Euler-Forward discretization. The same disruption can be observed for $\omega_\text{eff}$ and $\omega$ in the bottom panel. As we show below, this disruption can be entirely eliminated, if instead of the Euler-Forward discretization the SE-discretization is applied to the BHRF. We can apply the SE discretization to the BHRF by simply replacing $u^{t-1}$ with $u^t$ in Eq.~\eqref{eq:bhrf_ef_v}, resulting in the update equation
\begin{equation}
    v[k] = v[k-1] + \delta u[k],
\end{equation}
which we refer to as SE-BHRF in the following.
This minor change in the discretised eqution rescues the identity-mapping between parameter $b$ and effective damping $b_\text{eff}$ (Supplementary Fig.~\ref{fig:bhrf_se}b, top), as well as the near-identity-mapping between parameter $\omega$ and effective oscillation frequency $\omega_\text{eff}$ (Supplementary Fig.~\ref{fig:bhrf_se}b, bottom) from the continuous model. Note, that for very low frequencies and sufficiently high damping, a damped harmonic oscillator is overdamped and hence does not oscillate, resulting in the triangle-shaped anomaly region in the top-left corner of the $\omega$-$b$ plot shown in the top row of Fig.~\ref{fig:bhrf_se}. Furthermore, it can be observed that the contour lines in the bottom panels of Supplementary Fig.~\ref{fig:bhrf_se} are not exactly vertical. This results from the slight influence of $b$ on $\omega_\text{eff}$, a naturally occurring phenomenon in a damped harmonic oscillator. Applying the Euler Forward discretization additionally introduces a region of unstable neuron dynamics, that can be related to our observations of the EF-adLIF model shown in Fig.~3 in the main text. The authors of \cite{higuchi2024balanced} account for this unstable region by computing $b$ as a relative offset from the divergence boundary (the boundary of the unstable region), but as shown in Supplementary Fig.~\ref{fig:bhrf_se}, such region does not exist for the SE-BHRF, rendering additional stability measures such as the relative offset obsolete. 

\section*{Supplementary Figure 1}
\label{supfig1}
\nopagebreak
\begin{figure}[H]
    \centering
    \includegraphics[width=\textwidth]{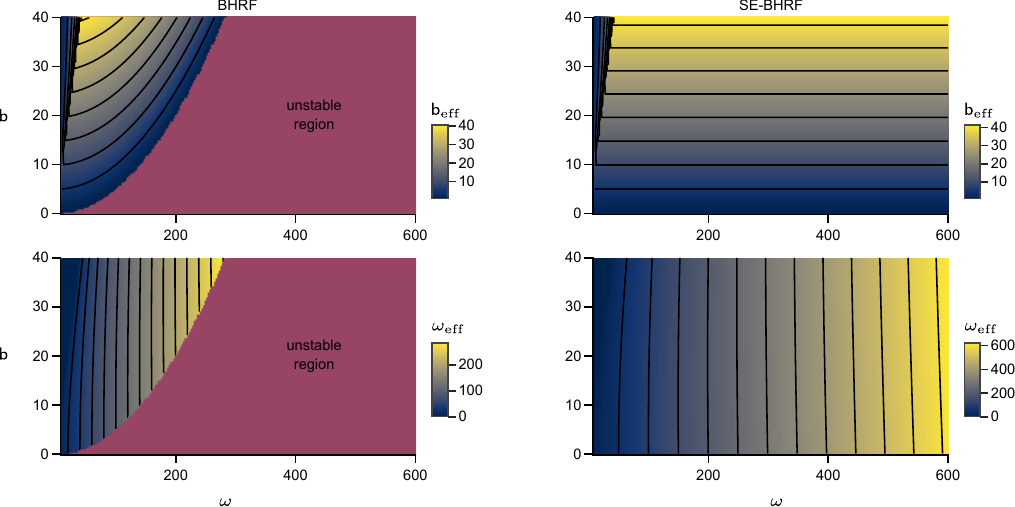}
    \caption{\textbf{Comparion between the Euler-Forward and Symplectic-Euler discretization method for the Balanced Harmonic Resonate-and-fire (BHRF) \cite{higuchi2024balanced} neuron model.} \textbf{(a)} Relationship between the BHRF-parameters $\omega$ and $b$ and the effective damping coefficient $b_\text{eff}$ (top panel) and the effective oscillation frequency $\omega_\mathrm{eff}$ (bottom panel) of the Euler-Forward-discretised BHRF model (details see text). The solid-colored area denotes the unstable region, where the neuron model is diverging. \textbf{(b)} Same as (a), but for the Symplectic-Euler-discretised BHRF model. Details see text. }
    \label{fig:bhrf_se}
\end{figure}

\section*{Supplementary Notes 2}
\label{supnotes2}
In the SHD, SSC, and ECG experiments with SE-adLIF neurons, we used the same parameter range as in Fig.~3c: $a \in [0,120]$, $\tau_u \in [5,25]$ ms and $\tau_w \in [60,300]$ ms. As shown in the same panel, this parameter range is unstable for EF-adLIF, hence we reduced the range of $a$ to $[0,60]$ for all experiments with EF-adLIF to ensure a decay rate $r<1$ for all parameter configurations. If the larger range of $a \in [0,120]$ is used for EF-adLIF, we can observe instabilities during training, resulting in significant degradation of performance. We report the performance drop on the ECG dataset in Supplementary Table \ref{tab:stab_ef_adlif}. Within these experiments, we observed exploding gradients during training in many of the runs with the extended range of $a$, which were not present for SE-adLIF with the same parameter range. Partially, these exploding gradients can be alleviated by using gradient clipping, resulting in slightly better performance. This experiment demonstrates the sensitivity of EF-adLIF stability to choices of parameter ranges. Note, that this sensitivity can directly be linked to the relation between neuron parameters and eigenvalues of $\eulerA$, as shown in Fig.~3.

\section*{Supplementary Table 1}
\label{suptab1}
\begin{table}[!h]
\centering
\begin{tabular}{clllll}
 \text{range of} $a$ & Grad. clip. & \text{Test Acc. [\%]} \\ \hline

 $[0,60]$ & \cmark & $87.10 \pm 0.46\%$ \\
 $[0,120]$ & \cmark & $70.66 \pm 24.90\%$ \\ 
 $[0,120]$ & \xmark & $55.62 \pm 24.84\%$ \\  \hline
\end{tabular}
\caption{\textbf{Instabilities of EF-adLIF for the ECG task.} Test accuracy for different ranges of parameter $a$ for EF-adLIF. Grad. clip. denotes rescaling the gradient to a norm of $1.5$ if exceeded. Table shows mean and std. dev. over $5$ runs. Chance level in this task is $1/6 = 16.67\%$. Note the high variance of accuracies in the less constrained cases.}
\label{tab:stab_ef_adlif}
\end{table}


\section*{Supplementary Notes 3}
\label{supnotes3}
In \cite{gerstner2014neuronal}, the authors show that adLIF models can, depending on their parametrization, account for many experimentally observed neocortical neuron spiking behaviors. It remained an open question whether these different neuronal dynamics and firing patterns could be beneficial in sequence processing tasks. To that end, we analyzed adLIF networks trained on the SHD task and observed a vast heterogeneity in the evolved neuron parameterizations, see Supplementary Fig.~\ref{fig:neuron_analysis}. 

We categorized the observed neurons based on features derived from their parameters and properties. One first distinction that can be made is whether the neurons exhibit integratory (overdamped, type I) or oscillatory membrane potentials (underdamped, type II), shown in Supplementary Fig.~\ref{fig:neuron_analysis}a. We found that most of the neurons ($87\%$) showed underdamped behavior. Most of these underdamped neurons converged to rather slow frequencies of $f<50$~Hz, Supplementary Fig.~\ref{fig:neuron_analysis}b. As shown in Fig.~2d and e, these slow-frequency neurons can act as feature detectors for input variations much slower than the spike rates. Another categorization can be made based on the parameter $b$: independent of oscillatory properties, this parameter determines the strength of feed-back from neuron spikes to the adaptation variable $w$. Interestingly, neurons with diverse degrees of such spike-triggered adaptation occured across the entire frequency range (color in Supplementary Fig.~\ref{fig:neuron_analysis}b). The distribution of $b$ was rather uniform, with two peaks where we clipped the parameter at the minimum and maximum range, representing weak and strong spike-triggered adaptation respectively. Interestingly, overdamped neurons acted on much longer timescales (see Supplementary Fig.~\ref{fig:neuron_analysis}d) than underdamped ones. This figure shows the distribution of the effective time constant $\tau_\mathrm{eff}  = -\frac{\Delta t}{\ln r}$, which is analogous to the membrane time constant of LIF neurons.

The heterogeneous combinations of the spike-triggered adaptation governed by parameter $b$ and the oscillatory behavior governed by parameter $a$ result in diverse non-linear super- and sub-threshold neuron dynamics (Supplementary Fig.~\ref{fig:neuron_analysis}e). Based on the overdamped-underdamped split and the distinction into high and low $b$-values, we selected four neurons from the network for further analysis. These four classes are hence given by overdamped-low-$b$ (OL) neurons, that are equivalent to LIF neurons, overdamped-high-$b$ (OH), underdamped-low-$b$ (UL) and underdamped-high-$b$ (UH) neurons. For example, overdamped neurons with high $b$ exhibit a hyper-polarization, similar to a relative refractory period after emitting a spike (2\textsuperscript{nd} row, 1\textsuperscript{st} column). A second example of an interesting learned behavior is given by underdamped neurons with low spike-triggered adaptation (class UL), which respond with an immediate short burst of action potentials to a positive step in the input current, but with a delayed short burst to a negative step (3\textsuperscript{rd} row, 2\textsuperscript{nd} column). Another intriguing property is the interference of spike-triggered adaptation and oscillations, in which an output spike of an UH-neuron could cancel the membrane potential oscillation (4\textsuperscript{th} row, 3\textsuperscript{rd} column). Such properties don't exist in LIF neurons (class OL), which only spike in response to sufficient input current, without any complex dynamics. In summary, our analysis revealed that recurrent networks of adLIF neurons exhibit a diverse set of neuronal dynamics when trained for optimal performance on a temporal classification task. This diversity is reminiscent of the diversity of neuronal behaviors found in biological neuronal networks \cite{izhikevich2001resonate,gerstner2014neuronal}.

\section*{Supplementary Figure 2}
\label{supfig2}
\nopagebreak
\begin{figure}[H]
    \centering
    \includegraphics[width=\textwidth]{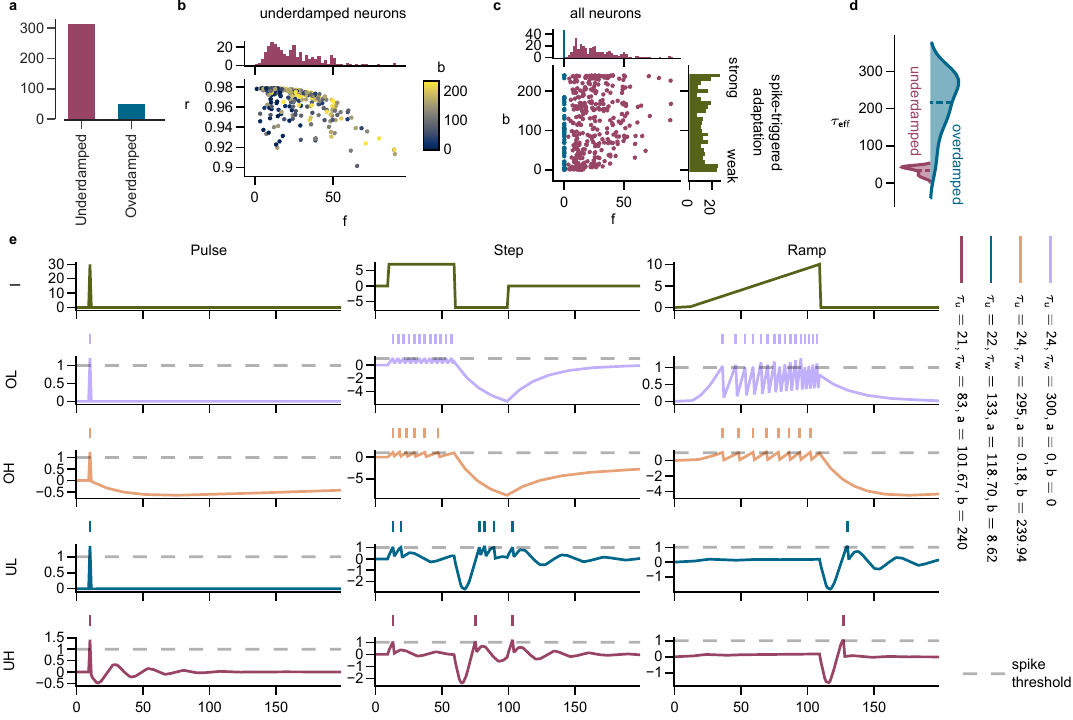}
    \caption{\textbf{Heterogeneity of adLIF neurons in a network trained on SHD} For the analysed network we used the same architecture as for the results reported in Table 2. The network achieved an accuracy of $97.08\%$ on the SHD data set. \textbf{a)} Comparison of the amount of underdamped and overdamped neurons in the network after training. \textbf{b)} Scatter plot of underdamped neurons with respect to their frequency $f$ (x-axis), decay rate $r$ (y-axis) and spike-triggered adaptation parameter $b$ (color). Histogram shows the marginal distribution, summed over the y-axis. \textbf{c)} Scatter plot of all neurons with respect to their intrinsic frequency $f$ and parameter $b$. Histograms show marginal distributions over the corresponding axes. \textbf{d)} Density plot of the distribution of effective time constants $\tau_\mathrm{eff} = -\frac{\Delta t}{\ln r}$ for underdamped (left) and overdamped (right) neurons. Horizontal lines denote the mean. \textbf{e)} Responses of different classes of adLIF neuron parameterizations to different input stimuli: pulse, step and ramp. Parameter $b$ controls the strength of spike-triggered adaptation. The four classes are OL: overdamped with low $b$, equivalent to a LIF neuron, OH: overdamped with high $b$, UL: underdamped with low $b$, and UH: underdamped with high $b$. The plots show the membrane potential over time in response to the input current in the first row. Vertical lines show output spikes of the neuron. All neuron parameterizations were taken from the same network as panels a-d. }
    \label{fig:neuron_analysis}
\end{figure}

\section*{Supplementary Figure 3}
\label{supfig3}
\nopagebreak
\begin{figure}[H]
    \centering
    \includegraphics[width=0.3\textwidth]{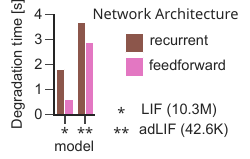}
    \caption{\textbf{Ablation study for "Accurate prediction of dynamical systems trajectories"} Comparison of LIF and adLIF networks trained with and without recurrent connections. Degradation time denotes the time when the model MSE reaches the MSE of the constant model.}
\end{figure}

\section*{Supplementary Notes 4}
\label{supnotes4}
We evaluated recurrent spiking networks with a modified variant of the ALIF neuron \cite{yinAccurateEfficientTimedomain2021} on the BSD task. For the case of $20$ classes this model achieved a test accuracy of $86.53 \pm 1.01\%$ over $10$ runs. For these experiments, we used the same model as in the paper from Yin et al. \cite{yinAccurateEfficientTimedomain2021} except for the following modifications: Neuron time constants $\tau_m$ and $\tau_\text{adp}$ were trained with the reparametrization described in Section "Training and Hyperparameter Search Details for all Tasks" in \meth~and the initialization was the same as for our other experiments (see Table 7). We applied the optimization-based feature visualization from Section "Networks of adLIF neurons tune to high-fidelity temporal features" to an ALIF network trained on the $20$-class BSD task and show the resulting feature maps in Supplementary Fig.~\ref{fig:supp_feature_vis_alif}. The hyperparameters for the ALIF model for the BSD task are shown in Supplementary Table \ref{tab:supp_hyperparams_alif_bsd}.

\section*{Supplementary Figure 4}
\label{supfig4}
\nopagebreak
\begin{figure}[H]
    \centering
    \includegraphics[width=355pt]{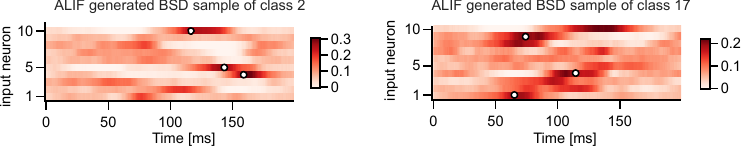}
    \caption{\textbf{Gradient-based feature visualization applied to ALIF models.} The feature maps of the ALIF network with highest accuracy on the $20$-class BSD task for a target class of $2$ (left) and $17$ (right). }
    \label{fig:supp_feature_vis_alif}
\end{figure}

\section*{Supplementary Table 2}
\label{suptab2}
\nopagebreak
\begin{table}[H]
\centering
\adjustbox{max width=\textwidth}{%
\begin{tabular}{ccccccccccccccc}
\hline 
                     & lr    & \rotatebox{90}{\# neurons} & \rotatebox{90}{\# layers} & q  & \rotatebox{90}{$\alpha$ (SLAYER)}  & \rotatebox{90}{c (SLAYER)} & Ep. & \rotatebox{90}{[$\tau_u^\text{min}, \tau_u^\text{max}$]} & \rotatebox{90}{[$\tau_\text{adp}^\text{min}, \tau_\text{adp}^\text{max}$]} &  \rotatebox{90}{dropout} & $\tau_\text{out}$ & \rotatebox{90}{batch size} \\

                     \hline

& $0.008$  & $511$       & $1$         & $120$ & $5$    & $0.4$      & $400$   & $[5,100]$   & $[60,200]$              &  $0\%$              & $15$             & $128$         \\\hline
\end{tabular}
}
    \caption{\textbf{Hyperparameters of the ALIF model for the BSD task.} lr: learning rate for ADAM optimizer, Ep.: number of training epochs, S-reg: Coefficient for spike regularization, $\tau_\text{out}$: Membrane time constants of output layer (leaky integrator), dropout: dropout rate}
    \label{tab:supp_hyperparams_alif_bsd}
\end{table}
\printbibliography
\end{document}